%% file: GaussNewtonArxiv.tex
\documentclass[twoside,11pt]{article}

\newif\ifARXIV
\ARXIVtrue

\ifARXIV
\usepackage{ARXIVjmlr2e}
\else
\usepackage{jmlr2e}
\fi

\usepackage{algorithm2e}
\usepackage{algorithmic}
\usepackage{tikz,xspace,amsmath, amssymb, amsbsy, amsfonts,tabularx,bm,amsthm,setspace}

\usepackage{caption}
\usepackage{subcaption}
\usepackage{framed}
\usepackage{url}

\newcommand{\nn}{\nonumber}
\makeatletter
\newtheorem*{rep@theorem}{\rep@title}
\newcommand{\newreptheorem}[2]{%
\newenvironment{rep#1}[1]{%
 \def\rep@title{#2 \ref{##1}}%
 \begin{rep@theorem}}%
 {\end{rep@theorem}}}
\makeatother

\newtheorem{Lemma}{Lemma}
\newtheorem{Theorem}{Theorem}
\newtheorem{Theorem*}{Theorem}

\newtheorem{Corollary*}{Corollary}

\newreptheorem{theorem}{Theorem}
\newreptheorem{lemma}{Lemma}
\newtheorem{Definition*}{Definition}
\newtheorem{Definition}{Definition}

\newtheorem*{example}{Example}

\newcommand{\cH}{\mathcal{H}}

\include{inputs}

\jmlrheading{1}{2000}{1-48}{4/00}{10/00}{Thomas Furmston}

\ShortHeadings{A Gauss-Newton Method for Markov Decision Processes}{Furmston \& Lever}
\firstpageno{1}

\RestyleAlgo{boxruled}

\begin{document}

\title{A Gauss-Newton Method for Markov Decision Processes}

\author{\name Thomas Furmston \email T.Furmston@cs.ucl.ac.uk \\
       \addr Department of Computer Science \\
       University College London \\
       London, WC1E 6BT
       \AND
       \name Guy Lever \email G.Lever@cs.ucl.ac.uk \\
       \addr Department of Computer Science \\
       University College London \\
       London, WC1E 6BT
}

\ifARXIV
\editor{}
\else
\editor{Anon}
\fi

\maketitle
\begin{abstract}Approximate Newton methods are a standard optimization tool which aim to maintain the benefits of Newton's method, such as a fast rate of convergence, whilst alleviating its drawbacks, such as computationally expensive calculation or estimation of the inverse Hessian. In this work we investigate approximate Newton methods for policy optimization in Markov decision processes  (MDPs). We first analyse the structure of the Hessian of the objective function for MDPs. We show that, like the gradient, the Hessian exhibits useful structure in the context of MDPs and we use this analysis to motivate two Gauss-Newton Methods for MDPs. Like the Gauss-Newton method for non-linear least squares, these methods involve approximating the Hessian by ignoring certain terms in the Hessian which are difficult to estimate. The approximate Hessians possess desirable properties, such as negative definiteness, and we demonstrate several important performance guarantees including guaranteed ascent directions, invariance to affine transformation of the parameter space, and convergence guarantees. We finally provide a unifying perspective of key policy search algorithms, demonstrating that our second Gauss-Newton algorithm is closely related to both the EM-algorithm and natural gradient ascent applied to MDPs, but performs significantly better in practice on a range of challenging domains.
\end{abstract}

\ifARXIV
\begin{keywords}
~
\end{keywords}
\else
\begin{keywords}
Markov Decision Processes, Reinforcement Learning, Newton Method, Function Approximation
\end{keywords}
\fi

\tableofcontents

\section{Introduction}\label{sec:intro}

Markov decision processes (MDPs) are the standard model for optimal control in a fully observable environment \citep{Bertsekas-2010}. Strong empirical results have been obtained in numerous challenging real-world optimal control problems using the MDP framework. This includes problems of non-linear control \citep{stengel-1993,li-todorov-2004,Todorov09iterativelocal,Deisenroth2011c,rawlik-2012-etal,spall-1998}, robotic applications \citep{Kober-peters-2011,kohl-2004,vlassis-etal-2009}, biological movement systems \citep{li-thesis}, traffic management \citep{richter-2007,srinivasan-2006}, helicopter flight control \citep{Abbeel-Coates}, elevator scheduling \citep{Crites-Barto} and numerous games, including chess \citep{veness-2009}, go \citep{gelly-2008}, backgammon \citep{Tesauro-1994} and Atari video games \citep{citeulike:13527831}.  

It is well-known that the global optimum of a MDP can be obtained through methods based on dynamic programming, such as  value iteration \citep{bellman-book} and policy iteration \citep{howard-policy-iteration-1960}. However, these techniques are known to suffer from the curse of dimensionality, which makes them infeasible for most real-world problems of interest. As a result, most research in the reinforcement learning and control theory literature has focused on obtaining approximate or locally optimal solutions. There exists a broad spectrum of such techniques, including approximate dynamic programming methods \citep{Bertsekas-2010}, tree search methods \citep{russell-norig-09,kocsis-2006-ecml,Browne-etal-2012}, local trajectory-optimization techniques, such as differential dynamic programming \citep{Jacobson-70} and iLQG \citep{li-todorov-2006}, and policy search methods \citep{williams92,baxter-etal-2001,sutton-etal-00,marbach-etal-2001,Kakade-2002,Kober-peters-2011}. 

The focus of this paper is on policy search methods, which are a family of algorithms that have proven extremely popular in recent years, and which have numerous desirable properties that make them attractive in practice. Policy search algorithms are typically specialized applications of techniques from numerical optimization \citep{nocedal-2006,dempster-1977}. As such, the controller is defined in terms of a differentiable representation and local information about the objective function, such as the gradient, is used to update the controller in a smooth, non-greedy manner. Such updates are performed in an incremental manner until the algorithm converges to a local optimum of the objective function. There are several benefits to such an approach: the smooth updates of the control parameters endow these algorithms with very general convergence guarantees; as performance is improved at each iteration (or at least on average in stochastic policy search methods) these algorithms have good anytime performance properties; it is not necessary to approximate the value function, which is typically a difficult function to approximate -- instead it is only necessary to approximate a low-dimensional projection of the value function, an observation which has led to the emergence of so called actor-critic methods \citep{Konda:2003:AA:942271.942292,DBLP:conf/nips/KondaT99,bhatnagar-incremental-NAC-2008,bhatnagar_2009}; policy search methods are easily extendable to models for optimal control in a partially observable environment, such as the finite state controllers \citep{DBLP:conf/uai/MeuleauPKK99,toussaint-etal-06}.

In (stochastic) steepest gradient ascent \citep{williams92,baxter-etal-2001,sutton-etal-00} the control parameters are updated by moving in the direction of the gradient of the objective function. While steepest gradient ascent has enjoyed some success, it suffers from a serious issue that can hinder its performance. Specifically, the steepest ascent direction is not invariant to rescaling the components of the parameter space and the gradient is often poorly-scaled, i.e., the variation of the objective function differs dramatically along the different components of the gradient, and this leads to a poor rate of convergence. It also makes the construction of a good step size sequence a difficult problem, which is an important issue in stochastic methods.\footnote{This is because line search techniques lose much of their desirability in stochastic numerical optimization algorithms, due to variance in the evaluations.} Poor scaling is a well-known problem with steepest gradient ascent and alternative numerical optimization techniques have been considered in the policy search literature. Two approaches that have proven to be particularly popular are Expectation Maximization \citep{dempster-1977} and natural gradient ascent \citep{amari-1996,amari-1998,amari-etal-1992}, which have both been successfully applied to various challenging MDPs (see \citet{Dayan-Hinton-EM-RL-97,Kober-Peters-08,toussaint-etal-2011} and \citet{Kakade-2002,Bagnell-2003} respectively).

An avenue of research that has received less attention is the application of Newton's method to Markov decision processes. Although \citet{baxter-etal-2001} provide such an extension of their \textit{GPOMDP} algorithm, they give no empirical results in either \citet{baxter-etal-2001} or the accompanying paper of empirical comparisons \citep{baxter-etal-2001b}. There has since been only a limited amount of research into using the second order information contained in the Hessian during the parameter update. To the best of our knowledge only two attempts have been made: in \citet{Schraudolph-etal-nips-2006} an on-line estimate of a Hessian-vector product is used to adapt the step size sequence in an on-line manner; in \citet{ngo-etal-2011}, Bayesian policy gradient methods \citep{ghavamzadeh-bayesian-gradient} are extended to the Newton method. There are several reasons for this lack of interest. Firstly, in many problems the construction and inversion of the Hessian is too computationally expensive to be feasible. Additionally, the objective function of a MDP is typically not concave, and so the Hessian isn't guaranteed to be negative-definite. As a result, the search direction of the Newton method may not be an ascent direction, and hence a parameter update could actually lower the objective. Additionally, the variance of sample-based estimators of the Hessian will be larger than that of estimators of the gradient. This is an important point because the variance of gradient estimates can be a problematic issue and various methods, such as baselines \citep{weaver-etal-2001,Greensmith-etal-2004}, exist to reduce the variance.

Many of these problems are not particular to Markov decision processes, but are general longstanding issues that plague the Newton method. Various methods have been developed in the optimization literature to alleviate these issues, whilst also maintaining desirable properties of the Newton method. For instance, quasi-Newton methods were designed to efficiently mimic the Newton method using only evaluations of the gradient obtained during previous iterations of the algorithm. These methods have low computational costs, a super-linear rate of convergence and have proven to be extremely effective in practice. See \citet{nocedal-2006} for an introduction to quasi-Newton methods. Alternatively, the well-known Gauss-Newton method is a popular approach that aims to efficiently mimic the Newton method. The Gauss-Newton method is particular to non-linear least squares objective functions, for which the Hessian has a particular structure. Due to this structure there exist certain terms in the Hessian that can be used as a useful proxy for the Hessian itself, with the resulting algorithm having various desirable properties. For instance, the pre-conditioning matrix used in the Gauss-Newton method is guaranteed to be positive-definite, so that the non-linear least squares objective is guaranteed to decrease for a sufficiently small step size.

While a straightforward application of quasi-Newton methods will not typically be possible for MDPs\footnote{In quasi-Newton methods, to ensure an increase in the objective function it is necessary to satisfy the secant condition \citep{nocedal-2006}. This condition is satisfied when the objective is concave/convex or the strong Wolfe conditions are met during a line search. For this reason, stochastic applications of quasi-Newton methods has been restricted to convex/concave objective functions \citep{Schraudolph-2007-etal}.}, in this paper we consider whether an analogue to the Gauss-Newton method exists, so that the benefits of such methods can be applied to MDPs. The specific contributions are as follows:
\begin{itemize}
\item  In Section~\ref{sec:analysisOfHessian}, we present an analysis of the Hessian for MDPs. Our starting point is a policy Hessian theorem (Theorem~\ref{theorem_policy_hessian_theorem}) and we analyse the behaiviour of individual terms of the Hessian to provide insight into constructing efficient approximate Newton methods for policy optimization. In particular we show that certain terms are negligible near local optima.
\item Motivated by this analysis, in Section~\ref{sec_apxn_method} we provide two Gauss-Newton type methods for policy optimization in MDPs which retain certain terms of our Hessian decomposition in the preconditioner in a gradient-based policy search algorithm. The first method discards terms which are negligible near local optima and are difficult to approximate. The second method further discards an additional term which we cannot guarantee to be negative-definite. We provide an analysis of our Gauss-Newton methods and give several important performance guarantees for the second Gauss-Newton method:
\begin{itemize}
\item We demonstrate that the pre-conditioning matrix is negative-definite when the controller is $\log$-concave in the control parameters (detailing some widely used controllers for which this condition holds) guaranteeing that the search direction is an ascent direction.
\item We show that the method is invariant to affine transformations of the parameter space and thus does not suffer the significant drawback of steepest ascent.
\item We provide a convergence analysis, demonstrating linear convergence to local optima, in terms of the step size of the update. One key practical benefit of this analysis is that the step size for the incremental update can be chosen independently of unknown quantities, while retaining a guarantee of convergence.
\item The preconditioner has a particular form which enables the assent direction to be computed particularly efficiently via a Hessian-free conjugate gradient method in large parameter spaces. 
\end{itemize} 
\item In Section~\ref{sec:unifying} we present a unifying perspective for several policy search methods. In particular we relate the search direction of our second Gauss-Newton algorithm to that of Expectation Maximization (which provides new insights in to the latter algorithm when used for policy search), and we also discuss its relationship to the natural gradient algorithm. 
\item In Section~\ref{sec:experiments} we present experiments demonstrating state-of-the-art performance on challenging domains including Tetris and robotic arm applications.
\end{itemize}

\section{Preliminaries and Background}

In Section~\ref{sec:mdps} we introduce Markov decision processes, along with some standard terminology relating to these models that will be required throughout the paper. In Section~\ref{sec_parametric_policy_search} we introduce policy search methods and detail several key algorithms from the literature.

\subsection{Markov Decision Processes}\label{sec:mdps}

In a Markov decision process an \textit{agent}, or \textit{controller}, interacts with an environment over the course of a planning horizon. 
At each point in the planning horizon the agent selects an action (based on the the current state of the environment) and receives a scalar reward. 
The amount of reward received depends on the selected action and the state of the environment. 
Once an action has been performed the system transitions to the next point in the planning horizon, and the new state of the environment is determined 
(often in a stochastic manner) by the action the agent selected and the current state of the environment.  
The optimality of an agent's behaviour is measured in terms of the total reward the agent can expect to 
receive over the course of the planning horizon, so that optimal control is obtained when this quantity 
is maximized.

Formally a MDP is described by the tuple $\{\myset{S}, \myset{A}, D, P, R \}$, in which $\myset{S}$ and $\myset{A}$ are sets, known respectively as the state and action space, 
$D$ is the initial state distribution, which is a distribution over the state space, 
$P$ is the transition dynamics and is formed of the set of conditional distributions over the state space, $\{P(\cdot|s,a)\}_{(s,a) \in \mathcal{S} \times \mathcal{A}}$, and $R :\mathcal{S} \times \mathcal{A} \to [0, R_{\textnormal{max}}]$  is the (deterministic) reward function, which is assumed to be bounded and non-negative. Given a planning horizon, $H \in \mathbb{N}$, and a time-point in the planning horizon, $t \in \intSet{H}$, we use the notation $s_t$ and $a_t$ to denote the random
variable of the state and action of the $t^\textnormal{th}$ time-point, respectively.
The state at the initial time-point is determined by the initial state distribution, $s_1 \sim D(\cdot)$. 
At any given time-point, $t \in \mathbb{N}_H$, and given the state of the environment, the agent selects an action, $a_t \sim \pi(\cdot|s_t)$, according to the \emph{policy} $\pi$. The state of the next point in the planning horizon is determined according to the transition dynamics, $s_{t+1} \sim P(\cdot| a_t, s_t)$.
This process of selecting actions and transitioning to a new state is iterated sequentially through all of the time-points in the planning horizon. 
At each point in the planning horizon the agent receives a scalar reward, which is determined by the reward function. 

The objective of a MDP is to find the policy that maximizes a given function of the expected reward over the course of the planning horizon. 
In this paper we usually consider the infinite horizon discounted reward framework, so that the objective function takes the form
\begin{equation}
U(\pi): = \sum_{t=1}^\infty \mathbb{E}_{s_t, a_t \sim p_t} \bigg[ \gamma^{t-1} R(s_t, a_t) ; \pi, D \bigg], \label{objectiveFunction}
\end{equation}
where we use the semi-colon to identify parameters of the distribution, rather than conditioning variables, and where the distribution of $s_t$ and $a_t$, which we denote by $p_t$, is given by the marginal at time $t$ of the joint distribution over $(s_{1:t},a_{1:t})$, where $s_{1:t} = (s_1,s_2,...,s_t)$, $a_{1:t} = (a_1,a_2,...,a_t)$, denoted by  
\begin{equation*}
p(s_{1:t}, a_{1:t}; \pi): = \pi(a_t|s_t) \bigg\{ \prod_{\tau=1}^{t-1} P( s_{\tau+1}| s_\tau, a_\tau) \times \pi(a_\tau|s_\tau) \bigg\} D(s_1).
\end{equation*}
The discount factor $\gamma \in [0,1)$, in (\ref{objectiveFunction}) ensures that the objective is bounded.

We use the notation $\xi_t = (s_1, a_1, s_2, a_2, ...., s_t, a_t)$ to denote trajectories through the state-action space of length, $t \in \mathbb{N}$. 
We use $\xi$ to denote trajectories that are of infinite length, and use $\Xi$ to denote the space of all such trajectories. Given a trajectory, $\xi \in \Xi$,
we use the notation $R(\xi)$ to denote the total discounted reward of the trajectory, so that 
\begin{equation*}
R(\xi) = \sum_{t=1}^{\infty} \gamma^{t-1} R(s_t, a_t).
\end{equation*}
Similarly, we use the notation $p(\xi;\pi)$ to denote the probability of generating the trajectory $\xi$ under the policy $\pi$.

We now introduce several functions that are of central importance. The value function w.r.t. policy $\pi$ is
defined as the total expected future reward given the current state,
\begin{equation}
V_\pi(s) := \sum_{t=1}^\infty \mathbb{E}_{s_t, a_t \sim p_t} \bigg[ \gamma^{t-1} R(s_t, a_t) \big| s_1 \ceql s;\pi \bigg]. \label{sValFnInfiniteSummation}
\end{equation}
It can be seen that $U(\pi) = \mathbb{E}_{s \sim D} [V_\pi(s)]$. The
value function can also be written as the solution of the following
fixed-point equation,   
\begin{equation}
V_\pi(s) = \mathbb{E}_{a \sim \pi(\cdot|s)} \bigg[ R(s, a) + \gamma \mathbb{E}_{s' \sim P(\cdot| s, a)} \big[ V_\pi(s') \big] \bigg], \label{sValFnFixedPoint}
\end{equation}
which is known as the Bellman equation \citep{Bertsekas-2010}. The state-action value function w.r.t. policy $\pi$ is given by 
\begin{equation}
Q_\pi(s, a): = R(s, a) + \gamma \mathbb{E}_{s' \sim P(\cdot| s, a)}
\bigg[ V_\pi(s') \bigg], \label{saValFn} 
\end{equation}
and gives the value of performing an action, in a given state, and then following the policy. Note that $V_\pi(s) = \sum_{a \in \mathcal{A}} \pi(a|s) Q_\pi(s,a)$. 
Finally, the advantage function \citep{Baird93}
\begin{equation*}
A_{\pi}(s, a) := Q_\pi(s, a) - V_\pi(s), \label{advFn}
\end{equation*}
gives the relative advantage of an action in relation to the
other actions available in that state and it can be seen that $\sum_{a \in \mathcal{A}} \pi(a|s) A_{\pi}(s,
a) = 0$, for each $s \in \mathcal{S}$.

\subsection{Policy Search Methods}\label{sec_parametric_policy_search}

In policy search methods the policy is given some differentiable parametric form, denoted $\pi(a|s;\bm{w})$ with $\bm{w}$ the policy parameter, and local information, such as the gradient of the objective function, is used to update the policy in a smooth non-greedy manner. This process is iterated in an incremental manner until the algorithm converges to a local optimum of the objective function. Denoting the parameter space by $\myset{W} \subset \mathbb{R}^n$, $n \in \mathbb{N}$, we write the objective function directly in terms of the parameter vector, i.e.,
\begin{equation}
U(\bm{w}) = \sum_{(s,a) \in \mathcal{S} \times \mathcal{A}} \quad \sum_{t=1}^\infty \gamma^{t-1} p_t(s,a;\bm{w}) R(s, a) , \quad \quad \quad \bm{w} \in \mathcal{W}, \label{ps_objectiveFunction}
\end{equation}
while the trajectory distribution is written in the form
\begin{equation}
p(a_{1:H}, s_{1:H};\bm{w}) = p(a_H|s_H;\bm{w}) \bigg\{ \prod_{t=1}^{H-1} p(s_{t+1}|a_t,s_t) \pi(a_t|s_t;\bm{w}) \bigg\} p_1(s_1), \quad H \in \mathbb{N}. \label{trajectory_prob_parametric} 
\end{equation}
Similarly, $V(s;\bm{w})$, $Q(s,a;\bm{w})$ and $A(s,a;\bm{w})$ denote respectively the value function, state-action value function and the advantage function in terms of the parameter vector $\bm{w}$. We introduce the notation
\begin{equation}
p_{\gamma}(s,a; \bm{w}): = \sum_{t=1}^{\infty} \gamma^{t-1} p_t(s,a; \bm{w}). \label{gamma_state_dist}
\end{equation}
Note that the objective function can be written
\begin{equation}
U(\bm{w}) = \sum_{(s,a) \in \mathcal{S} \times \mathcal{A}} p_{\gamma}(s,a; \bm{w}) R(s, a). \label{ps_objectiveFunction_secondForm}
\end{equation}


We shall consider two forms of policy search algorithm in this paper, gradient-based optimization methods and methods based on iteratively optimizing a lower-bound on the objective function. 
In gradient-based methods the update of the policy parameters take the form
\begin{equation}
\bm{w}^{\textnormal{new}} =  \bm{w} + \alpha \mathcal{M}(\bm{w}) \nabla_{\bm{w}} U(\bm{w}), \label{general_gradient_step}
\end{equation}
where $\alpha \in \RealLine^+$ is the step size parameter and $\mathcal{M}(\bm{w})$ is some preconditioning matrix that possibly depends on $\bm{w} \in \mathcal{W}$. If $\mathcal{M}(\bm{w})$ is positive-definite and $\alpha$ is sufficiently small, then such an update will increase the total expected reward. Provided that the preconditioning matrix is always negative-definite and the step size sequence is appropriately selected, by iteratively updating the policy parameters according to (\ref{general_gradient_step}) the policy parameters will converge to a local optimum of (\ref{ps_objectiveFunction}). This generic gradient-based policy search algorithm is given in Algorithm~\ref{GenericPolicySearchAlg}. Gradient-based methods vary in the form of the preconditioning matrix used in the parameter update. The choice of the preconditioning matrix determines various aspects of the resulting algorithm, such as the computational complexity, the rate at which the algorithm converges to a local optimum and invariance properties of the parameter update. Typically the gradient $\nabla_{\bm{w}} U(\bm{w})$ and the preconditioner $\mathcal{M}(\bm{w})$ will not be known exactly and must be approximated by collecting data from the system. In the context of reinforcement learning, the Expectation Maximization (EM) algorithm searches for the optimal policy by iteratively optimizing a lower bound on the objective function. While the EM-algorithm doesn't have an update of the form given in (\ref{general_gradient_step}) we shall see in Section~\ref{compareGaussNewton2EM} that the algorithm is closely related to such an update. We now review specific policy search methods.

\begin{algorithm}[tb]
  \caption{Generic gradient-based policy search algorithm}\label{GenericPolicySearchAlg}
  \KwIn{Initial vector of policy parameters, $\bm{w}_0 \in \mathcal{W}$, and a step size sequence, $\{ \alpha_k \}_{k=0}^\infty$, with $\alpha_k \in \mathbb{R}^+$ for $k \in \mathbb{N}$.}

  Set iteration counter, $k \leftarrow 0$.
  
  \Repeat{ \textnormal{Convergence of the policy parameters}}{
  
  Calculate the gradient of the objective $\nabla_{\bm{w} = \bm{w}_k} U(\bm{w})$, and the preconditioner $\mathcal{M}(\bm{w}_k)$ at the current point in the parameter space.

  \begin{spacing}{1.5}  
  \end{spacing}      
  
  Update policy parameters, $\bm{w}_{k+1} = \bm{w}_k + \alpha_k \mathcal{M}(\bm{w}_k) \nabla_{\bm{w} = \bm{w}_k} U(\bm{w})$.
  
  \begin{spacing}{1.5}  
  \end{spacing}      
  
  Update iteration counter, $k \leftarrow k+1$.
  
  \begin{spacing}{1.5}  
  \end{spacing}        
  
  }

  \begin{spacing}{2.5}  
  \end{spacing}      
  
  \KwRet{$\bm{w}_k$}

\end{algorithm} 

\subsubsection{Steepest Gradient Ascent}\label{sec_steepest_gradient_ascent}

Steepest gradient ascent corresponds to the choice $\mathcal{M}(\bm{w}) = I_{n}$, where $I_{n}$ denotes the $n \times n$ identity matrix so that the parameter update takes the form:
\begin{framed}
\noindent\underline{Policy search update using steepest ascent}
\begin{equation}
\bm{w}^{\textnormal{new}} =  \bm{w} + \alpha \nabla_{\bm{w}} U(\bm{w}). \label{steepest_gradient_step}
\end{equation}
\end{framed}
\noindent The gradient $\nabla_{\bm{w}} U(\bm{w})$ can be written in a relatively simple form using the following theorem \citep{sutton-etal-00}:

\begin{Theorem}[Policy Gradient Theorem \citep{sutton-etal-00}]\label{theorem_policy_gradient_theorem}
Suppose we are given a Markov Decision Process with objective (\ref{ps_objectiveFunction}) and Markovian trajectory distribution (\ref{trajectory_prob_parametric}). For any given parameter vector, $\bm{w} \in \mathcal{W}$, the gradient of (\ref{ps_objectiveFunction}) takes the form
\begin{equation}
\nabla_{\bm{w}} U(\bm{w}) = \sum_{s \in \myset{S}} \sum_{a \in \myset{A}} p_{\gamma}(s,a; \bm{w}) Q(s,a;\bm{w}) \nabla_{\bm{w}} \log \pi(a|s; \bm{w}). \label{pg_gradient}
\end{equation}
\begin{proof}
This is a well-known result that can be found in \citet{sutton-etal-00}. A derivation of (\ref{pg_gradient}) is provided in Section~\ref{app:pol_grad_theorem} in the Appendix.
\end{proof}
\end{Theorem}
It is not possible to calculate the gradient exactly for most real-world MDPs of interest. For instance, in discrete domains the size of the state-action space may be too large for enumeration over these sets to be feasible. Alternatively, in continuous domains the presence of non-linearities in the transition dynamics makes the calculation of the occupancy marginals an intractable problem. Various techniques have been proposed in the literature to estimate the gradient, including the method of finite-differences \citep{kiefer-1952,kohl-2004,Tedrake-2005}, simultaneous perturbation methods \citep{spall-1992,spall-1998,srinivasan-2006} and likelihood-ratio methods \citep{glynn-86, glynn-90, williams92, baxter-etal-2001,Konda:2003:AA:942271.942292,DBLP:conf/nips/KondaT99,sutton-etal-00,bhatnagar_2009,Kober-peters-2011}. Likelihood-ratio methods, which originated in the statistics literature and were later applied to MDPs, are now the prominent method for estimating the gradient. There are numerous such methods in the literature, including Monte-Carlo methods \citep{williams92,baxter-etal-2001} and actor-critic methods \citep{Konda:2003:AA:942271.942292,DBLP:conf/nips/KondaT99,sutton-etal-00,bhatnagar_2009,Kober-peters-2011}.

Steepest gradient ascent is known to perform poorly on objective functions that are poorly-scaled, that is, if changes to some parameters produce much larger variations to the function than changes in other parameters. In this case steepest gradient ascent zig-zags along the ridges of the objective in the parameter space \citep[see e.g.,][]{nocedal-2006}. It can be extremely difficult to gauge an appropriate scale for these steps sizes in poorly-scaled problems and the robustness of optimization algorithms to poor scaling is of significant practical importance in reinforcement learning since line search procedures to find a suitable step size are often impractical.

\subsubsection{Natural Gradient Ascent}\label{pg_natural_gradients_section}

Natural gradient ascent techniques originated in the neural network and blind source separation literature \citep{amari-1996,amari-1998,amari-1996b,amari-etal-1992}, and were introduced into the policy search literature in \citet{Kakade-2002}. To address the issue of poor scaling, natural gradient methods take the perspective that the parameter space should be viewed with a manifold structure in which distance between points on the manifold captures discrepancy between the models induced by different parameter vectors. In natural gradient ascent $\mathcal{M}(\bm{w}) = G^{-1}(\bm{w})$ in (\ref{general_gradient_step}), with $G(\bm{w})$ denoting the Fisher information matrix, so that the parameter update takes the form
\begin{framed}
\noindent\underline{Policy search update using natural gradient ascent}
\begin{equation}
\bm{w}^\textnormal{new} = \bm{w} + \alpha G^{-1}(\bm{w}) \nabla_{\bm{w}} U(\bm{w}). \label{natural_gradients_step}
\end{equation}
\end{framed}
\noindent In the case of Markov decision processes the Fisher information matrix takes the form,
\begin{equation}
G(\bm{w}) = - \sum_{s \in \mathcal{S}} \sum_{a \in \mathcal{A}} p_{\gamma}(s,a; \bm{w}) \nabla_{\bm{w}} \nabla_{\bm{w}}^\top \log \pi(a| s; \bm{w}), \label{fisher_information}
\end{equation}
which can then be viewed as a imposing a local norm on the parameter space which is second order approximation to the KL-divergence between induced policy distributions. When the trajectory distribution satisfies the Fisher regularity conditions \citep{Lehmann-book} there is an alternate, equivalent, form of the Fisher information matrix given by  
\begin{equation}
G(\bm{w}) = \sum_{s \in \mathcal{S}} \sum_{a \in \mathcal{A}} p_{\gamma}(s,a; \bm{w}) \nabla_{\bm{w}} \log \pi(a| s; \bm{w}) \nabla_{\bm{w}}^\top \log \pi(a| s; \bm{w}) . \label{fisher_information2}
\end{equation}

There are several desirable properties of the natural gradient approach: the Fisher information matrix is always positive-definite, regardless of the policy parametrization; The search direction is invariant to the parametrization of the policy, \citep{Bagnell-2003,peters-NAC_2008}. Additionally, when using a compatible function approximator \citep{sutton-etal-00} within an actor-critic framework, then the optimal critic parameters coincide with the natural gradient. Furthermore, natural gradient ascent has been shown to perform well in some difficult MDP environments, including Tetris \citep{Kakade-2002} and several challenging robotics problems \citep{peters-NAC_2008}.  However, theoretically, the rate of convergence of natural gradient ascent is the same as steepest gradient ascent, i.e., linear, although, it has been noted to be substantially faster in practice. 

\subsubsection{Expectation Maximization}\label{pg_em_construct}

An alternative optimization procedure that has been the focus of much research in the planning and reinforcement learning communities is the EM-algorithm \citep{Dayan-Hinton-EM-RL-97, toussaint-etal-06, toussaint-etal-2011, Kober-Peters-08, Kober-peters-2011, hoffman-etal-09, furmston-barber, furmston-barber-10}. The EM-algorithm is a powerful optimization technique popular in the statistics and machine learning community \citep[see e.g.,][]{dempster-1977, little-em-book, neal-hinton-em-variants} that has been successfully applied to a large number of problems. See \citet{barber-book-11} for a general overview of some of the applications of the algorithm in the machine learning literature. Among the strengths of the algorithm are its guarantee of increasing the objective function at each iteration, its often simple update equations and its generalization to highly intractable models through variational Bayes approximations \citep{saul-etal-1996}. 

Given the advantages of the EM-algorithm it is natural to extend the algorithm to the MDP framework. Several derivations of the EM-algorithm for MDPs exist \citep{Kober-peters-2011,toussaint-etal-2011}. For reference we state the lower-bound upon which the algorithm is based in the following theorem.

\begin{Theorem}\label{theorem_em_lower_bound}
Suppose we are given a Markov Decision Process with objective (\ref{ps_objectiveFunction}) and Markovian trajectory distribution (\ref{trajectory_prob_parametric}). Given any distribution, $q$, over the space of trajectories, $\Xi$, then the following bound holds, 
\begin{equation}
\log U(\bm{w}) \ge H_{\textnormal{entropy}}(q(\xi)) + \expectation{\xi \sim q(\cdot)} \bigg[ \log \big( p(\xi; \bm{w}) R(\xi) \big) \bigg], \quad \quad \forall \bm{w} \in \mathcal{W}, \label{pg_MDP_bound}
\end{equation}
in which $H_{\textnormal{entropy}}$ denotes the entropy function \citep{barber-book-11}.
\begin{proof}
The proof is based on an application of Jensen's inequality and can be found in \citet{Kober-peters-2011}. 
\end{proof}
\end{Theorem}

The distribution, $q$, in Theorem \ref{theorem_em_lower_bound} is often referred to as the variational distribution. An EM-algorithm is obtained through coordinate-wise optimization of (\ref{pg_MDP_bound}) with respect to the variational distribution (the E-step) and the policy parameters (the M-step). In the E-step the lower-bound is optimized when $q(\xi) \propto p(\xi; \bm{w}') R(\xi)$, in which $\bm{w}'$ are the current policy parameters. In the M-step the lower-bound is optimized with respect to $\bm{w}$, which, given $q(\xi) \propto p(\xi; \bm{w}') R(\xi)$ and the Markovian structure of $\log p(\xi; \bm{w})$, is equivalent to optimizing the function,
\begin{equation}
\mathcal{Q}(\bm{w},\bm{w}') = \sum_{(s,a) \in \mathcal{S} \times \mathcal{A}} p_{\gamma}(s,a;\bm{w}') Q(s,a;\bm{w}') \bigg[ \log \pi(a|s; \bm{w}) \bigg], \label{pg_general_em}
\end{equation}
with respect to the first parameter, $\bm{w}$. The E-step and M-step are iterated in this manner until the policy parameters converge to a local optimum of the objective function.

\section{The Hessian of Markov Decision Processes}\label{sec:analysisOfHessian}

As noted in Section~\ref{sec:intro}, the Newton method suffers from issues that often make its application to MDPs 
unattractive in practice. As a result there has been comparatively little research into the Newton method in the 
policy search literature. However, the Newton method has significant attractive properties, such as affine invariance of the policy parametrization and a quadratic rate of convergence. It is of interest, therefore, to consider whether one can construct an efficient Gauss-Newton type method for MDPs, in which the positive aspects of the Newton method are maintained and the negative aspects are alleviated. 
To this end, in this section we provide an analysis of the Hessian of a MDP. This analysis will 
then be used in Section~\ref{sec_apxn_method} to propose Gauss-Newton type methods for MDPs.

In Section~\ref{sec:policyHessianTheorem} we provide a novel representation of the Hessian of a MDP, in Section~\ref{sec:policyHessianDefiniteProperties}
we detail the definiteness properties of certain terms in the Hessian and in 
Section~\ref{hessian_analysis_section_local_optimum} we analyse the behaviour 
of individual terms of the Hessian in the vicinity of a local optimum.

\subsection{The Policy Hessian Theorem}\label{sec:policyHessianTheorem}

There is a standard expansion of the Hessian of a MDP in the policy search literature \citep{baxter-etal-2001,Kakade-2001,Kakade-2002} 
that, as with the gradient, takes a relatively simple form. This is summarized in the following result. 

\begin{Theorem}[Policy Hessian Theorem]\label{theorem_policy_hessian_theorem}
Suppose we are given a Markov Decision Process with objective (\ref{ps_objectiveFunction}) and Markovian trajectory distribution (\ref{trajectory_prob_parametric}). For any given parameter vector, $\bm{w} \in \mathcal{W}$, the Hessian of (\ref{ps_objectiveFunction}) takes the form
\begin{equation}
\mathcal{H} (\bm{w}) = \mathcal{H}_1(\bm{w}) + \mathcal{H}_2(\bm{w}) + \mathcal{H}_{12}(\bm{w}) + \mathcal{H}_{12}^\top(\bm{w}), \label{pg_hessian}
\end{equation}
in which the matrices $\mathcal{H}_1(\bm{w})$, $\mathcal{H}_2(\bm{w})$ and $\mathcal{H}_{12}(\bm{w})$ can be written in the form
\begin{align}
\mathcal{H}_1(\bm{w}) &:= \sum_{s \in \myset{S}} \sum_{a \in \myset{A}} p_{\gamma}(s,a; \bm{w}) Q(s,a;\bm{w}) \nabla_{\bm{w}} \log \pi(a|s; \bm{w}) \nabla_{\bm{w}}^\top \log \pi(a|s; \bm{w}), \label{pg_hessian_H1} \\
\mathcal{H}_2(\bm{w}) &:= \sum_{s \in \myset{S}} \sum_{a \in \myset{A}} p_{\gamma}(s,a; \bm{w}) Q(s,a;\bm{w}) \nabla_{\bm{w}} \nabla_{\bm{w}}^\top \log \pi(a|s; \bm{w}), \label{pg_hessian_H2} \\
\mathcal{H}_{12}(\bm{w}) &:= \sum_{s \in \myset{S}} \sum_{a \in \myset{A}} p_\gamma(s,a;\bm{w}) \nabla_{\bm{w}} \log \pi(a|s; \bm{w}) \nabla_{\bm{w}}^\top Q(s,a;\bm{w}). \label{pg_hessian_H12}
\end{align}
\begin{proof}
A derivation for a sample-based estimator of the Hessian can be found in \citet{baxter-etal-2001}. For ease of reference a derivation of (\ref{pg_hessian}) is provided in Section~\ref{app:pol_grad_theorem} in the Appendix.
\end{proof}
\end{Theorem}

We remark that $\mathcal{H}_1(\bm{w})$ and $\mathcal{H}_2(\bm{w})$ are relatively simple to estimate, in the same manner as estimating the policy gradient. The term $\mathcal{H}_{12}(\bm{w})$ is more difficult to estimate since it contains terms involving the unknown gradient $\nabla_{\bm{w}}^\top Q(s,a;\bm{w})$ and removing this dependence would result in a double sum over state-actions.

Below we will present a novel form for the Hessian of a MDP, with attention given to the term $\mathcal{H}_1(\bm{w}) + \mathcal{H}_2(\bm{w})$ in (\ref{pg_hessian}), which will require the following notion of parametrization with constant curvature.
\begin{Definition}\label{defn_const_curv}
A policy parametrization is said to have constant curvature with respect to the action space, 
if for each $(s,a) \in \mathcal{S} \times \mathcal{A}$ the Hessian of the log-policy, 
$\nabla_{\bm{w}} \nabla^\top_{\bm{w}} \log \pi(a|s;\bm{w})$, does not depend upon the action, i.e.,
\begin{equation*}
\nabla_{\bm{w}} \nabla^\top_{\bm{w}} \log \pi(a|s;\bm{w}) = \nabla_{\bm{w}} \nabla^\top_{\bm{w}} \log \pi(a'|s;\bm{w}), \quad \quad \forall a, a' \in \mathcal{A}.
\end{equation*}
When a policy parametrization satisfies this property the notation, $\nabla_{\bm{w}} \nabla^\top_{\bm{w}} \log \pi(s;\bm{w})$, is used to denote $\nabla_{\bm{w}} \nabla^\top_{\bm{w}} \log \pi(a|s;\bm{w})$, for each $a \in \mathcal{A}$.
\end{Definition}

A common class of policy which satisfies the property of Definition~\ref{defn_const_curv} is, 
$\pi(a|s;\bm{w}) \propto \exp( \bm{w}^\top \bm{\phi} (a, s))$, in which $\bm{\phi} (a, s)$ 
is a vector of features that depends on the state-action pair, $(a,s) \in \mathcal{A} \times \mathcal{S}$. 
Under this parametrization, 
\begin{equation*}
\nabla_{\bm{w}} \nabla^\top_{\bm{w}} \log \pi(a|s;\bm{w}) = - \textnormal{Cov}_{a' \sim \pi(\cdot|s;\bm{w})} \big( \bm{\phi} (a', s), \bm{\phi} (a', s) \big),
\end{equation*}
which does not depend on, $a \in \mathcal{A}$. In the case when the action space is continuous, then the policy 
parametrization $\pi(a|s; \bm{w}; \Sigma) \propto \exp \big( - \frac{1}{2} (a - \bm{w}^\top \bm{\phi}(s))^\top \Sigma^{-1} (a - \bm{w}^\top \bm{\phi}(s)) \big)$,
in which $\bm{\phi}: \mathcal{S} \to \mathbb{R}^n$ is a given feature map, satisfies the properties of Definition~\ref{defn_const_curv} with 
respect to the mean parameters, $\bm{w} \in \mathcal{W}$.

We now present a novel decomposition of the Hessian for Markov decision processes.

\begin{Theorem}\label{corollary_policy_hessian_theorem}
Suppose we are given a Markov Decision Process with objective (\ref{ps_objectiveFunction}) and Markovian trajectory distribution (\ref{trajectory_prob_parametric}). For any given parameter vector, $\bm{w} \in \mathcal{W}$, the Hessian of (\ref{ps_objectiveFunction}) takes the form
\begin{equation}
\mathcal{H} (\bm{w}) = \mathcal{A}_1(\bm{w}) + \mathcal{A}_2(\bm{w}) + \mathcal{H}_{12}(\bm{w}) + \mathcal{H}_{12}^\top(\bm{w}). \label{pg_hessian2a}
\end{equation}
Where,
\begin{align}
\mathcal{A}_1(\bm{w}) &:= \sum_{(s,a) \in \mathcal{S} \times \mathcal{A}} p_\gamma(s,a;\bm{w}) A(s,a;\bm{w}) \nabla_{\bm{w}} \log \pi(a|s;\bm{w}) \nabla^\top_{\bm{w}} \log \pi(a|s;\bm{w}) \nn \\
\mathcal{A}_2(\bm{w}) &:= \sum_{(s,a) \in \mathcal{S} \times \mathcal{A}} p_\gamma(s,a;\bm{w}) A(s,a;\bm{w}) \nabla_{\bm{w}}  \nabla^\top_{\bm{w}} \log \pi(a|s;\bm{w}). \nn
\end{align}
When the curvature of the $\log$-policy is independent of the action, then the Hessian takes the form
\begin{equation}
\mathcal{H} (\bm{w}) = \mathcal{A}_1(\bm{w}) + \mathcal{H}_{12}(\bm{w}) + \mathcal{H}_{12}^\top(\bm{w}). \label{pg_hessian2b}
\end{equation}
\end{Theorem}

\begin{proof} See Section~\ref{NovelHessianProof} in the Appendix. \end{proof}

We now present an analysis of the terms of the policy Hessian, simplifying the expansion and demonstrating conditions under which certain terms disappear. The analysis will be used to motivate our Gauss-Newton methods in Section~\ref{sec_apxn_method}.

\subsection{Analysis of the Policy Hessian -- Definiteness}\label{sec:policyHessianDefiniteProperties}

An interesting comparison can be made between the expansions (\ref{pg_hessian}) and (\ref{pg_hessian2a}, \ref{pg_hessian2b}) in terms of the definiteness properties of the component matrices. 
As the state-action value function is non-negative over the entire state-action space, it can be seen that $\mathcal{H}_1(\bm{w})$
is positive-definite for all $\bm{w} \in \mathcal{W}$. Similarly, it can be shown that under certain common policy 
parametrizations $\mathcal{H}_2(\bm{w})$ is negative-definite over the entire parameter space. This is summarized in the 
following theorem.

\begin{Theorem} \label{NegativeDefiniteLemma}
The matrix $\mathcal{H}_2(\bm{w})$ is negative-definite for all $\bm{w} \in \mathcal{W}$ if: 1) the policy is $\log$-concave with respect to the policy parameters; or 2) the policy parametrization has constant curvature with respect to the action space.
\end{Theorem}
\begin{proof}
See Section~\ref{negative_definite_appendix} in the Appendix.
\end{proof}

It can be seen, therefore, that when the policy parametrization satisfies the properties of Theorem~\ref{NegativeDefiniteLemma} the expansion 
(\ref{pg_hessian}) gives $\mathcal{H}(\bm{w})$ in terms of a positive-definite term, $\mathcal{H}_1(\bm{w})$, a negative-definite 
term, $\mathcal{H}_2(\bm{w})$, and a remainder term, $\mathcal{H}_{12}(\bm{w}) + \mathcal{H}_{12}^\top(\bm{w})$, which we shall 
show, in Section~\ref{hessian_analysis_section_local_optimum}, becomes negligible around a local optimum when given a 
sufficiently rich policy parametrization. In contrast to the state-action value function, the advantage function takes both 
positive and negative values over the state-action space. As a result, the matrices $\mathcal{A}_1(\bm{w})$ and $\mathcal{A}_2(\bm{w})$ in (\ref{pg_hessian2a}, \ref{pg_hessian2b}) can be indefinite over parts of the parameter space.

\subsection{Analysis in Vicinity of a Local Optimum}\label{hessian_analysis_section_local_optimum}

In this section  we consider the term $\mathcal{H}_{12}(\bm{w}) + \mathcal{H}_{12}^\top(\bm{w})$, which is both difficult to estimate and not guaranteed to be negative definite. In particular, we shall consider the conditions under which these terms vanish at a local optimum. We start by noting that
\begin{align}
\mathcal{H}_{12}(\bm{w}) &= \sum_{(s,a) \in \mathcal{S} \times \mathcal{A}} p_\gamma(s, a; \bm{w}) \nabla_{\bm{w}} \log \pi(a|s; \bm{w}) \nabla_{\bm{w}}^\top\left(R(s,a) + \gamma \sum_{s'} p(s'|a,s) V(s';\bm{w})\right), \nonumber\\
&= \gamma \sum_{(s,a) \in \mathcal{S} \times \mathcal{A}} p_\gamma(s,a;\bm{w}) \nabla_{\bm{w}} \log \pi(a| s; \bm{w})   \sum_{s'} p(s'|a,s) \nabla_{\bm{w}}^\top V(s';\bm{w}). \label{H12_dwV_version}
\end{align}
This means that if  $\nabla_{\bm{w}} V(s';\bm{w}) = \bm{0}$, for all $s' \in \mathcal{S}$, then $\cH_{12}(\bm{w}) + \cH_{12}^\top(\bm{w})=\bm{0}$. It is sufficient, therefore, to require that $\nabla_{\bm{w}|\bm{w}=\bm{w}^*} V(s;\bm{w}) = \bm{0}$, for all $s\in\mathcal{S}$, at a local optimum $\bm{w}^*\in \mathcal{W}$. We therefore consider the situations in which this occurs.  We start by introducing the notion of a \textit{value consistent} policy class. This property of a policy class captures the idea that the policy class is rich enough such that changing a parameter to maximally improve the value in one state, does not worsen the value in another state. i.e., when a policy class is value consistent, there are no trade-offs between improving the value in different states.

\begin{Definition} \label{ValueConsistent}
A policy parametrization is said to be \textit{value consistent} w.r.t. a Markov decision process if whenever,
\begin{equation}
\bm{e}_i^\top \nabla_{\bm{w}} V(\hat{s}; \bm{w}) \ne 0, \label{definitionEquation1}
\end{equation}
for some $\hat{s} \in \mathcal{S}$, $\bm{w} \in \mathcal{W}$ and $i \in \mathbb{N}_n$, then $\forall s \in \mathcal{S}$ it holds that either
\begin{equation}
\textnormal{sign} \big(  \bm{e}_i^\top \nabla_{\bm{w}} V(s; \bm{w}) \big) = \textnormal{sign} \big( \bm{e}_i^\top \nabla_{\bm{w}} V(\hat{s}; \bm{w}) \big), \label{definitionEquation2}
\end{equation}
or
\begin{equation}
 \bm{e}_i^\top \nabla_{\bm{w}} V(s; \bm{w})  = 0. \label{definitionEquation3}
\end{equation}
Furthermore, for any state, $s \in \mathcal{S}$, for which (\ref{definitionEquation3}) holds it also holds that
\begin{equation*}
\bm{e}_i^\top \nabla_{\bm{w}} \pi(a|s; \bm{w})  = 0, \quad \quad \quad  \forall a \in \mathcal{A}.
\end{equation*}
The notation $\bm{e}_i$ is used to denote the standard basis vector of $\mathbb{R}^n$ in which the $i^{\textnormal{th}}$ component is equal to one, and all other components are equal to zero.
\end{Definition}
\begin{example}
\textnormal{To illustrate the concept of a value consistent policy parametrization we now consider two simple maze navigation MDPs, one with a value consistent policy parametrization, and one with a policy parametrization that is not value consistent. The two MDPs are displayed in Figure \ref{ToyExamples}. Walls of the maze are solid lines, while the dotted lines indicate state boundaries and are passable. The agent starts, with equal probability, in one of the states marked with an `S'. The agent receives a positive reward for reaching the goal state, which is marked with a `G', and is then reset to one of the start states. All other state-action pairs return a reward of zero. There are four possible actions (up, down, left, right) in each state, and the optimal policy is to move, with probability one, in the direction indicated by the arrow. We consider the policy parametrization, $\pi(a|s;\bm{w}) \propto \exp ( \bm{w}^\top \bm{\phi}(s') )$, where $s'$ denotes the successor state of state-action pair $(s, a)$ and $\bm{\phi}$ is a feature map. We consider the feature map $\bm{\phi}:\mathcal{S}\to \{0,1\}^4$ which indicates the presence of a wall on each of the four state boundaries. 
Perceptual aliasing \citep{whitehead-thesis} occurs in both MDPs under this policy parametrization, with states $2$, $3$ \& $4$ aliased in the hallway problem, and states 
$4$, $5$ \& $6$ aliased in McCallum's grid. In the hallway problem all of the aliased states have the same optimal action, and the value of these states all increase/decrease in unison. Hence, it can be seen that the policy parametrization is value consistent for the hallway problem. In McCallum's grid, however, the optimal action for states $4$ \& $6$ is to move upwards, while in state $5$ it is to move downwards. In this example increasing the probability of moving downwards in state $5$ will also increase the probability of moving downwards in states $4$ \& $6$. There is a point, therefore, at which increasing the probability of moving downwards in state $5$ will decrease the value of states $4$ \& $6$. Thus this policy parametrization is not value consistent for McCallum's grid.}
\end{example}

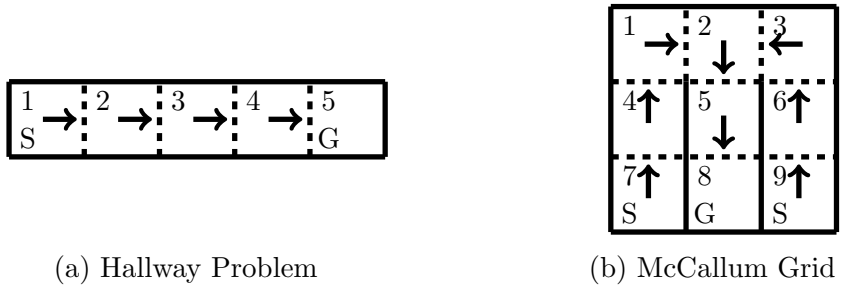
\begin{figure}[t]
\centering
\begin{tikzpicture}
\draw[line width=2pt] (1, 3) -> (6, 3);
\draw[line width=2pt] (1, 2) -> (6, 2);
\draw[line width=2pt] (1, 2) -> (1, 3);
\draw[line width=2pt,dashed] (2, 2) -> (2, 3);
\draw[line width=2pt,dashed] (3, 2) -> (3, 3);
\draw[line width=2pt,dashed] (4, 2) -> (4, 3);
\draw[line width=2pt,dashed] (5, 2) -> (5, 3);
\draw[line width=2pt] (6, 2) -> (6, 3);

\node at (1.25,2.25) {S};
\node at (5.25,2.25) {G};

\node at (1.25,2.75) {1};
\node at (2.25,2.75) {2};
\node at (3.25,2.75) {3};
\node at (4.25,2.75) {4};
\node at (5.25,2.75) {5};

\draw[line width=2pt,->] (1.45, 2.5) -> (1.9, 2.5);
\draw[line width=2pt,->] (2.45, 2.5) -> (2.9, 2.5);
\draw[line width=2pt,->] (3.45, 2.5) -> (3.9, 2.5);
\draw[line width=2pt,->] (4.45, 2.5) -> (4.9, 2.5);

\draw[line width=2pt] (12, 1) -> (12, 4);
\draw[line width=2pt] (9, 1) -> (9, 4);
\draw[line width=2pt] (9, 1) -> (12, 1);
\draw[line width=2pt] (9, 4) -> (12, 4);
\draw[line width=2pt] (10, 1) -> (10, 3);
\draw[line width=2pt] (11, 1) -> (11, 3);
\draw[line width=2pt,dashed] (10, 3) -> (10, 4);
\draw[line width=2pt,dashed] (11, 3) -> (11, 4);
\draw[line width=2pt,dashed] (9, 2) -> (12, 2);
\draw[line width=2pt,dashed] (9, 3) -> (12, 3);

\node at (9.25,1.25) {S};
\node at (10.25,1.25) {G};
\node at (11.25,1.25) {S};

\draw[line width=2pt,->] (9.5, 1.45) -> (9.5, 1.9);
\draw[line width=2pt,->] (9.5, 2.45) -> (9.5, 2.9);
\draw[line width=2pt,->] (11.5, 1.45) -> (11.5, 1.9);
\draw[line width=2pt,->] (11.5, 2.45) -> (11.5, 2.9);
\draw[line width=2pt,->] (10.5, 2.55) -> (10.5, 2.1);
\draw[line width=2pt,->] (10.5, 3.55) -> (10.5, 3.1);
\draw[line width=2pt,->] (9.45, 3.5) -> (9.9, 3.5);
\draw[line width=2pt,->] (11.55, 3.5) -> (11.1, 3.5);

\node at (9.25,3.75) {1};
\node at (10.25,3.75) {2};
\node at (11.25,3.75) {3};
\node at (9.25,2.75) {4};
\node at (10.25,2.75) {5};
\node at (11.25,2.75) {6};
\node at (9.25,1.75) {7};
\node at (10.25,1.75) {8};
\node at (11.25,1.75) {9};

\node at (3.35,0.5) {(a) Hallway Problem};
\node at (10.35,0.5) {(b) McCallum Grid};
\end{tikzpicture}
\caption{(a) The hallway problem. Under the feature map, $\bm{\phi}$, states 2, 3 and 4 map to the the same feature, and the optimal policy is identical on these states. (b) McCallum's grid. Under the feature map, $\bm{\phi}$, states 4, 5 and 6 map to the same feature, but now the optimal policy differs among these states.} \label{ToyExamples}
\end{figure}

We now show that tabular policies -- i.e., policies such that, for each state $s \in \mathcal{S}$, the conditional distribution $\pi(a|s;\bm{w}_{s})$ is parametrized by a separate parameter vector $\bm{w}_{s} \in \mathbb{R}^{n_s}$ for some $n_s \in \mathbb{N}$ -- are value consistent, regardless of the given Markov decision process.

\begin{Theorem} \label{TabNonDec}
Suppose that a given Markov decision process has a tabular policy parametrization, then the policy parametrization is value consistent.
\begin{proof}
See Section~\ref{ProofOfTabNonDec} in the Appendix.
\end{proof}
\end{Theorem}

We now show that under a value consistent policy parametrization the terms $\mathcal{H}_{12}(\bm{w})$ and $\mathcal{H}_{12}^\top(\bm{w})$ vanish near local optima.
 
\begin{Theorem} \label{lem:nondecreasing}
Suppose that $\bm{w}^* \in \mathcal{W}$ is a local optimum of the differentiable objective function, $U(\bm{w}) = \mathbb{E}_{s\sim p_1(\cdot)} \big[ V(s;\bm{w}) \big]$. Suppose that the Markov chain induced by $\bm{w}^*$ is ergodic. Suppose that the policy parametrization is value consistent w.r.t. the given Markov decision process. Then $\bm{w}^*$ is a stationary point of $V(s;\bm{w})$ for all $s \in \mathcal{S}$, and $\mathcal{H}_{12}(\bm{w^*}) = \mathcal{H}_{12}^\top(\bm{w^*}) = {\bm 0}$
\begin{proof}
See Appendix~\ref{supp_hessian_analysis_section_local_optimum}
\end{proof}
\end{Theorem}

Furthermore, when we have the additional condition that the gradient of the value function is continuous in $\bm{w}$ (at $\bm{w} = \bm{w}^*$) then $\cH_{12}(\bm{w})+\cH_{12}^\top(\bm{w}) \to \bm{0}$ as $\bm{w} \to \bm{w}^*$. This condition will be satisfied if, for example, the policy is continuously differentiable w.r.t. the policy parameters. 

\begin{example}[continued]
\textnormal{Returning to the MDPs given in Figure \ref{ToyExamples}, we now empirically observe the behaviour of the term $\cH_{12}(\bm{w})+\cH_{12}^\top(\bm{w})$ as the policy approaches a local optimum of the objective function. Figure \ref{ToyExamples_Hessian_stats_plots} gives the magnitude of $\cH_{12}(\bm{w})+\cH_{12}^\top(\bm{w})$, in terms of the spectral norm, in relation to the distance from the local optimum. In correspondence with the theory, $\cH_{12}(\bm{w})+\cH_{12}^\top(\bm{w}) \to \bm{0}$ as $\bm{w} \to \bm{w}^*$ in the hallway problem, while this is not the case in McCallum's grid. This simple example illustrates the fact that if the feature representation is well-chosen and sufficiently rich the term $\cH_{12}(\bm{w})+\cH_{12}^\top(\bm{w})$ vanishes in the vicinity of a local optimum.}
\end{example}

\begin{figure}[t]
\captionsetup[subfigure]{justification=centering}
\begin{minipage}[b]{.5\linewidth}
\centering \includegraphics[height=2.5in,width=2.75in]{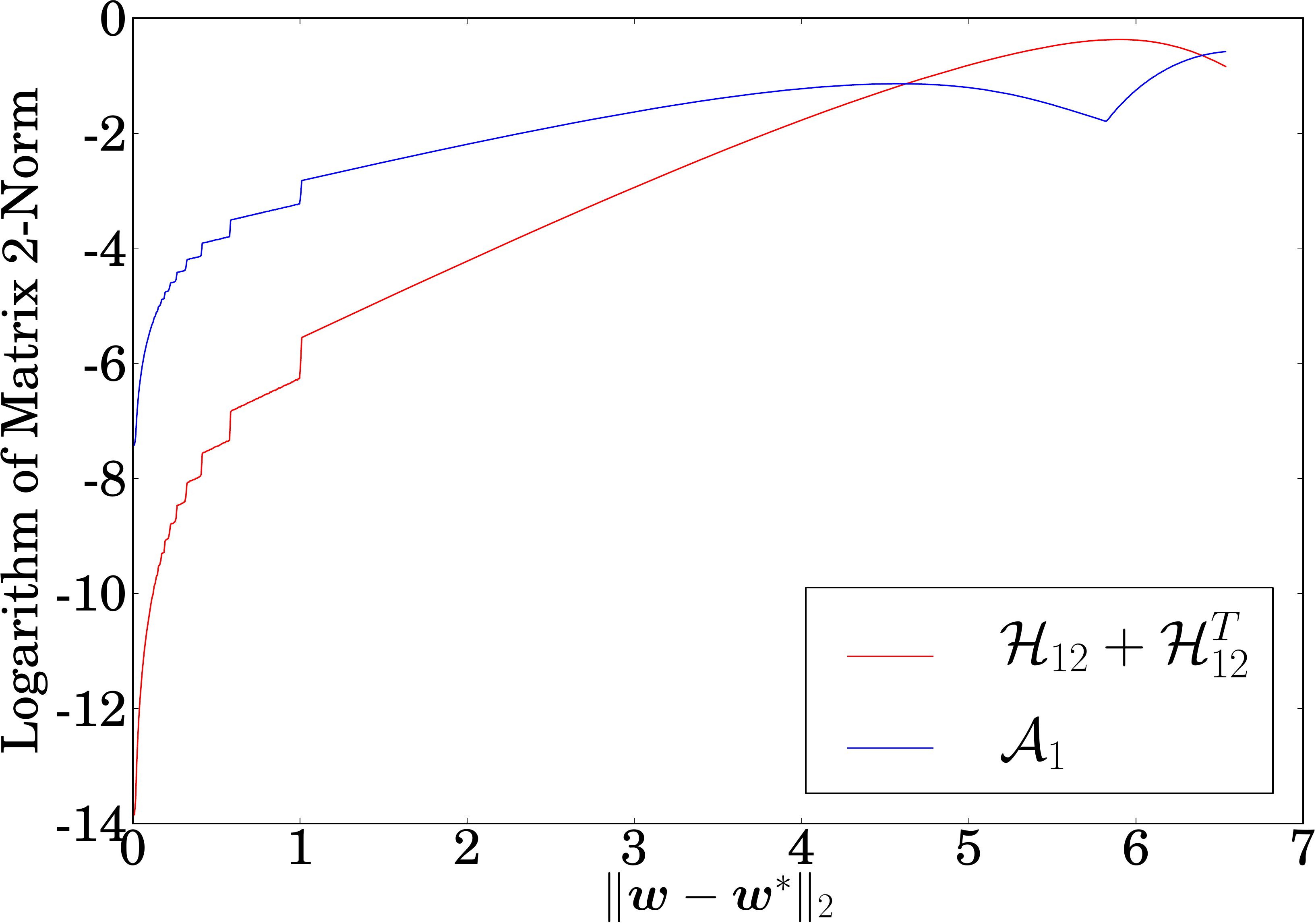}
\subcaption{Hallway Problem}\label{ToyOne_H12_plots}
\end{minipage}%
\begin{minipage}[b]{.5\linewidth}
\centering \includegraphics[height=2.5in,width=2.75in]{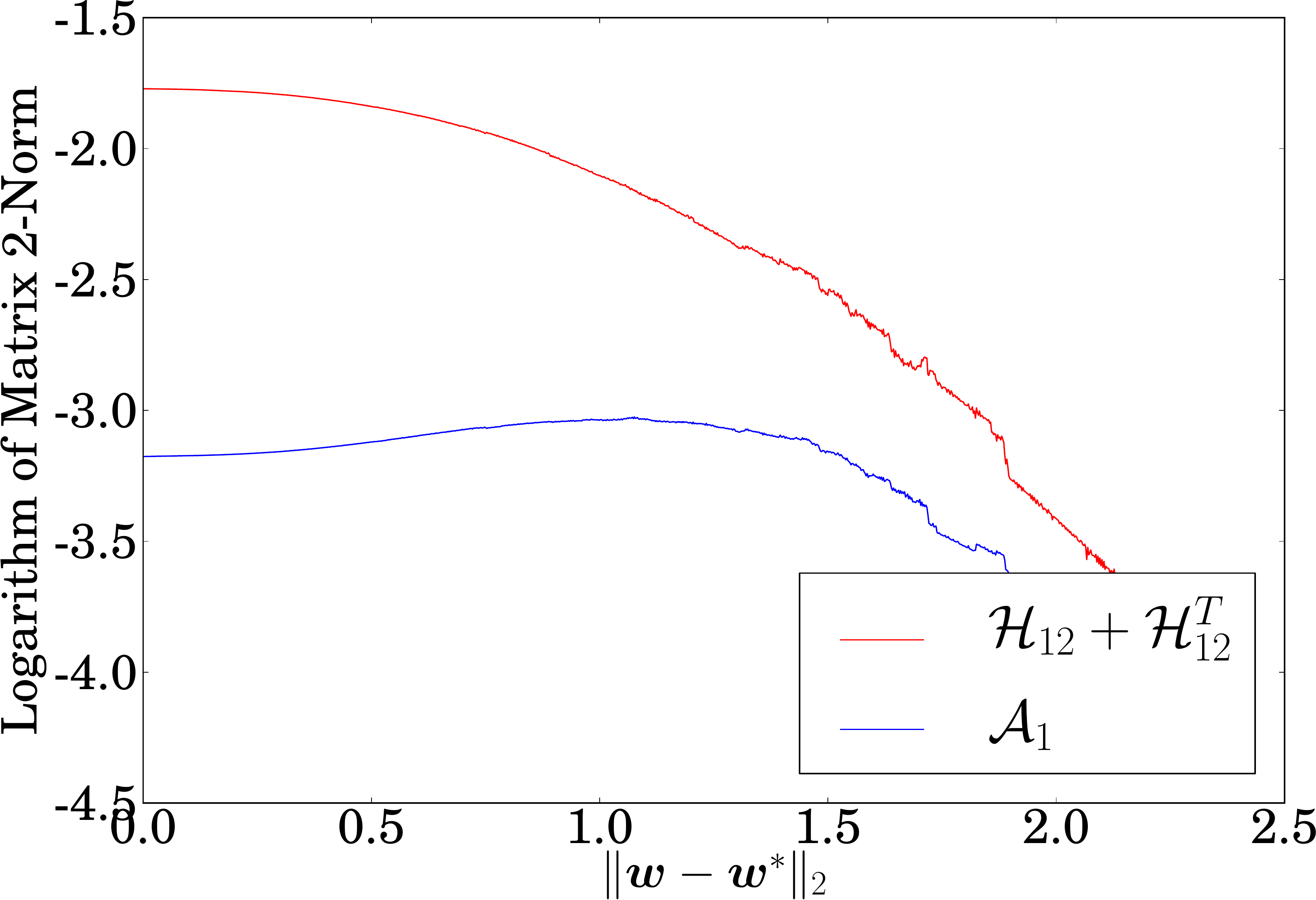}
\subcaption{McCallum's Grid}\label{ToyTwo_H12_plots}
\end{minipage}%
\caption{Graphical illustration of the logarithm of the spectral norm of $\mathcal{H}_{12}(\bm{w}) + \mathcal{H}_{12}^\top(\bm{w})$ and $\mathcal{A}_1(\bm{w})$ in terms of $\| \bm{w} - \bm{w}^* \|_2$ for the hallway problem (a) and McCallum's grid (b). For the given policy parametrization 
$\mathcal{H}(\bm{w}) = \mathcal{A}_1(\bm{w}) + \mathcal{H}_{12}(\bm{w}) + \mathcal{H}_{12}^\top(\bm{w})$, so the plot displays the two components of the Hessian
as the policy converges to a local optimum. As expected, in the hallway problem $\mathcal{H}_{12}(\bm{w}) + \mathcal{H}_{12}^\top(\bm{w}) \to \bm{0}$
as $\bm{w} \to \bm{w}^*$, and $\mathcal{A}_1(\bm{w})$ dominates. In this example the magnitude of $\mathcal{A}_1(\bm{w})$ is roughly six hundred times greater than that of $\mathcal{H}_{12}(\bm{w}) + \mathcal{H}_{12}^\top(\bm{w})$ when $\| \bm{w} - \bm{w}^* \|_2 \approx 0.003$. Conversely, in McCallum's grid $\mathcal{H}_{12}(\bm{w}) + \mathcal{H}_{12}^\top(\bm{w}) \not\to \bm{0}$
as $\bm{w} \to \bm{w}^*$. In fact, $\mathcal{H}_{12}(\bm{w}) + \mathcal{H}_{12}^\top(\bm{w})$ has larger magnitude than $\mathcal{A}_1(\bm{w})$ at $\bm{w}^*$ in this example.} \label{ToyExamples_Hessian_stats_plots}
\end{figure}

\section{Gauss-Newton Methods for Markov Decision Processes}\label{sec_apxn_method}

In this section we propose several Gauss-Newton type methods for MDPs, motivated by the analysis of Section~\ref{sec:analysisOfHessian}. The algorithms are outlined in Section~\ref{subsec_apxn_method}, and key performance analysis is provided in Section~\ref{subsec_apxn_method_analysis}.

\subsection{The Gauss-Newton Methods}\label{subsec_apxn_method}

The first Gauss-Newton method we propose drops the Hessian terms which are difficult to estimate, but are expected to be negligible in the vicinity of local optima. Specifically, it was shown in Section~\ref{hessian_analysis_section_local_optimum} that if the policy parametrization is value consistent with a given MDP, then $\mathcal{H}_{12}(\bm{w}) + \mathcal{H}_{12}^\top (\bm{w}) \to \bm{0}$ as $\bm{w}$ converges towards a local optimum of the objective function. Similarly, if the policy parametrization is sufficiently rich, although not necessarily value consistent, then it is to be expected that $\mathcal{H}_{12}(\bm{w}) + \mathcal{H}_{12}^\top (\bm{w})$ will be negligible in the vicinity of a local optimum. In such cases $\mathcal{A}_1(\bm{w}) + \mathcal{A}_2(\bm{w})$, as defined in Theorem~\ref{corollary_policy_hessian_theorem}, will be a good approximation to the 
Hessian in the vicinity of a local optimum. For this reason, the first Gauss-Newton method that we propose for MDPs is 
to precondition the gradient with $\mathcal{M}(\bm{w}) = - (\mathcal{A}_1(\bm{w}) + \mathcal{A}_2(\bm{w}))^{-1}$ in (\ref{general_gradient_step}), so that the update is of the form:
\begin{framed}
\noindent\underline{Policy search update using the first Gauss-Newton method}
\begin{equation}
\bm{w}^\textnormal{new} = \bm{w} - \alpha (\mathcal{A}_1(\bm{w}) + \mathcal{A}_2(\bm{w}))^{-1} \nabla_{\bm{w}} U(\bm{w}). \label{firstGN_step}
\end{equation}
\end{framed} 
\noindent When the policy parametrization has constant curvature with respect to the action space $ \mathcal{A}_2(\bm{w}) = 0$ and it is sufficient to calculate just 
$\left( \mathcal{A}_1(\bm{w}) \right)^{-1}$.

The second Gauss-Newton method we propose removes further terms from the Hessian which are not guaranteed to be negative definite. As was seen in Section~\ref{sec:policyHessianTheorem}, when the policy parametrization satisfies the properties of Theorem~\ref{NegativeDefiniteLemma} then $\mathcal{H}_2(\bm{w})$ is negative-definite over the entire parameter space. Recall that in (\ref{general_gradient_step}) it is necessary that $\mathcal{M}(\bm{w})$ is positive-definite (in the Newton method this corresponds to requiring the Hessian to be negative-definite) to ensure an increase of the objective function. That $\mathcal{H}_2(\bm{w})$ is negative-definite over the entire parameter space is therefore a highly desirable property of a preconditioning matrix, and for this reason the second Gauss-Newton method that we propose for MDPs is to precondition the gradient with $\mathcal{M}(\bm{w}) = - \mathcal{H}_2(\bm{w})^{-1}$ in (\ref{general_gradient_step}), so that the update is of the form:
\begin{framed}
\noindent\underline{Policy search update using the second Gauss-Newton method}
\begin{equation}
\bm{w}^\textnormal{new} = \bm{w} - \alpha \mathcal{H}_2(\bm{w})^{-1} \nabla_{\bm{w}} U(\bm{w}). \label{secondGN_step}
\end{equation}
\end{framed}
We shall see that the second Gauss-Newton method has important performance guarantees including: a guaranteed ascent direction; linear convergence to a local optimum under a step size which does not depend upon unknown quantities; invariance to affine transformations of the parameter space; and efficient estimation procedures for the preconditioning matrix. We will also show, in Section~\ref{sec:unifying} that the second Gauss-Newton method is closely related to both the EM and natural gradient algorithms.

We shall also consider a diagonal form of the approximation for both forms of Gauss-Newton methods. Denoting the 
diagonal matrix formed from the diagonal elements of $\mathcal{A}_1(\bm{w}) + \mathcal{A}_2(\bm{w})$ and $\mathcal{H}_2(\bm{w})$ by 
$\mathcal{D}_{\mathcal{A}_1 + \mathcal{A}_2}(\bm{w})$ and $\mathcal{D}_{\mathcal{H}_2}(\bm{w})$, respectively, then we shall 
consider the methods that use $\mathcal{M}(\bm{w}) = -\mathcal{D}^{-1}_{\mathcal{A}_1 + \mathcal{A}_2}(\bm{w})$ and 
$\mathcal{M}(\bm{w}) = -\mathcal{D}^{-1}_{\mathcal{H}_2}(\bm{w})$ in (\ref{general_gradient_step}). We call these  
methods the diagonal first and second Gauss-Newton methods, respectively. This diagonalization amounts to performing the approximate Newton methods on each parameter independently, but simultaneously.

\subsubsection{Estimation of the Preconditioners and the Gauss-Newton Update Direction}

It is possible to extend typical techniques used to estimate the policy gradient to estimate the preconditioner for the Gauss-Newton method, by including either the Hessian of the $\log$-policy, the outer product of the derivative of the $\log$-policy, or the respective diagonal terms. As an example, in Section~\ref{sec_supp_recurrent_state_eval} of the Appendix we detail the extension of the recurrent state formulation of gradient evaluation in the average reward framework \citep{williams92} to the second Gauss-Newton method. 
We use this extension in the Tetris experiment that we consider in Section~\ref{sec:experiments}. 
Given $n_{s}$ sampled state-action pairs, the complexity of this extension scales as 
$\mathcal{O}(n_{s} n^2)$ for the second Gauss-Newton method, while it scales as $\mathcal{O}(n_{s} n)$ 
for the diagonal version of the algorithm.

We provide more details of situations in which the inversion of the preconditioning matrices can be performed more efficiently in Section~\ref{HessianInversion} of the Appendix. Finally, for the second Gauss-Newton method the ascent direction can be estimated particularly efficiently, even for large parameter spaces, using a Hessian-free conjugate-gradient approach, which is detailed in Section~\ref{HessianFree} of the Appendix.

\subsection{Performance Guarantees and Analysis}\label{subsec_apxn_method_analysis}

\subsubsection{Ascent Directions}

In general the objective (\ref{ps_objectiveFunction}) is not concave, which means that the Hessian will not be negative-definite over the entire parameter space. In such cases the Newton method can actually lower the objective and this is an undesirable aspect of the Newton method. We now consider ascent directions for the Gauss-Newton methods, and in particular demonstrate that the proposed second Gauss-Newton method guarantees an ascent direction in typical settings.

\paragraph{Ascent directions for the first Gauss-Newton method:} As mentioned previously, the matrix $\mathcal{A}_1(\bm{w}) + \mathcal{A}_2(\bm{w})$ will typically be indefinite, and so a straightforward
application of the first Gauss-Newton method will not necessarily result in an increase in the objective function. There are, however, 
standard correction techniques that one could consider to ensure that an increase in the objective function is obtained, such as 
adding a ridge term to the preconditioning matrix. A survey of such correction techniques can be found in \cite{boyd-vandenberghe-book}.

\paragraph{Ascent directions for the second Gauss-Newton method:} It was seen in Theorem~\ref{NegativeDefiniteLemma} that $\mathcal{H}_2(\bm{w})$ will be negative-definite over the 
entire parameter space if either the policy is $\log$-concave with respect to the policy parameters, or the policy has constant 
curvature with respect to the action space. It follows that in such cases an increase of the objective function will be obtained 
when using the second Gauss-Newton method with a sufficiently small step-size. Additionally, the diagonal terms of a negative-definite matrix are negative, so that 
$\mathcal{D}_{\mathcal{H}_2}(\bm{w})$ is negative-definite whenever $\mathcal{H}_2(\bm{w})$ is negative-definite, and thus similar
performance guarantees exist for the diagonal version of the second Gauss-Newton algorithm. 


To motivate this result we now briefly consider some widely used policies that are either $\log$-concave or blockwise $\log$-concave. 
Firstly, consider the Gibb's policy, $\pi(a|s; \bm{w}) \propto \exp \bm{w}^T \bm{\phi}(a, s) $, in which $\bm{\phi}(a, s) \in \mathbb{R}^n$ 
is a feature vector. This policy is widely used in discrete systems and is $\log$-concave in $\bm{w}$, which can be seen from the fact 
that $\log \pi(a|s; \bm{w})$ is the sum of a linear term and a negative \textit{log-sum-exp} term, both of which are concave 
\citep{boyd-vandenberghe-book}. In systems with a continuous state-action space a common choice of controller is 
$\pi(a|s; K, \Sigma) = \mathcal{N}(a| K \bm{\phi}(s), \Sigma)$, in which $\bm{\phi}(s) \in \mathbb{R}^n$ is a feature vector. This controller is not jointly $\log$-concave in $K$ and $\Sigma$, but it is blockwise $\log$-concave in $K$  and $\Sigma^{-1}$. In terms of 
$K$ the $\log$-policy is quadratic and the coefficient matrix of the quadratic term is negative-definite. In terms of $\Sigma^{-1}$ the 
$\log$-policy consists of a linear term and a $\log$-determinant term, both of which are concave.

\subsubsection{Affine Invariance}

A undesirable aspect of steepest gradient ascent is that its performance is dependent on the choice of basis used to represent the parameter space. An important and desirable property of the Newton method is that it is invariant to non-singular affine transformations of the parameter space \citep{boyd-vandenberghe-book}. This means that given a non-singular affine mapping, $T \in \mathbb{R}^{n \times n}$, the Newton update of the objective $\tilde{U}(\bm{w}) = U(T\bm{w})$ is related to the Newton update of the original objective through the same affine mapping, i.e., $\bm{v} + \Delta \bm{v}_{\textnormal{nt}} = T \big( \bm{w} + \Delta \bm{w}_{\textnormal{nt}} \big)$, in which $\bm{v} = T\bm{w}$ and $\Delta \bm{v}_{\textnormal{nt}}$ and $\Delta \bm{w}_{\textnormal{nt}}$ denote the respective Newton steps. A method is said to be scale invariant if it is invariant to non-singular rescalings of the parameter space. In this case the mapping $T \in \mathbb{R}^{n \times n}$, is given by a non-singular diagonal matrix. The proposed approximate Newton methods have various invariance properties, and these properties are summarized in the following theorem.

\begin{Theorem} \label{AffineInvariance}
The first and second Gauss-Newton methods are invariant to (non-singular) affine transformations of the parameter space. 
The diagonal versions of these algorithms are invariant to (non-singular) rescalings of the parameter space.
\begin{proof}
See Section~\ref{affine_invariance_appendix} in the Appendix.
\end{proof}
\end{Theorem}

\subsubsection{Convergence Analysis}

We now provide a local convergence analysis of the Gauss-Newton framework. 
We shall focus on the full Gauss-Newton methods, with the analysis of the diagonal Gauss-Newton method following similarly. 
Additionally, we shall focus on the case in which a constant step size is considered throughout, which is denoted by 
$\alpha \in \mathbb{R}^+$.
We say that an algorithm converges linearly to a limit $L$ at a rate $r\in(0,1)$ if 
\begin{align}
\lim_{k\to\infty} \frac{|U({\bm w}_{k+1}) - L|}{|U({\bm w}_{k}) - L|} = r. \nn
\end{align}
If $r=0$ then the algorithm converges super-linearly. We denote the parameter update function of the first and second Gauss-Newton methods by $G_1$ and $G_2$, respectively, 
so that $G_1(\bm{w}) = \bm{w} - \alpha (\mathcal{A}_1(\bm{w}) + \mathcal{A}_2(\bm{w}))^{-1} \nabla U(\bm{w})$
and $G_2(\bm{w}) = \bm{w} - \alpha \mathcal{H}_2(\bm{w})^{-1} \nabla U(\bm{w})$. Given a matrix, $A \in L(\mathbb{R}^n)$ we denote the spectral radius of $A$ by $\rho(A) = \max_i | \lambda_i |$, where $\{\lambda_i\}_{i=1}^n$ are the eigenvalues of $A$. Throughout this section we shall use $\nabla G(\bm{w}^*)$ to denote $\nabla_{\bm{w} | \bm{w} = \bm{w}^*} G(\bm{w})$.

\begin{Theorem}[Convergence analysis for the first Gauss-Newton method]\label{FAN_lemma_convergence_lemma}
Suppose that $\bm{w}^* \in \mathcal{W}$ is such that $\nabla_{\bm{w}|\bm{w} = \bm{w}^*} U(\bm{w}) = \bm{0}$ 
and $\mathcal{A}_1(\bm{w}^*) + \mathcal{A}_2(\bm{w}^*)$ is invertible, then $G_1$ is Fr\'{e}chet differentiable at $\bm{w}^*$ 
and $\nabla G_1(\bm{w}^*)$ takes the form,
\begin{align}
\nabla G_1(\bm{w}^*) &= I - \alpha (\mathcal{A}_1(\bm{w}^*) + \mathcal{A}_2(\bm{w}^*))^{-1} \mathcal{H}(\bm{w}^*). \label{G1_Jacobian}
\end{align}
If $\mathcal{H}(\bm{w}^*)$ and $\mathcal{A}_1(\bm{w}^*) + \mathcal{A}_2(\bm{w}^*)$ are negative-definite, and the step size
is in the range,
\begin{equation}
\alpha \in \left(0, 2 / \rho\left((\mathcal{A}_1(\bm{w}^*) + \mathcal{A}_2(\bm{w}^*)) ^{-1}  \mathcal{H}(\bm{w}^*)\right) \right) \label{permissable_step_sizes_1GN}
\end{equation}
then $\bm{w}^*$ is a point of attraction of the first Gauss-Newton method, the convergence is at least linear and the rate is given by $\rho(\nabla G_1(\bm{w}^*)) < 1$.
When the policy parametrization is value consistent with respect to the given Markov Decision Process, then (\ref{G1_Jacobian}) simplifies to
\begin{align}
\nabla G_1(\bm{w}^*) &= (1 - \alpha) I, \label{G1_Jacobian_value_consistent}
\end{align}
and whenever $\alpha\in(0,2)$ then $\bm{w}^*$ is a point of attraction of the first Gauss-Newton method, and the convergence to $\bm{w}^*$ is linear if $\alpha \ne 1$ with a rate given by $\rho(\nabla G_1(\bm{w}^*))<1$, and convergence is super-linear when $\alpha=1$.
\begin{proof}
See Section~\ref{convergence_analysis_appendix} in the Appendix.
\end{proof}
\end{Theorem}

Additionally we make the following remarks for the case when the policy parametrization is not value consistent with respect to the given Markov decision process. For simplicity, we shall consider the case in which $\alpha = 1$. In this case $\nabla G_1(\bm{w}^*)$ takes the form,
\begin{equation*}
\nabla G_1(\bm{w}^*) = -(\mathcal{A}_1(\bm{w}^*) + \mathcal{A}_2(\bm{w}^*))^{-1} \bigg( \mathcal{H}_{12}(\bm{w}^*) + \mathcal{H}_{12}^\top(\bm{w}^*) \bigg).
\end{equation*}
From the analysis in Section~\ref{hessian_analysis_section_local_optimum} we expect that when the policy parametrization is rich, but not value consistent with respect to the given Markov decision process, that $\rho ( \big( \mathcal{H}_{12}(\bm{w}^*) + \mathcal{H}_{12}(\bm{w}^*)^\top \big)^{-1} \big( \mathcal{A}_1(\bm{w}^*) + \mathcal{A}_2(\bm{w}^*) \big) )$ will generally be small. In this case the first Gauss-Newton method will converge linearly, and the rate of convergence will be close to zero. 


\begin{Theorem}[Convergence analysis for the second Gauss-Newton method]\label{SAN_lemma_convergence_lemma}
Suppose that $\bm{w}^* \in \mathcal{W}$ is such that $\nabla_{\bm{w}|\bm{w} = \bm{w}^*} U(\bm{w}) = \bm{0}$ 
and $\mathcal{H}_2(\bm{w}^*)$ is invertible, then $G_2$ is Fr\'{e}chet differentiable at $\bm{w}^*$ 
and $\nabla G_2(\bm{w}^*)$ takes the form,
\begin{align}
\nabla G_2(\bm{w}^*) &= I - \alpha \mathcal{H}^{-1}_2(\bm{w}^*) \mathcal{H}(\bm{w}^*). \label{G2_Jacobian}
\end{align}
If $\mathcal{H}(\bm{w}^*)$ is negative-definite and the step size
is in the range,
\begin{equation}
\alpha \in (0, 2 /\rho(\mathcal{H}_2(\bm{w}^* )^{-1}\mathcal{H}(\bm{w}^*)  ) ) \label{permissable_step_sizes_2GN}
\end{equation}
then  $\bm{w}^*$ is a point of attraction of the second Gauss-Newton method, convergence to $\bm{w}^*$ is at least linear and the rate is given by $\rho(\nabla G_2(\bm{w}^*)) < 1$. Furthermore, $\alpha \in (0,2)$ implies  condition (\ref{permissable_step_sizes_2GN}).
When the policy parametrization is value consistent with respect to the given Markov decision process, then (\ref{G2_Jacobian}) simplifies to
\begin{align}
\nabla G_2(\bm{w}^*) &= I - \alpha \mathcal{H}_2^{-1}(\bm{w}^*) \mathcal{A}_1(\bm{w}^*). \label{G2_Jacobian_value_consistent}
\end{align}
\begin{proof}
See Section~\ref{convergence_analysis_appendix} in the Appendix.
\end{proof}
\end{Theorem}

The conditions of Theorem~\ref{SAN_lemma_convergence_lemma} look analogous to those of Theorem~\ref{FAN_lemma_convergence_lemma}, but they differ in important ways: it is not necessary to assume that the 
preconditioning matrix is negative-definite and the sets in (\ref{permissable_step_sizes_1GN}) will 
not be known in practice, whereas the condition $\alpha \in (0,2)$ in Theorem~\ref{SAN_lemma_convergence_lemma} is more practical, i.e., for the second Gauss-Newton method convergence is guaranteed for a constant step size which is easily selected and does not depend upon unknown quantities.

It will be seen in Section~\ref{compareGaussNewton2EM} that the second Gauss-Newton method has a close relationship
to the EM-algorithm. For this reason we postpone additional discussion about the rate of convergence of the second 
Gauss-Newton method until then.

\section{Relation to Existing Policy Search Methods} \label{sec:unifying}

In this section we detail the relationship between the second Gauss-Newton method and existing policy search methods; In Section~\ref{compareGaussNewton2NaturalGradients} we detail the relationship with natural gradient ascent and in Section~\ref{compareGaussNewton2EM} we detail the relationship with the EM-algorithm.

\subsection{Natural Gradient Ascent and the Second Gauss-Newton Method}\label{compareGaussNewton2NaturalGradients}

Comparing the form of the Fisher information matrix given in (\ref{fisher_information}) with $\mathcal{H}_2$ (\ref{pg_hessian_H2}) it can be seen that there is a close relationship between natural gradient ascent and the second Gauss-Newton method: in $\mathcal{H}_2$ there is an additional weighting of the integrand from the state-action value function. Hence, $\mathcal{H}_2$ incorporates information about the reward structure of the objective 
function that is not present in the Fisher information matrix.

We now consider how this additional weighting affects the search direction for natural gradient ascent and the Gauss-Newton approach. Given a norm on the parameter space, $||\cdot||$, the steepest ascent direction at $\bm{w} \in \mathcal{W}$ with respect to that 
norm is given by,
\begin{align}
\hat{\bm{p}} = \argmax{\{\bm{p}:||\bm{p}||=1\}} \lim_{\alpha \to 0}  \frac{U(\bm{w} + \alpha \bm{p}) - U(\bm{w})}{\alpha}.  \nn
\end{align}
Natural gradient ascent is obtained by considering the (local) norm $||\cdot||_{G(\bm{w})}$ given by
\begin{align}
||\bm{w} - \bm{w'}||^2_{G(\bm{w})}:=(\bm{w} - \bm{w'})^\top G(\bm{w}) (\bm{w} - \bm{w'}), \nn
\end{align}
with $G(\bm{w})$ as in (\ref{fisher_information2}).
The natural gradient method allows less movement in the directions that have high norm which, as can be seen from the form of (\ref{fisher_information2}), are those directions that induce large changes to the policy over the parts of the state-action space that are likely to be visited under the current policy parameters. More movement is allowed in directions that either induce a small change in the policy, or induce large changes to the policy, but only in parts of the state-action space that are unlikely to be visited under the current policy parameters. In a similar manner the second Gauss-Newton method can be obtained by considering the (local) norm $||\cdot||_{\mathcal{H}_2(\bm{w})}$,
\begin{align}
||\bm{w} - \bm{w'}||^2_{\mathcal{H}_2(\bm{w})}:= - (\bm{w} - \bm{w'})^\top \mathcal{H}_2(\bm{w}) (\bm{w} - \bm{w'}), \nn
\end{align}
so that each term in (\ref{fisher_information}) is additionally weighted by the state-action value function, $Q(s,a;\bm{w})$. Thus, the directions which have high norm are those in which the policy is rapidly changing 
in state-action pairs that are not only likely to be visited under the current policy, but also have high value. 
Thus the second Gauss-Newton method updates the parameters more carefully if the behaviour in high value states 
is affected. Conversely, directions which induce a change only in state-action pairs of low value have low norm, 
and larger increments can be made in those directions.

\subsection{Expectation Maximization and the Second Gauss-Newton Method}\label{compareGaussNewton2EM}

It has previously been noted \citep{Kober-peters-2011} that the parameter update of steepest gradient ascent and the EM-algorithm can be related through the function $\mathcal{Q}$ defined in (\ref{pg_general_em}). In particular, the gradient (\ref{pg_gradient}) evaluated at $\bm{w}_k$ can be written in terms of $\mathcal{Q}$ as follows, 
\begin{equation*}
\nabla_{\bm{w}|\bm{w} = \bm{w}_k} U(\bm{w}) = \nabla_{\bm{w}|\bm{w} = \bm{w}_k} \mathcal{Q}(\bm{w}, \bm{w}_k), 
\end{equation*}
while the parameter update of the EM-algorithm is given by,
\begin{equation*}
\bm{w}_{k+1} = \argmax{\bm{w}\in\mathcal{W}} \mathcal{Q}(\bm{w}, \bm{w}_k).
\end{equation*}
In other words, steepest gradient ascent moves in the direction that most rapidly increases $\mathcal{Q}$ with respect to the first variable, while the EM-algorithm maximizes $\mathcal{Q}$ with respect to the first variable. While this relationship is true, it is also quite a negative result. It states that in situations in which it is not possible to explicitly maximize $\mathcal{Q}$ with respect to its first variable, then the alternative, in terms of the EM-algorithm, is a generalized EM-algorithm, which is equivalent to steepest gradient ascent. Given that the EM-algorithm is typically used to overcome the negative aspects of steepest gradient ascent, this is an undesirable alternative. It is possible to find the optimum of (\ref{pg_general_em}) numerically, but this is also undesirable as it results in a double-loop algorithm that could be computationally expensive. Finally, this result provides no insight into the behaviour of the EM-algorithm, in terms of the direction of its parameter update, when the maximization over $\bm{w}$ in (\ref{pg_general_em}) can be performed explicitly.

We now demonstrate that the step-direction of the EM-algorithm has an underlying relationship with the second of our proposed Gauss-Newton methods. In particular, we show that under suitable regularity conditions the direction of the EM-update, i.e., $\bm{w}_{k+1} - \bm{w}_k$, is the same, up to first order, as the direction of the second Gauss-Newton method that uses $\mathcal{H}_2(\bm{w})$ in place of $\mathcal{H}(\bm{w})$. 
 
\begin{Theorem}\label{theorem:analysis_em_generalized_gradient}
Suppose we are given a Markov decision process with objective (\ref{ps_objectiveFunction}) and Markovian trajectory distribution (\ref{trajectory_prob_parametric}). Consider the parameter update (M-step) of Expectation Maximization at the $k^{\textnormal{th}}$ iteration of the algorithm, i.e.,
\begin{equation*}
\bm{w}_{k+1} = \argmax{\bm{w}\in\mathcal{W}} \mathcal{Q}(\bm{w}, \bm{w}_k).
\end{equation*}
Provided that $\mathcal{Q}(\bm{w}, \bm{w}_k)$ is twice continuously differentiable in the first parameter we have that 
\begin{equation}
\bm{w}_{k+1} - \bm{w}_k = - \mathcal{H}_2^{-1}(\bm{w}_k) \nabla_{\bm{w}|\bm{w} = \bm{w}_k} U(\bm{w}) + \mathcal{O}(\|\bm{w}_{k+1} - \bm{w}_k\|^2). \label{apxn_to_EM_relation}
\end{equation}
Additionally, in the case where the $\log$-policy is quadratic the relation to the second Gauss-Newton method is exact, i.e., the second term on the \rhs of (\ref{apxn_to_EM_relation}) is zero.
\begin{proof}
See Section~\ref{EMconnectionProof} in the Appendix.
\end{proof}
\end{Theorem}

Given a sequence of parameter vectors, $\{\bm{w}_k\}_{k=1}^\infty$, generated through an application of the EM-algorithm, then $\lim_{k \to \infty} \| \bm{w}_{k+1} - \bm{w}_k \| = 0$. 
This means that the rate of convergence of the EM-algorithm will be the same as that of the second Gauss-Newton method when considering a constant step size of one. 
We formalize this intuition and provide the convergence properties of the EM-algorithm when applied to Markov decision processes in the following theorem. This is, to our knowledge, the first formal derivation of the convergence properties for this application of the EM-algorithm.   

\begin{Theorem}\label{EM_lemma_convergence_lemma}
Suppose that the sequence, $\{\bm{w}_k\}_{k \in \mathbb{N}}$, is generated by an application of the EM-algorithm, where the sequence converges to $\bm{w}^*$. Denoting the update operation of the EM-algorithm by $G_{\textnormal{EM}}$, so that $\bm{w}_{k+1} = G_{\textnormal{EM}}(\bm{w}_k)$, then
\begin{align}
\nabla G_{\textnormal{EM}}(\bm{w}^*) &= I - \mathcal{H}^{-1}_2(\bm{w}^*) \mathcal{H}(\bm{w}^*). \nn
\end{align}
When the policy parametrization is value consistent with respect to the given Markov Decision Process this simplifies to $\nabla G_{\textnormal{EM}}(\bm{w}^*) = I - \mathcal{H}_2(\bm{w}^*)^{-1} \mathcal{A}_1(\bm{w}^*)$. When the Hessian, $\mathcal{H}(\bm{w}^*)$, is negative-definite then $\rho(\nabla G_{\textnormal{EM}}(\bm{w}^*)) < 1$ and $\bm{w}^*$ is a local point of a attraction for the EM-algorithm.
\begin{proof}
See Section~\ref{ConvEM_appendix} in the Appendix.
\end{proof}
\end{Theorem}

\section{Experiments}\label{sec:experiments}

In this section we provide an empirical evaluation of the Gauss-Newton methods on a varied set of challenging domains.

\subsection{Affine Invariance Experiment}

In the first experiment we give an empirical illustration that the full Gauss-Newton methods are invariant to affine transformations of the parameter space. Additionally, we illustrate that the diagonal Gauss-Newton methods are invariant to (non-zero) rescalings of the dimensions of the parameter space. We consider the simple two state example of \cite{Kakade-2002}. In this example problem the policy has only two parameters, so that it is possible to plot the trace of the policy during training. The policy is trained using steepest gradient ascent, the full Gauss-Newton methods and the diagonal Gauss-Newton methods. We train the policy in both the original and linearly transformed parameter space. The policy traces of the various algorithms are given in Figure~\ref{fig:invariance_experiment}. As expected steepest gradient ascent is affected by both forms of transformation, while the diagonal Gauss-Newton methods are invariant to diagonal rescalings of the parameter space, and the full Gauss-Newton methods are invariant to both forms of transformation. 

\begin{figure}[t]
\captionsetup[subfigure]{justification=centering}
\begin{minipage}[b]{.5\linewidth}
\centering \includegraphics[height=2.5in,width=2.75in]{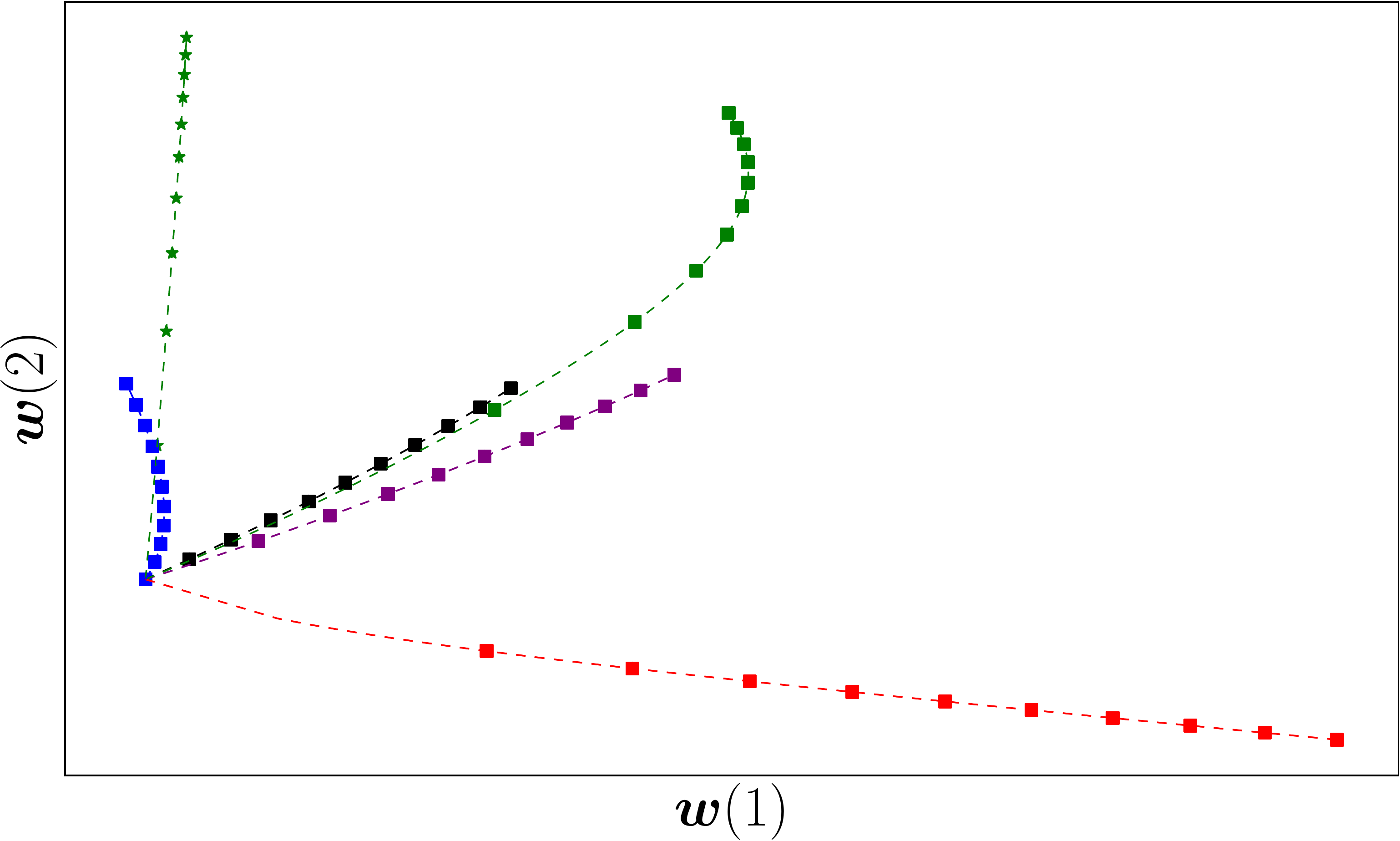}
\subcaption{Scale Invariance Experiment}\label{fig:scale_invariance_experiment}
\end{minipage}%
\begin{minipage}[b]{.5\linewidth}
\centering \includegraphics[height=2.5in,width=2.75in]{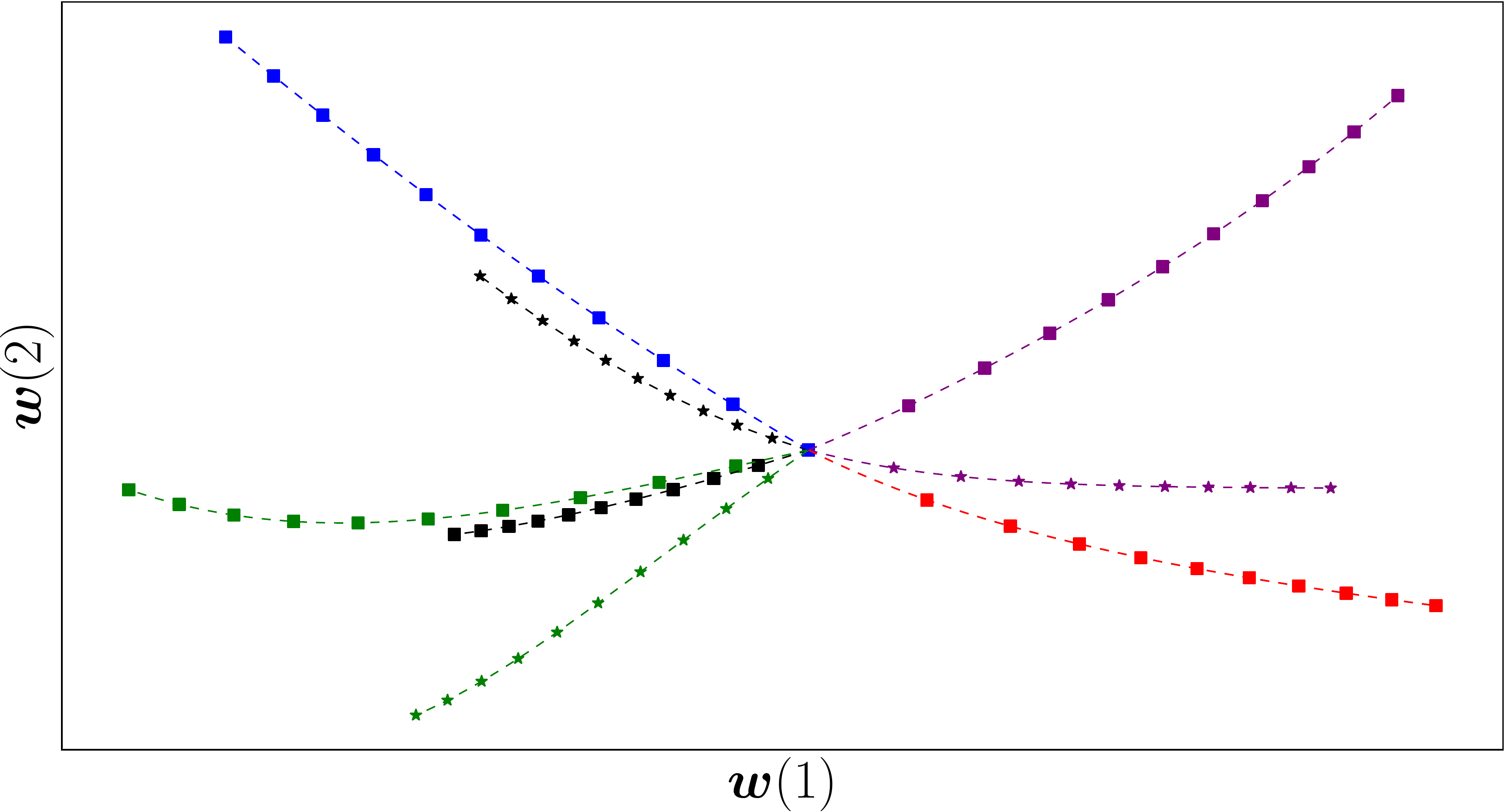}
\subcaption{Affine Invariance Experiment}\label{fig:affine_invariance_experiment}
\end{minipage}%
\newline
\begin{minipage}[b]{1.0\linewidth}
\centering \includegraphics[height=1.0in,width=5.25in]{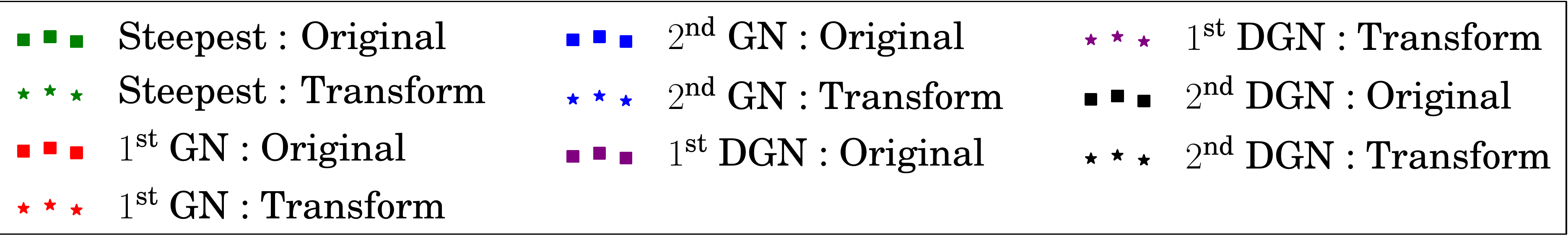}
\end{minipage}%
\caption{Results from (a) the scale invariance experiment and (b) the affine invariance experiment. The plots show the trace of the policy through the parameter space during the course of training. The plots give the trace of the policy when trained in the original parameter space (square markers), and when trained in the transformed parameter space (star markers). For comparison, the policy traces in the transformed parameter space have been mapped back to the original space. The plots show the trace of the policy when the policy is trained with steepest gradient ascent (green), the first Gauss-Newton method (red), the second Gauss-Newton method (blue), the first diagonal Gauss-Newton method (purple) and the second diagonal Gauss-Newton method (black).}\label{fig:invariance_experiment}
\end{figure}

\subsection{Cart-Pole Swing-Up Benchmark Experiment}

We also implemented the Gauss-Newton methods on the standard simulated Cart-pole benchmark problem. This problem involves a pole attached at a pivot to a cart, and by applying force to the cart the pole must be swung to the vertical position and balanced. The problem is under-actuated in the sense that insufficient power is available to drive the pole directly to the vertical position hence the problem captures the notion of trading off immediate reward for long term gain. In this episodic experiment we used an actor-critic architecture \citep{DBLP:conf/nips/KondaT99} using compatible features to fit the Q-function.

We used the same simulator as \citet{DBLP:journals/jmlr/LagoudakisP03}, except here we allow continuous actions and choose a continuous reward signal. The state space is two dimensional, $s = (\theta, \dot{\theta})$ representing the angle ($\theta = 0$ when the pole is pointing vertically upwards) and angular velocity of the pole. The action space is $\myset{A} = [-50,50]$ representing the horizontal force in Newtons applied to the cart (i.e., any actions of greater magnitude returned by the controller are clipped at $\pm 50$). Uniform noise in $[-10,10]$ is added to each action (before clipping). The system dynamics are $\theta_{t+1} = \theta_t + \Delta_t \dot {\theta_t}$, $\dot\theta_{t+1} = \dot\theta_t + \Delta_t \ddot {\theta_t}$ where
$$ \ddot {\theta} = \frac{g \sin(\theta) - \alpha m \ell (\dot{\theta})^2 \sin(2 \theta)/2 - \alpha \cos(\theta) u}{ 4\ell / 3 - \alpha m \ell \cos^2(\theta)}, $$
where $g=9.8m/s^2$ is the acceleration due to gravity, $m=2kg$ is the mass of the pole, $M = 8kg$ is the mass of the cart, $\ell = 0.5m$ is the length of the pole and $\alpha = 1/(m+M)$. We choose $\Delta_t = 0.1s$. Rewards $R(s,a) = \frac{1+ \cos(\theta)}{2}$, the discount factor is $\gamma = 0.99$, the horizon is $H=100$, and the pole begins in the downwards position, $s_0 = (\pi,0)$.

The controller is a Gaussian,
\begin{equation*}
\pi(a|s; \bm{w}) = \mathcal{N}(a|{\bm \phi}(s)^\top {\bm w}, \sigma^2),
\end{equation*}
with radial basis features, $\phi_i(s) = \exp{\frac{1}{2}(c_i - s)^\top \Lambda (c_i - s)}$. For each separate experiment the 100 centers $c_i$ were drawn uniformly at random from $[-\pi,\pi]\times[-4\pi,4\pi]$, the bandwidth was fixed $\Lambda =\left( \begin{array}{cc}
1 & 0 \\
0 & 1/4 \\ \end{array} \right)$ and the policy noise $\sigma$ was fixed at $2$ (these parameters were found by an informal search). Controller weights ${\bm w_0}$ were initialized randomly for each experiment.

The policy was updated after every 10 trajectories, i.e., each iteration corresponds to 10 episodes of experience. Of these, 5 trajectories were used to estimate the policy gradient and the preconditioning matrix, while the remaining 5 trajectories were used to learn an approximation $\hat Q(s,a;{\bm w}) = {\bm\psi}(s,a;{\bm w})^\top {\bm \theta}$ to the Q-function $Q(s,a;{\bm w})$ using the compatible features \citep{Kakade-2002},
\begin{align}
{\bm\psi}(s,a;{\bm w}) = \nabla _{\bm w}\log \pi(a|s;{\bm w}) = \frac{1}{\sigma^2} (a-{\bm \phi}(s)^\top {\bm w}){\bm \phi}(s). \nn
\end{align}
The weight vector $\bm \theta$ was learnt using least-squares linear regression. For each $(s_t,a_t)$ in an experienced trajectory the targets were provided by Monte-Carlo roll-out estimates
$$ Q(s_t,a_t;{\bm w}) \approx \sum_{\tau = 1}^H \gamma^{\tau-1} R(s_{t+\tau-1},a_{t+\tau-1}). $$
Note that each trajectory was therefore simulated for a length $2H$, rather than $H$, in order to gather the target data. A regularization parameter was validated on a held out subset of the data.

We compared 5 algorithms: steepest ascent, {\it{`Steepest'}}, (\ref{steepest_gradient_step}); the natural gradients algorithm, {\it{`Natural'}}, (\ref{natural_gradients_step}) with preconditioner ${\mathcal M}({\bm w}) = G({\bm w})^{-1}$; compatible natural gradients, {\it{`Comp Natural'}}, in which the policy parameter is updated in the direction ${\bm \theta}$ of the Q-function weight vector \citep{Kakade-2002}; the first Gauss-Newton method, {\it{`First G-N'}}, (\ref{firstGN_step}) using ${\mathcal M}({\bm w}) = -{\mathcal A}_1({\bm w})^{-1}$; the second Gauss-Newton method, {\it{`Second G-N'}}, (\ref{secondGN_step}) using ${\mathcal M}({\bm w}) = -{\mathcal H}_2({\bm w})^{-1}$. To precondition the gradients we solved the required linear systems using steepest descent using the gradient as a warm start, for a maximum of 250 iterations, rather than direct inversion. This was found to be more stable in this experiment than inversion of the preconditioning matrices for all methods since the Fisher information matrix and the (approximate) Hessians can be poorly conditioned: for example when the policy trajectories are supported entirely on a region of space in which some features are never active, neither the gradient, Hessian or Fisher information matrix will have any components corresponding to those feature dimensions.

We used a step size of $\alpha_t = \frac{A}{1+t/100}$ i.e.,
$$ {\bm w}_{t+1}  = {\bm w}_{t} + \alpha_t d({\bm w}_{t}) $$
where $d({\bm w}_{t})$ is the search direction at iteration $t$. We ran the experiment 20 times over a range $A\in \{ 1/4 , 1/2 ,1,2,4,..., 512, 1024, 2048 \}$ to choose the best step size for each method. The experiments were then run 50 times for the best step size to get the unbiased estimate of performance for that step size, which we report. After each policy update we estimated the cumulative reward of the policy (this requires no additional data, since the data used to estimate the return is exactly the data used to estimate the Q-function) and if the return was found to have decreased we returned to the previous parameter point. This simple heuristic (a 2-point line search) prevents variance in the gradient estimates from causing policy degradation and instability.

Figure~\ref{CartPoleFig} shows the cumulative reward after each iteration for the 5 methods along with the standard error. Cumulative reward of 50 is a near optimal policy in which the pole is quickly swung up and balanced for the entire episode. Cumulative reward of 40 to 45 indicates that the pole is swung up and balanced, but either not optimally quickly, or that the controller is unable to balance the pole for the entire remainder of the episode. The Gauss-Newton methods significantly outperform all competitor methods both in terms of the speed at which good policies are learned and the average value of the policy at convergence. Furthermore, as predicted by theory, a step-size of 1 for the  Gauss-Newton methods was found to perform well; i.e., good performance could be obtained without step-size tuning.

\subsection{Non-Linear Navigation Experiment}

The next domain that we consider is the synthetic two-dimensional non-linear MDP considered in \cite{vlassis-etal-2009}. 
The state-space of the problem is two-dimensional, $\bm{s} = (s^1, s^2)$, in which $s^1$ is the agent's position and $s^2$ is the agent's velocity. The control is one-dimensional and the dynamics of the system is given as follows,
\begin{align*}
s^1_{t+1} &= s^1_t + \frac{1}{1 + e^{-u_t}} - 0.5 + \kappa, \\
s^2_{t+1} &= s^2_t - 0.1 s^1_{t+1} + \kappa,
\end{align*}
with $\kappa$ a zero-mean Gaussian random variable with standard deviation $\sigma_{\kappa} = 0.02$. The agent starts in the state $\bm{s} = (0,1)$, with the addition of Gaussian noise with standard deviation $0.001$, and the objective is for the agent to reach the target state, $\bm{s}_{\textnormal{target}} = (0,0)$. We use the same policy as in \cite{vlassis-etal-2009}, which is given by $a_t = (\bm{w} + \bm{\epsilon}_t)^\top \bm{s}_t$, with control parameters, $\bm{w}$, and $\bm{\epsilon}_t \sim \mathcal{N}(\bm{\epsilon}_t; \bm{0}, \sigma_{\epsilon}^2 I)$. The objective function is non-trivial for $\bm{w} \in [0, 60] \times [-8, 0]$. In the experiment the initial control parameters were sampled from the region $\bm{w}_0 \in [0, 60] \times [-8, 0]$. In all algorithms $50$ trajectories were sampled during each training iteration and used to estimate the search direction. We consider a finite planning horizon, $H = 80$. The experiment was repeated $100$ times and the results of the experiment are given in Figure~\ref{fig:vlassis_main_experiment}, which gives the mean and standard error of the results. The step size sequences of steepest gradient ascent, natural gradient ascent and the Gauss-Newton method were all tuned for performance and the results shown were obtained from the best step size sequence for each algorithm. 

\begin{figure*}[t!]
    \centering
    \begin{subfigure}[t]{0.5\textwidth}
        \centering
	\centering \includegraphics[height=2.3in]{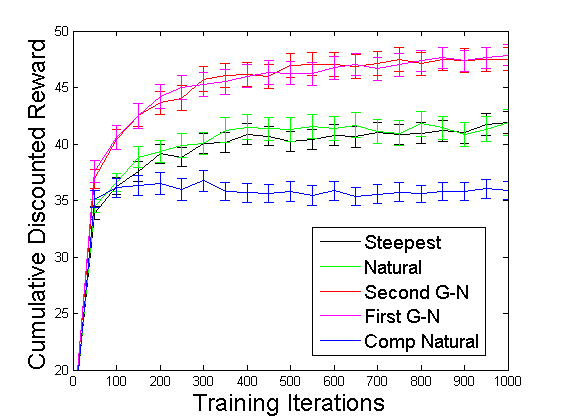}
	\caption{Cart-Pole Experiment: Results}\label{CartPoleFig}
    \end{subfigure}%
    ~ 
    \begin{subfigure}[t]{0.5\textwidth}
        \centering
        \includegraphics[height=2.2in]{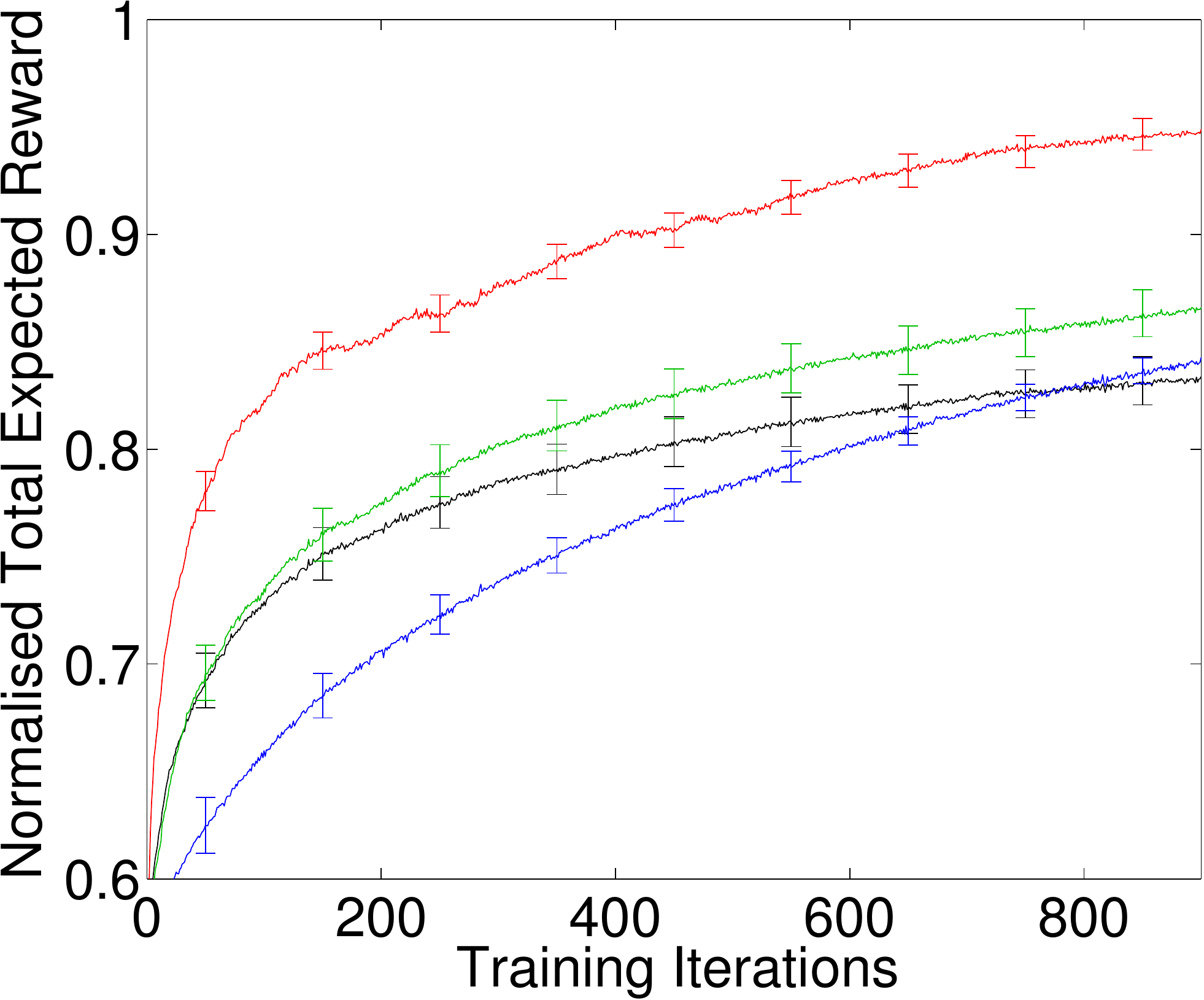}
        \caption{Non-Linear Navigation Task : Results}\label{fig:vlassis_main_experiment}
    \end{subfigure}
    \caption{(a) Results from the cart-pole experiment. (b) Results from the non-linear navigation task, 
 with the results for steepest gradient ascent (black), Expectation Maximization (blue), natural gradient ascent (green) and the Gauss-Newton method (red).}
\end{figure*}

\subsection{$N$-link Rigid Manipulator Experiments}

The $N$-link rigid robot arm manipulator is a standard continuous model, consisting of an end effector connected to an $N$-linked rigid body \citep{khalil-2001}. A graphical depiction of a 3-link rigid manipulator is given in Figure \ref{fig:3_link_rigid_manipulator}. A typical continuous control problem for such systems is to apply appropriate torque forces to the joints of the manipulator so as to move the end effector into a desired position. The state of the system is given by $\bm{q}$, $\dot{\bm{q}}$, $\ddot{\bm{q}} \in \mathbb{R}^N$, where  $\bm{q}$, $\dot{\bm{q}}$ and $\ddot{\bm{q}}$ denote the angles, velocities and accelerations of the joints respectively, while the control variables are the torques applied to the joints $\bm{\tau} \in \mathbb{R}^N$. The nonlinear state equations of the system are given by \citep{spong-etal-2005},
\begin{equation}
M(\bm{q})\ddot{\bm{q}} + C(\dot{\bm{q}}, \bm{q}) \dot{\bm{q}} + g(\bm{q}) = \bm{\tau},
\end{equation}
where $M(\bm{q})$ is the inertia matrix, $C(\dot{\bm{q}}, \bm{q})$ denotes the Coriolis and centripetal forces and $g(\bm{q})$ is the gravitational force. While this system is highly nonlinear it is possible to define an appropriate control function $\hat{\bm{\tau}}(\bm{q}, \dot{\bm{q}})$ that results in linear dynamics in a different state-action space. This technique is known as feedback linearisation \citep{khalil-2001}, and in the case of an $N$-link rigid manipulator recasts the torque action space into the acceleration action space. This means that the state of the system is now given by $\bm{q}$ and $\dot{\bm{q}}$, while the control is $\bm{a} = \ddot{\bm{q}}$. Ordinarily in such problems the reward would be a function of the generalized co-ordinates of the end effector, which results in a non-trivial reward function in terms of $\bm{q}$, $\dot{\bm{q}}$ and $\ddot{\bm{q}}$. This can be 
accounted for by modelling the reward function as a mixture of Gaussians \citep{hoffman-etal-09}, but for simplicity we consider the simpler problem where the reward is a function of $\bm{q}$, $\dot{\bm{q}}$ and $\ddot{\bm{q}}$ directly. 
In all of the experiments in this section we consider a $3$-link rigid manipulator. 

Under certain forms of policy parametrization it is possible to perform exact evaluation of the search direction in these systems. 
As such, these systems allow for the direct comparison of the search direction of various policy search algorithms, but yet are sufficiently difficult optimization problems to provide a challenging platform for these methods. 
In all experiments we consider a policy of the form,
\begin{equation*}
\pi(a|s; \bm{w}) = \mathcal{N}(a|Ks + \bm{m}, \sigma^2 I),
\end{equation*}
with $\bm{w} = (K, \bm{m}, \sigma)$ and $s \in \mathbb{R}^{n_s}$, $a \in \mathbb{R}^{n_a}$, for some $n_s, n_a \in \mathbb{N}$. 
We consider the finite horizon undiscounted problem in this section, so that the gradient of the objective function takes the form
\begin{equation*}
\nabla_{\bm{w}} U(\bm{w}) = \int \int ds da \nabla_{\bm{w}} \log \pi(a|s; \bm{w}) \sum_{t=1}^H p_t(s, a; \bm{w}) Q(s,a, t; \bm{w}),
\end{equation*}
with the preconditioning matrices of natural gradient ascent and the Gauss-Newton methods taking analogous forms.
For any $(s, a) \in \mathcal{S} \times \mathcal{A}$, it can be shown that the derivative of $\pi(a|s; \bm{w})$ is a quadratic in $(s, a)$. This means that to calculate the search directions of steepest gradient ascent, natural gradient ascent, Expectation Maximization and the Gauss-Newton methods it is necessary to calculate the first two moments of 
$p_t(s, a; \bm{w}) Q(s, a, t; \bm{w})$ w.r.t. $(s, a)$, for each $t \in \mathbb{N}_H$.
These calculations can be done using the methods presented in \cite{furmston-thesis}. In these experiments the maximal value of the objective function varied dramatically depending on the random initialization of the system. To account for the variation in the maximal value of the objective function the results of each experiment are normalized by the maximal value achieved between the algorithms for that experiment so that the result displayed is the percentage of reward received in comparison to the best results among the algorithms considered in the experiment.

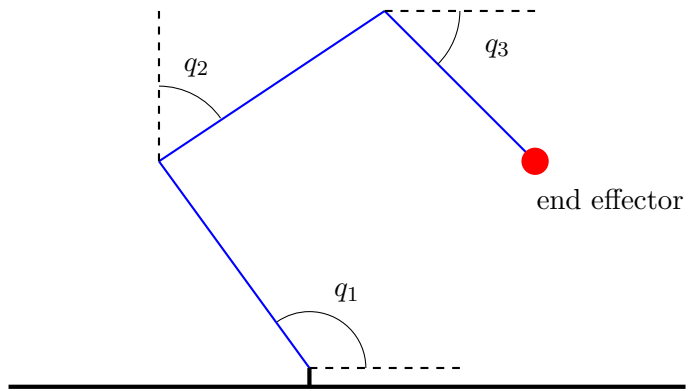
\begin{figure}[t]
\centering
\begin{tikzpicture}
\draw[ultra thick] (-1,1) -- (8,1);
\draw[ultra thick] (3,1) -- (3,1.25);
\draw[blue,thick] (3,1.25) -- (1,4);
\draw[blue,thick] (1,4) -- (4,6);
\draw[blue,thick] (4,6) -- (6,4);
\draw[dashed, thick] (3, 1.25) -- (5,1.25);
\draw[dashed, thick] (1,4) -- (1,6);
\draw[dashed, thick] (4,6) -- (6,6);
\draw[thin] (3.75, 1.25) arc (0:125:0.75cm);
\draw[thin] (1,5) arc (90:35:1cm);
\draw[thin] (5,6) arc (0:-45:1cm);
\node(q_1) at (3.5,2.25) {$q_1$};
\node(q_2) at (1.5,5.25) {$q_2$};
\node(q_3) at (5.5,5.5) {$q_3$};
\node [circle,double,fill=red!100] (label) at (6,4) {};
\node(end_effector) at (7,3.5) {end effector};
\end{tikzpicture}
\caption[A Graphical Depiction of the 3-link Rigid Manipulator.]{A graphical depiction of a 3-link rigid manipulator, with the angles of the joints given by $q_1$, $q_2$ and $q_3$ respectively.} \label{fig:3_link_rigid_manipulator}
\end{figure}

\subsubsection{Experiment Using Line Search}

In the first experiment we compare the search direction of steepest gradient ascent, natural gradient ascent, Expectation Maximization and the second Gauss-Newton method.
For all algorithms that required the specification of a step size we use the \texttt{minFunc}\footnote{This software library is freely available at \url{http://www.di.ens.fr/~mschmidt/Software/minFunc.html}.} optimization library to perform a line search. We also use the \texttt{minFunc} library to provide a stopping criterion for all algorithms. We found that both the line search algorithm and the step size initialization had a significant effect on the performance of all algorithms. We therefore tried various combinations of these settings for each algorithm and selected the one that gave the best performance. We tried bracketing line search algorithms with: step size halving; quadratic/cubic interpolation from new function values; cubic interpolation from new function and gradient values; step size doubling and bisection; cubic interpolation/extrapolation with function and gradient values. We tried the following step size initializations: quadratic initialization using previous function value, and new function value and gradient; twice the previous step size. To handle situations where the initial policy parametrization was in a `flat' area of the parameter space far from any optima we set the function and point toleration of \texttt{minFunc} to zero for all algorithms. We repeated each experiment $100$ times and the results are shown in Figure \ref{fig:3_link_line_search_results}. The second Gauss-Newton method significantly outperforms all of the comparison algorithms. The step direction of Expectation Maximization is very similar to the search direction of the second Gauss-Newton method in this problem. In fact, given that the $\log$-policy is quadratic in the mean parameters, they are the same for the mean parameters. The difference in performance between the Gauss-Newton method and Expectation Maximization is largely explained by the tuning of the step size in the Gauss-Newton method, compared to the constant step size of one in Expectation Maximization. To observe the effect of poor scaling on the performance of the various algorithms we observe the number of iterations that each algorithm requires. These counts are given in table \ref{tab:linear_system_iteration_counts}. Steepest gradient ascent required far more iterations than either natural gradient ascent or the Gauss-Newton method, both of which require roughly the same amount of iterations. This validates that both natural gradient ascent and the Gauss-Newton method are more robust to poor scaling than steepest gradient ascent. 

\begin{figure*}[t!]
    \centering
    \begin{subfigure}[t]{0.5\textwidth}
        \centering
        \includegraphics[height=2.2in]{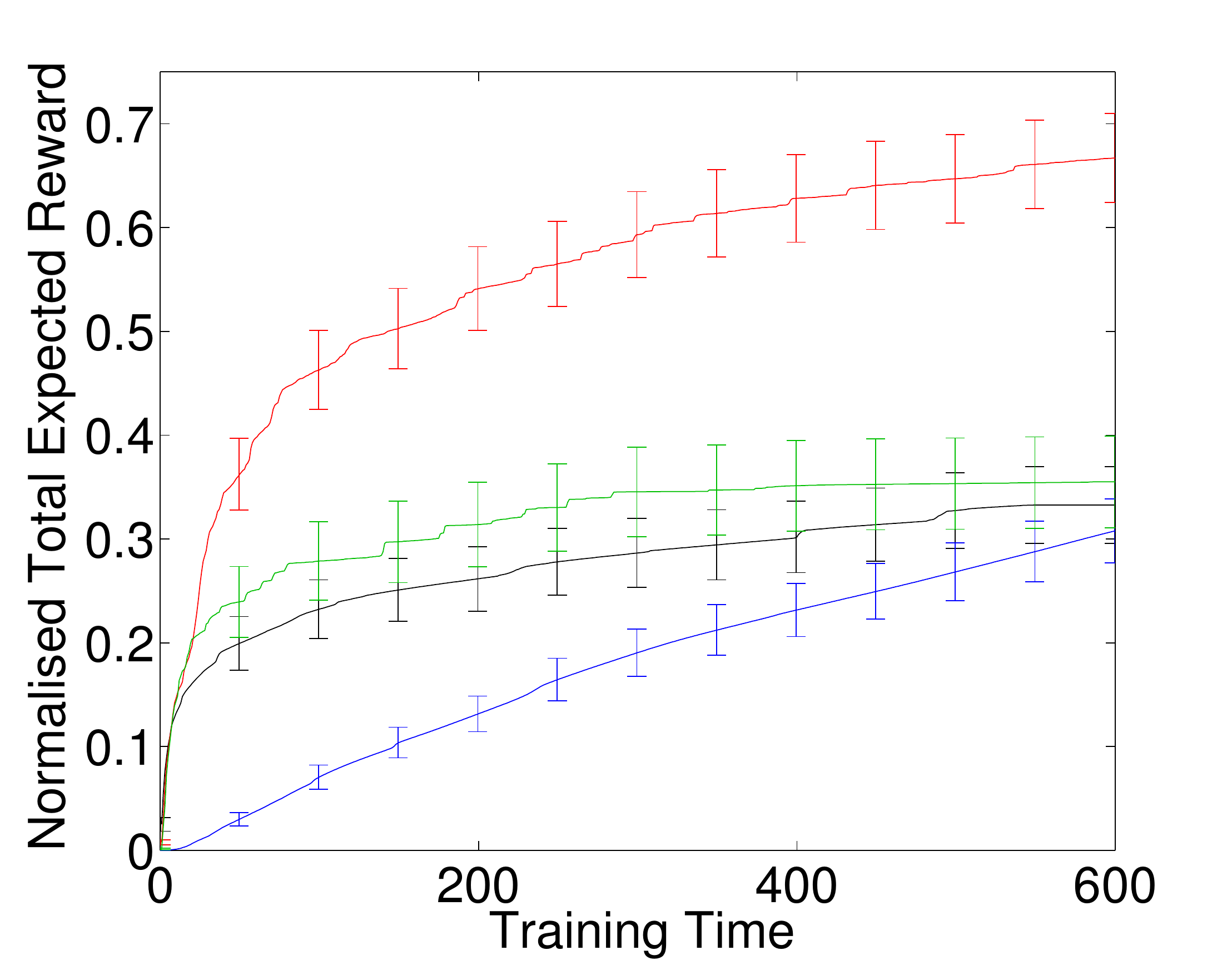}
        \caption{3-Link Manipulator : Line Search Results}\label{fig:3_link_line_search_results}
    \end{subfigure}%
    ~ 
    \begin{subfigure}[t]{0.5\textwidth}
        \centering
        \includegraphics[height=2.2in]{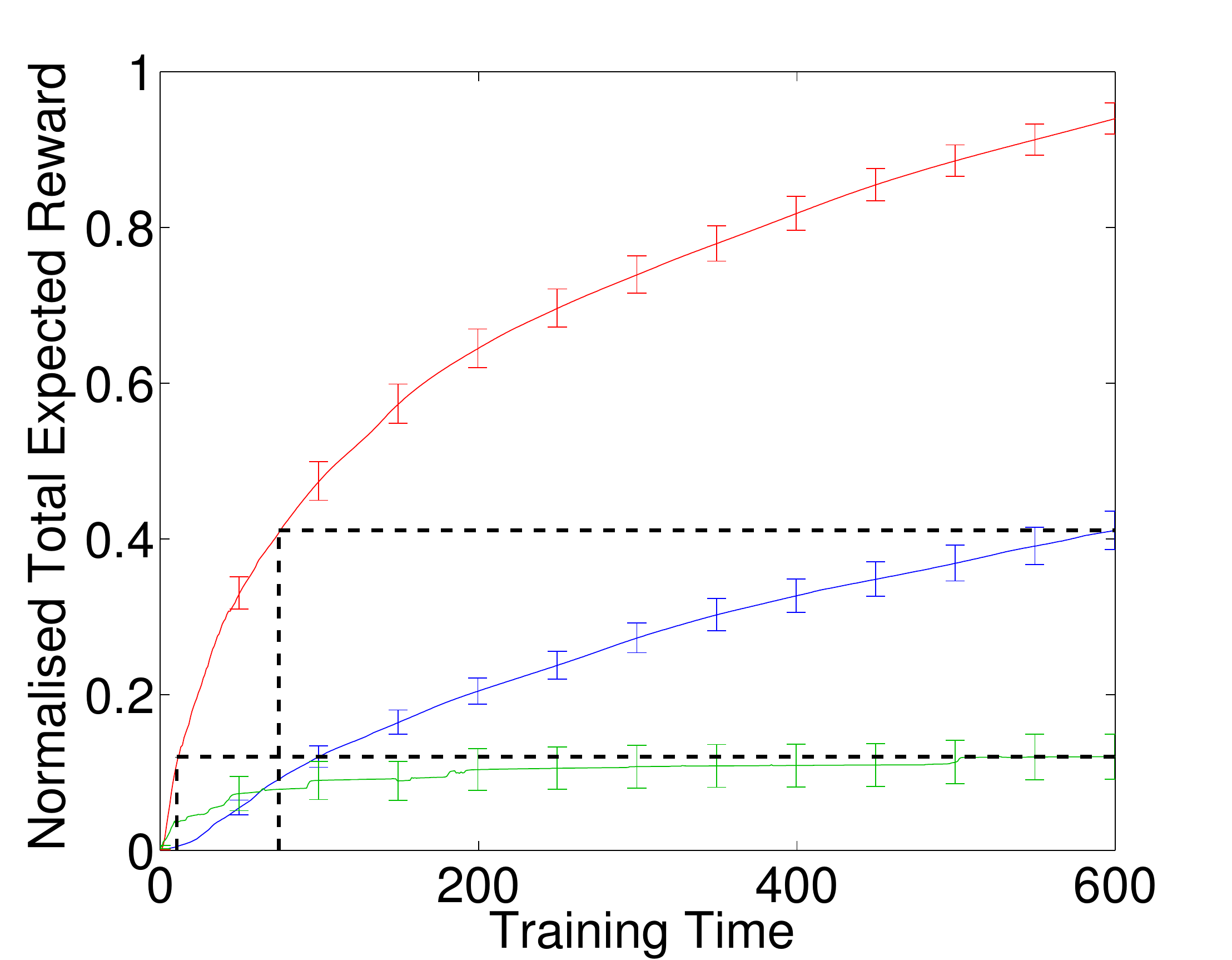}
        \caption{3-Link Manipulator : Fixed step size Results}\label{fig:3_link_fixed_step_size_results}
    \end{subfigure}
    \caption{Normalized total expected reward plotted against training time (in seconds) for the $3$-link rigid manipulator. (a) The results from the line search experiment, with the plot showing the results for steepest gradient ascent (black), Expectation Maximization (blue), the second Gauss-Newton method (red) and natural gradient ascent (green). (b) The results from the fixed step size experiment, with the plot showing the results for steepest gradient ascent (black), Expectation Maximization (blue), the second Gauss-Newton method (red), natural gradient ascent (green).}
\end{figure*}

\subsubsection{Experiment Using Fixed Step Size}

Line search as performed in the previous experiment is expensive to perform in practice, particularly in stochastic environments where many function evaluations may be required to obtain accurate function estimates. To obtain a gauge on the difficulty of selecting a step size sequence for the various policy search methods we again consider the 
3-link manipulator, but now consider a fixed step size throughout training. This is a difficult problem for algorithms such as steepest gradient ascent because the parameter space has a non-trivial number of dimensions and the objective is poorly-scaled. In both steepest gradient ascent and natural gradient ascent we considered the following fixed step sizes: $0.001$, $0.01$, $0.1$, $1$, $10$, $20$, $30$, $100$ and $250$. We were unable to obtain any reasonable results with steepest gradient ascent with any of these fixed step sizes, for which reason the results are omitted. In natural gradient ascent we found $30$ to be the best step size of those considered. In the Gauss-Newton method we considered the following fixed step sizes: $10$, $20$, $30$, $100$ and $250$ and found that the fixed step size of $30$ gave consistently good results without overstepping in the parameter space. The smaller step sizes obtained better results than Expectation Maximization, but less than the fixed step size of $30$. The larger step sizes often found superior results, but would sometimes overstep in the parameter space. For these reasons we used the fixed step size of $30$ in the final experiment. We repeated the experiment $100$ times and the results of the experiment are plotted in Figure \ref{fig:3_link_fixed_step_size_results}. The results show that even though this step size tuning is crude it is still possible to obtain strong results in comparison to Expectation Maximization, which doesn't require the selection of a step size sequence. In the experiment the Gauss-Newton method only took around $50$ seconds to obtain the same performance as $300$ seconds of training with Expectation Maximization. Furthermore Expectation Maximization was only able to obtain $40\%$ of the performance of the Gauss-Newton method, while natural gradient ascent was only able to obtain around $15\%$ of the performance. The reason that natural gradient ascent performed so poorly in this problem was because the initial control parameters were typically in a plateau region of the parameter space where the objective was close to zero. To get out of this plateau region on a regular basis and in the given amount of training time would require on overly large step size. However, once in a high reward part of the parameter space we found that, using natural gradient ascent, these large step sizes would result in overshooting in the parameter space and poor performance. The step size of $30$ was able to locate areas of high reward in a subset of the problems considered in the experiment, while not suffering from overshooting as much as the larger step sizes. The experiment highlights the robustness of the Gauss-Newton method to poor scaling, as well as the relative ease (in comparison to algorithms such as natural gradient ascent) of selecting a good step size sequence.

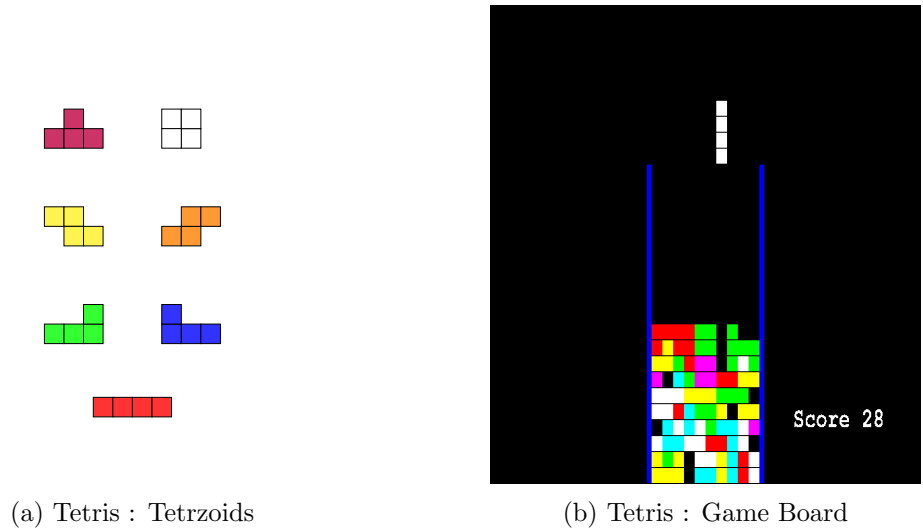
\begin{figure}
\begin{minipage}[b]{.5\linewidth}
\centering 
\scalebox{0.65}{
\begin{tikzpicture}
  [red_square/.style={rectangle,draw=black!100,fill=red!80,thin,inner sep=0pt,minimum size=4mm},
  green_square/.style={rectangle,draw=black!100,fill=green!80,thin,inner sep=0pt,minimum size=4mm},
  blue_square/.style={rectangle,draw=black!100,fill=blue!80,thin,inner sep=0pt,minimum size=4mm},
  orange_square/.style={rectangle,draw=black!100,fill=orange!80,thin,inner sep=0pt,minimum size=4mm},
  yellow_square/.style={rectangle,draw=black!100,fill=yellow!80,thin,inner sep=0pt,minimum size=4mm},
  purple_square/.style={rectangle,draw=black!100,fill=purple!80,thin,inner sep=0pt,minimum size=4mm},
  white_square/.style={rectangle,draw=black!100,fill=white!80,thin,inner sep=0pt,minimum size=4mm},
  empty_square/.style={rectangle,draw=black!0,fill=white!100,thin,inner sep=0pt,minimum size=1mm}]

    \node at (1,0.0) [empty_square] (p11){};

    \node at (1,1.5) [red_square] (p11){};
    \node at (1.4,1.5) [red_square] (p12){};
    \node at (1.8,1.5) [red_square] (p13){};
    \node at (2.2,1.5) [red_square] (p14){};

    \node at (0,3) [green_square] (p21){};
    \node at (0.4,3) [green_square] (p22){};
    \node at (0.8,3) [green_square] (p23){};
    \node at (0.8,3.4) [green_square] (p24){};

    \node at (2.4,3) [blue_square] (p31){};
    \node at (2.8,3) [blue_square] (p32){};
    \node at (3.2,3) [blue_square] (p33){};
    \node at (2.4,3.4) [blue_square] (p34){};

    \node at (2.4,5) [orange_square] (p41){};
    \node at (2.8,5) [orange_square] (p42){};
    \node at (2.8,5.4) [orange_square] (p43){};
    \node at (3.2,5.4) [orange_square] (p44){};

    \node at (0,5.4) [yellow_square] (p51){};
    \node at (0.4,5.4) [yellow_square] (p52){};
    \node at (0.4,5) [yellow_square] (p53){};
    \node at (0.8,5) [yellow_square] (p54){};

    \node at (0,7) [purple_square] (p61){};
    \node at (0.4,7.4) [purple_square] (p62){};
    \node at (0.4,7) [purple_square] (p63){};
    \node at (0.8,7) [purple_square] (p64){};

    \node at (2.8,7) [white_square] (p71){};
    \node at (2.4,7.4) [white_square] (p72){};
    \node at (2.8,7.4) [white_square] (p73){};
    \node at (2.4,7) [white_square] (p74){};
\end{tikzpicture}
}
\subcaption{Tetris : Tetrzoids}\label{fig:tetris_tetrzoids}
\end{minipage}%
\begin{minipage}[b]{.5\linewidth}
\centering \includegraphics[height=2.5in,width=2.25in]{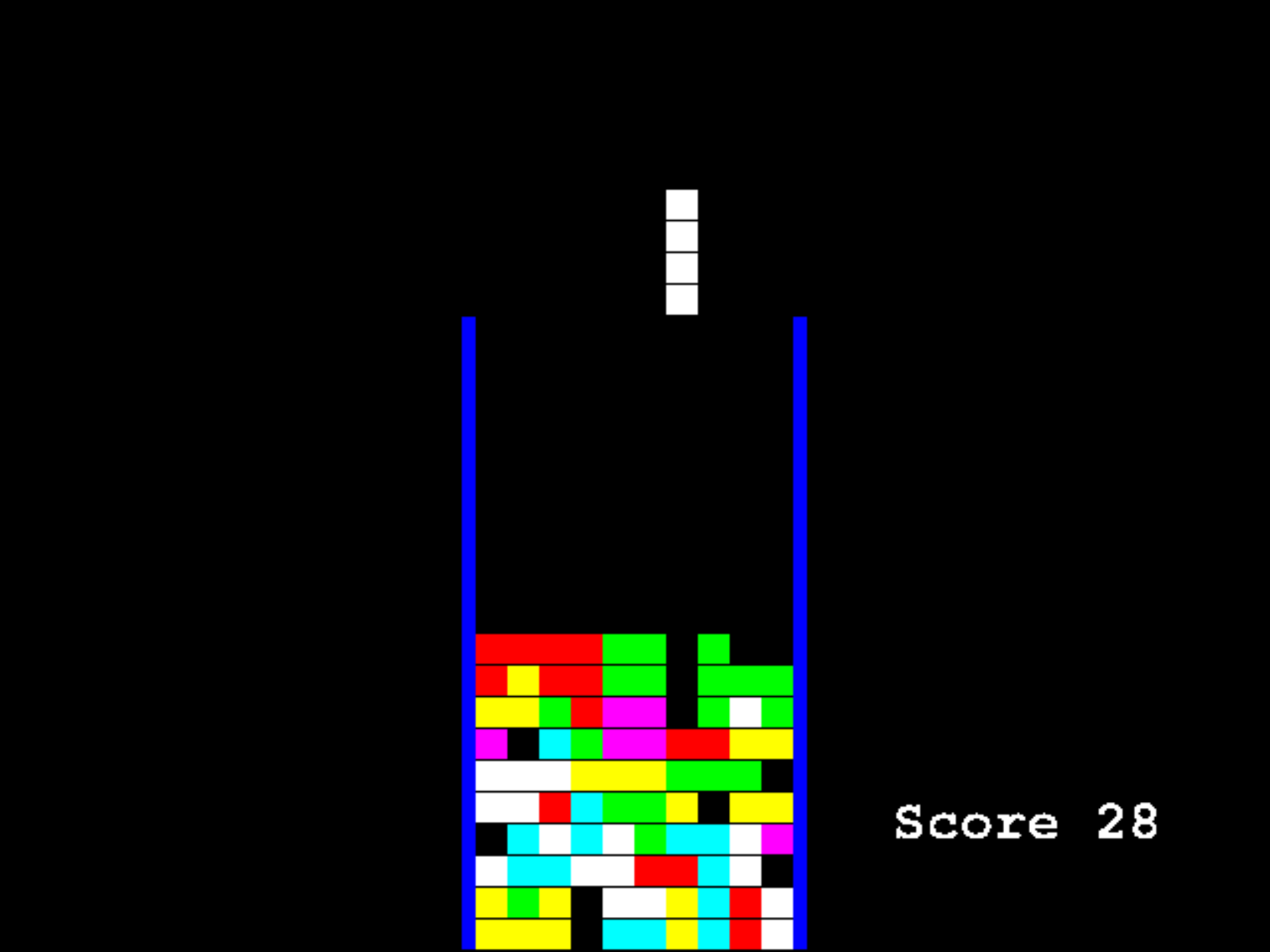}
\subcaption{Tetris : Game Board}\label{fig:tetris_game_board}
\end{minipage}
\caption{A graphical illustration of the game of tetris with (a) the collection of possible pieces, or tetrozoids, of which there are seven (b) a possible configuration of the board, which in this example is of height $20$ and width $10$.} \label{fig:tetris_illusatration}
\end{figure}

\begin{table}[b]
\centering
\scalebox{0.8}{
\begin{tabular}{| l | c | c | c |}
\hline
  & Steepest Gradient Ascent & Natural Gradient Ascent & Gauss-Newton Method \\ \hline 
\hline
Iterations & $3684 \pm 314$ & $203 \pm 34$ & $310 \pm 40$ \\ \hline 
\end{tabular}}
\caption[Iteration Counts of the 3-Link Manipulator Experiment.]{Iteration counts of the $3$-link manipulator experiment for steepest gradient ascent, natural gradient ascent and the Gauss-Newton method when using the \texttt{MinFunc} optimization library. \label{tab:linear_system_iteration_counts}}
\end{table}

\subsection{Tetris Experiment}

In this experiment we consider the Tetris domain, which is a popular computer game designed by Alexey Pajitnov in $1985$. 
In Tetris there exists a board, which is typically a $20\times10$ grid, which is empty at the beginning of a game. 
During each stage of the game a four block piece, called a tetrzoid, appears at the top of the board and begins to 
fall down the board. Whilst the tetrzoid is moving the player is allowed to rotate the tetrzoid and to move it left or right. The tetrzoid stops moving once it reaches either the bottom of the board or a previously positioned tetrzoid. In this manner the board 
begins to fill up with tetrzoid pieces. There are seven different variations of tetrzoid, as shown in Figure~\ref{fig:tetris_tetrzoids}. When a horizontal line of the board is completely filled with (pieces of) tetrzoids the 
line is removed from the board and the player receives a score of one. The game terminates when the player is not able to 
fully place a tetrzoid on the board due to insufficient space remaining on the board. 
An example configuration of the board during a game of Tetris is given in Figure~\ref{fig:tetris_game_board}.
More details on the game of Tetris can be found in \cite{Fahey_tetris_site}. 
As in other applications of Tetris in the reinforcement learning literature \citep{Kakade-2002,Bertsekas-1996} we consider a simplified version of the game in which the current tetrzoid remains above the board
until the player decides upon a desired rotation and column position for the tetrzoid.

\begin{figure*}[t!]
    \centering
    \begin{subfigure}[t]{0.5\textwidth}
        \centering
        \includegraphics[height=2.2in,width=2.7in]{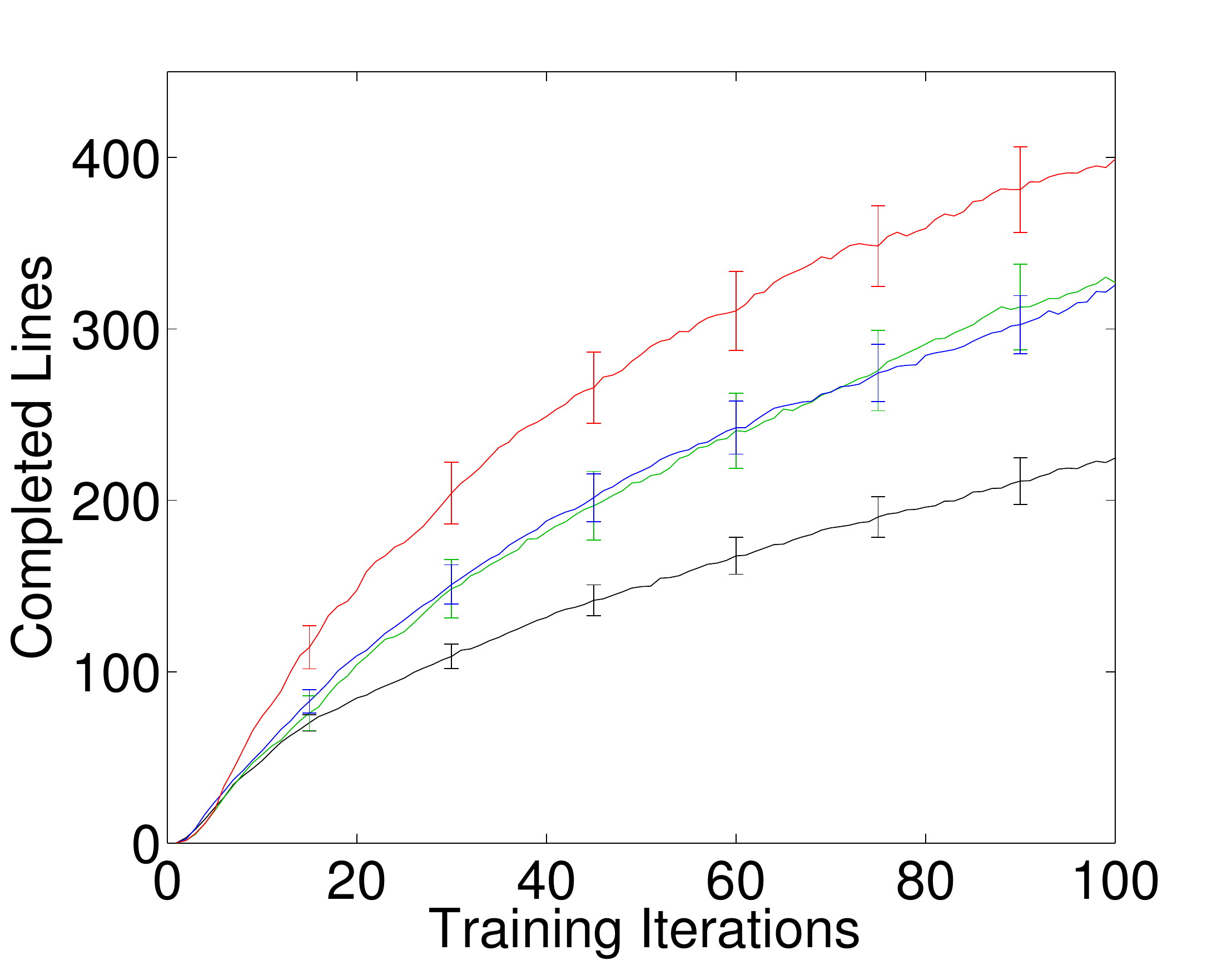}
        \caption{Tetris Results}\label{fig:tetris_results}
    \end{subfigure}%
    ~ 
    \begin{subfigure}[t]{0.5\textwidth}
        \centering
        \includegraphics[height=2.2in,width=2.7in]{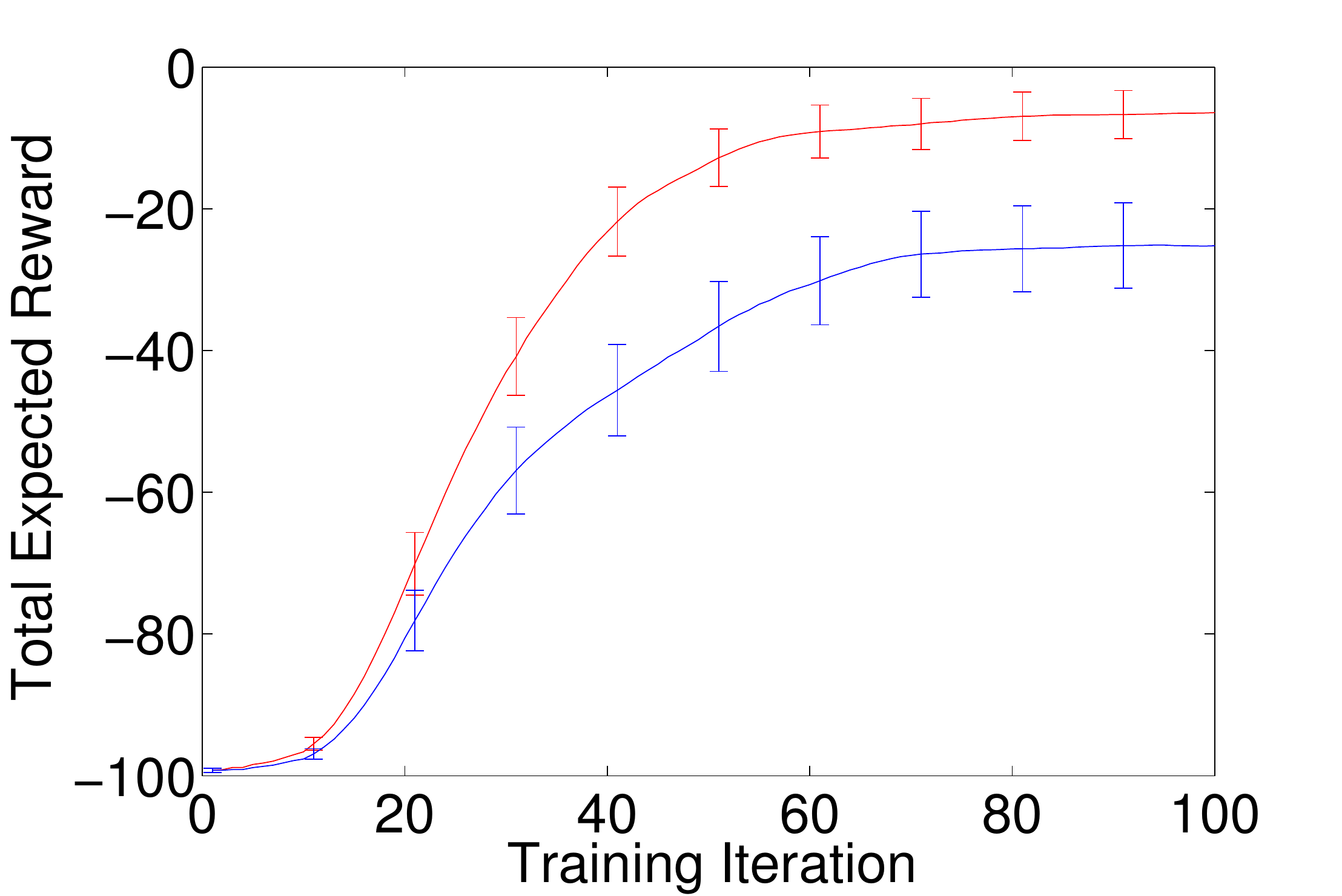}
        \caption{Robot Arm Results}\label{fig:robot_arm_results}
    \end{subfigure}
    \caption{(a) Results from the Tetris experiment, with results for steepest gradient ascent (black), natural gradient ascent (green), the diagonal Gauss-Newton method (blue) and the Gauss-Newton method (red). (b) Results from the robot arm experiment, with results for the second Gauss-Newton method (red) and the EM-algorithm (blue).}
\end{figure*}

Firstly, we compare the performance of the full and diagonal second Gauss-Newton methods to other policy search methods. Due to computational costs we consider a $10 \times 10$ board in this experiment, which results in a state space with roughly $7 \times 2^{100}$ states \citep{Bertsekas-1996}. We model the policy using a Gibb's distribution, and consider a feature vector with the following features: the heights of each column, the difference in heights between adjacent columns, the maximum height and the number of `holes'. This is the same set of features as used in \cite{Bertsekas-1996} \& \cite{Kakade-2002}. Under this policy it is not possible to obtain the explicit maximum over $\bm{w}$ in (\ref{pg_general_em}), so a straightforward application of the EM-algorithm is not possible in this problem. We therefore compare the diagonal and full Gauss-Newton methods with steepest and natural gradient ascent. We use the same procedure to evaluate the search direction for all the algorithms in the experiment. Irrespective of the policy, a game of Tetris is guaranteed to terminate after a finite number of turns \citep{Bertsekas-1996}. We therefore model each game as an absorbing state MDP. The reward at each time-point is equal to the number of lines deleted. We use a recurrent state approach \citep{williams92} to estimate the gradient, using the empty board as a recurrent state. (Since a new game starts with an empty board this state is recurrent.\footnote{This is actually an approximation because it doesn't take into account that the state is given by the configuration of the board and the current piece, so this particular `recurrent state' ignores the current piece. Empirically we found that this approximation gave better results, presumably due to reduced variance in the estimands, and there is no reason to believe that it is unfairly biasing the comparison between the various parametric policy search methods.}) We use analogous versions of this recurrent state approach for natural gradient ascent, the diagonal Gauss-Newton method and the full Gauss-Newton method. As in \cite{Kakade-2002}, we use the sample trajectories obtained during the gradient evaluation to estimate the Fisher information matrix. During each training iteration an approximation of the search direction is obtained by sampling $1000$ games, using the current policy to sample the games. Given the current approximate search direction we use the following basic line search method to obtain a step size: For every step size in a given finite set of step sizes sample a set number of games and then return the step size with the maximal score over these games. In practice, in order to reduce the susceptibility to random noise, we used the same simulator seed for each possible step size in the set. In this line search procedure we sampled $1000$ games for each  of the possible step sizes. We use the same set of step sizes
\begin{equation*}
\big\{0.1, 0.5, 1.0, 2.0, 4.0, 8.0, 16.0, 32.0, 64.0, 128.0 \big\}.
\end{equation*}
in all of the different training algorithms in the experiment. To reduce the amount of noise in the results we use the same set of simulator seeds in the search direction evaluation for each of the algorithms considered in the experiment. In particular, we generate a $n_{\textnormal{experiments}} \times n_{\textnormal{iterations}}$ matrix of simulator seeds, with $n_{\textnormal{experiments}}$ the number of repetitions of the experiment and $n_{\textnormal{iterations}}$ the number of training iterations in each experiment. We use this one matrix of simulator seeds in all of the different training algorithms, with the element in the $j^{\textnormal{th}}$ column and $i^{\textnormal{th}}$ row corresponding to the simulator seed of the $j^{\textnormal{th}}$ training iteration of the $i^{\textnormal{th}}$ experiment. In a similar manner, the set of simulator seeds we use for the line search procedure is the same for all of the different training algorithms. Finally, to make the line search consistent among all of the different training algorithms we normalize the search direction and use the resulting unit vector in the line search procedure. We ran $100$ repetitions of the experiment, each consisting of $100$ training iterations, and the mean and standard error of the results are given in Figure~\ref{fig:tetris_results}. It can be seen that the full Gauss-Newton method outperforms all of the other methods, while the performance of the diagonal Gauss-Newton method is comparable to natural gradient ascent. 

We also ran several training runs of the full approximate Newton method on the full-sized $20\times10$ board and were able to obtain a score in the region of $14,000$ completed lines, which was obtained after roughly $40$ training iterations. An approximate dynamic programming based method has previously been applied to the Tetris domain in \cite{Bertsekas-1996}. The same set of features were used and a score of roughly $4,500$ completed lines was obtained after around $6$ training iterations, after which the solution then deteriorated. More recently a modified policy iteration approach \citep{gabillon-etal-2013-nips} was able to obtain significantly better performance in the game of Tetris, completing approximately 51 million lines in a $20\times10$ board. However, these results were obtained through an entirely different set of features, and analysis of the results in \citep{gabillon-etal-2013-nips} indicate that this difference in features makes a substantial difference in performance. On a $10\times10$ board using the same features as \cite{Bertsekas-1996} the approach of \citep{gabillon-etal-2013-nips} was able to complete approximately $500$ lines on average.

\subsection{Robot Arm Experiment}

In the final experiment we consider a robotic arm application. We use the Simulation Lab \citep{schaal-2006} environment, which provides a physically realistic engine of a Barrett WAM\textsuperscript{TM} robot arm. We consider the ball-in-a-cup domain \citep{Kober-Peters-08}, which is a challenging motor skill problem that is based on the traditional children's game. In this domain a small cup is attached to the end effector of the robot arm. A ball is attached to the cup through a piece of string. At the beginning of the task the robot arm is stationary and the ball is hanging below the cup in a stationary position. The aim of the task is for the robot arm to learn an appropriate set of joint movements to first swing the ball above the cup and then to catch the ball in the cup when the ball is in its downward trajectory. The domain is episodic, with each episode $20$ seconds in length. The state of the domain is given by the angles and velocities of the seven joints in the robot arm, along with the Cartesian coordinates of the ball. The action is given by the joint accelerations of the robot arm. We denote the position of the cup and the ball by $(x_c, y_c, z_c) \in \mathbb{R}^3$ and $(x_b, y_b, z_b) \in \mathbb{R}^3$ respectively. The reward function is given by,
\begin{equation*}
r(x_c, y_c, x_b, y_b, t) = \left\{
\begin{array}{rl}
- 20 \big( (x_c - x_b)^2 + (y_c - y_b)^2 \big) & \text{if } t = t_c,\\
0 & \text{if } t \ne t_c,
\end{array} \right.
\end{equation*}
in which $t_c$ is the moment the ball crosses the $z$-plane (level with the cup) in a downward direction. If no such $t_c$ exists then the reward of the episode is given by $-100$.

We use the motor primitive framework \citep{ijspeert-2002,ijspeert-2003,schaal-2007,Kober-peters-2011} in this domain, applying a separate motor primitive to each dimension of the action space. Each motor primitive consists of a parametrized curve that models the desired action sequence (for the respective dimension of the action space) through the course of the episode. Given this collection of motor primitives the control engine within the simulator tries to follow the desired action sequence as closely as possible whilst also satisfying the constraints on the system, such as the physical constraints on the torques that can safely be applied without damaging the robot arm. As in \citet{Kober-peters-2011} we use dynamic motor primitives, using $10$ shape parameters for each of the individual motor primitives. The robot arm has 7 joints, so that there are $70$ motor primitive parameters in total. We optimize the parameters of the motor primitives by considering the MDP induced by this motor primitive framework. The action space corresponds to the space of possible motor primitives, so that $\mathcal{A} = \mathbb{R}^{70}$. There is no state space in this MDP and the planning horizon is $1$, so that this MDP is effectively a bandit problem. The reward of an action is equal to the total reward of the episode induced by the motor primitive. We consider a policy of the form,
\begin{equation*}
\pi(a; \bm{w}) = \mathcal{N} \big( a| \bm{\mu}, \big( L L^{*} \big)^{-1} \big),
\end{equation*}
with $\bm{w} = (\bm{\mu}, L)$, $\bm{\mu}$ the mean of the Gaussian and $L L^{*}$ the Cholesky decomposition of the precision matrix. We consider a diagonal precision matrix, which results in a total of $140$ policy parameters.

In this experiment we compare steepest gradient ascent, natural gradient ascent, Expectation Maximization, the first Gauss-Newton method and the second Gauss-Newton method. As the planning horizon is of length 1 it follows that $\mathcal{H}_{12}(\bm{w}) = \bm{0}$, $\forall \bm{w} \in \mathcal{W}$, so that the first Gauss-Newton method coincides with the Newton method for this MDP. The policy is block-wise $\log$-concave in $\bm{\mu}$ and $L$, but not jointly $\log$-concave in $\bm{\mu}$ and $L$. As a result we construct block diagonal forms of the preconditioning matrices for the first and second Gauss-Newton methods, with a separate block for $\bm{\mu}$ and $L$. Additionally, since the planning horizon is of length 1 it is possible to calculate the Fisher information exactly in this domain. For steepest gradient ascent and natural gradient ascent we considered several different step size sequences. Each sequence considered had a constant step size throughout, and the sequences differed in the size of this step size. We considered step sizes of length 1, 0.1, 0.01 and 0.001. For both Gauss-Newton methods we considered a fixed step size of one throughout training (i.e., no tuning of the step size sequence was performed for either the first or the second Gauss-Newton methods). As in \citet{Kober-Peters-08} the initial value of $\bm{\mu}$ is set so that the trajectory of the robot arm mimics that of a given human demonstration. The diagonal elements of the precision matrix are initialized to 0.01. During each training iteration we sampled $15$ actions from the policy and used the episodes generated from these samples to estimate the search direction. To deal with this low number of samples we used the samples from the last 10 training iterations when calculating the search direction, taking the `effective'  sample size up to 150. Finally, we used the reward/fitness shaping approach of \citet{wierstra-etal-2014} in all the algorithms considered, using the same shaping function as in \citet{wierstra-etal-2014}. In each run of the experiment we performed 100 updates of the policy parameters. We repeated the experiment 50 times and the results are given in Figure~\ref{fig:robot_arm_results}. We were unable to successfully learn to catch the ball in the cup using either steepest gradient ascent, natural gradient ascent or the first Gauss-Newton method. For this reason the results for these algorithms are omitted. It can be seen that the second Gauss-Newton method significantly outperforms the EM-algorithm in this domain. Out of the 50 runs of the experiment, the second Gauss-Newton method was successfully able to learn to catch the ball in the cup 45 times. The EM-algorithm successfully learnt the task 36 times. As the $\log$-policy is quadratic in $\bm{\mu}$ and a fixed step size of one was used in the second Gauss-Newton method it follows that the update of $\bm{\mu}$ in the second Gauss-Newton method and the EM-algorithm are the same. The difference in performance can therefore be attributed to the difference in the updates of $L$ between the two algorithms.

\section{Conclusions}

Approximate Newton methods, such as quasi-Newton methods and the Gauss-Newton method, are standard optimization techniques. These methods aim to maintain the benefits of Newton's method, whilst alleviating its shortcomings. In this paper we have considered approximate Newton methods in the context of policy optimization in MDPs. The first contribution of this paper was to provide a novel analysis of the Hessian of the MDP objective function for policy optimization. This included providing a novel form for the Hessian, as well as detailing the positive/negative definiteness properties of certain terms in the Hessian. Furthermore, we have shown that when the policy parametrization is sufficiently rich then the remaining terms in the Hessian vanish in the vicinity of a local optimum. Motivated by this analysis we introduced two Gauss-Newton Methods for MDPs. Like the Gauss-Newton method for non-linear least squares, these methods involve approximating the Hessian by ignoring certain terms in the Hessian which are difficult to estimate. The approximate Hessians possess desirable properties, such as negative definiteness, and we demonstrate several important performance guarantees including guaranteed ascent directions, invariance to affine transformation of the parameter space, and convergence guarantees. We also demonstrated our second Gauss-Newton algorithm is closely related to both the EM-algorithm and natural gradient ascent applied to MDPs, providing novel insights into both of these algorithms. We have compared the proposed Gauss-Newton methods with other techniques in the policy search literature over a range of challenging domains, including Tetris and a robotic arm application. We found that the second Gauss-Newton method performed significantly better than other methods in all of the domains that we considered. 


\acks{We would like to thank Peter Dayan, David Silver, Nicolas Heess and David Barber for helpful discussions on this work. We would also like to thank Gerhard Neumann and Christian Daniel for their assistance in the robot arm experiment. This work was supported by the European Community Seventh Framework Programme (FP7/2007-2013) under grant agreement 270327 (CompLACS), and by the EPSRC under grant agreement EP/M006093/1 (C-PLACID).}


\newpage

\appendix

\section{Proofs}

\subsection{Proofs of Theorems~\ref{theorem_policy_gradient_theorem} and \ref{theorem_policy_hessian_theorem}}   \label{app:pol_grad_theorem}

We begin with an auxiliary Lemma.
\begin{Lemma} \label{ValueGradientLemma}
Suppose we are given a Markov decision process with objective (\ref{ps_objectiveFunction}) and Markovian trajectory distribution (\ref{trajectory_prob_parametric}). For any given parameter vector, $\bm{w} \in \mathcal{W}$, the following identities hold,
\begin{align}
\nabla_{\bm{w}} V(s;\bm{w}) &=  \sum_{t=1}^\infty\sum_{s_t \in \myset{S}} \sum_{a_t \in \myset{A}} \gamma^{t-1} p(s_t,a_t|s_1=s; \bm{w}) Q(s_t,a_t;\bm{w}) \nabla_{\bm{w}} \log \pi(a_t|s_t; \bm{w}) \label{value_gradient} \\
\nabla_{\bm{w}} Q(s,a;\bm{w}) &= \sum_{t=2}^\infty\sum_{s_t \in \myset{S}} \sum_{a_t \in \myset{A}}  \gamma^{t-1} p(s_{t},a_{t}|s_1=s,a_1=a; \bm{w}) Q(s_{t},a_{t};\bm{w}) \nabla_{\bm{w}} \log \pi(a_{t}|s_{t}; \bm{w}). \label{Q_gradient}
\end{align}
\end{Lemma}
\begin{proof}
We start by writing the value function in the form
\begin{equation}
V(s;\bm{w}) =  \sum_{t=1}^\infty \sum_{s_{1:t}} \sum_{a_{1:t}} \gamma^{t-1} p(s_{1:t}, a_{1:t}|s_1=s; \bm{w}) R(s_t, a_t), \label{theorem_policy_gradient_theorem_proof_eq1}
\end{equation}
so that,
\begin{equation}
\nabla_{\bm{w}} V(s;\bm{w}) = \sum_{t=1}^\infty \sum_{s_{1:t}} \sum_{a_{1:t}} \gamma^{t-1} p(s_{1:t}, a_{1:t}|s_1=s; \bm{w}) \nabla_{\bm{w}} \log p(s_{1:t}, a_{1:t}|s_1=s; \bm{w}) R(s_t, a_t). \nn
\end{equation}
Using the fact that
\begin{equation}
\nabla_{\bm{w}} \log p(s_{1:t}, a_{1:t}|s_1=s; \bm{w}) = \sum_{\tau = 1}^t \nabla_{\bm{w}} \log \pi(a_\tau| s_\tau; \bm{w}), \label{gradlogsums}
\end{equation}
we have that,
\begin{align}
&\nabla_{\bm{w}} V(s;\bm{w}) =  \sum_{t=1}^\infty \sum_{s_t,a_t} \sum_{\tau=1}^t \sum_{s_\tau,a_\tau} \gamma^{t-1} p(s_\tau, a_\tau, s_t, a_t|s_1=s; \bm{w}) \nabla_{\bm{w}} \log \pi(a_\tau| s_\tau; \bm{w})R(s_t, a_t) \nonumber \\
&=  \sum_{\tau=1}^\infty \sum_{s_\tau,a_\tau} \gamma^{\tau-1} p(s_\tau, a_\tau|s_1=s; \bm{w}) \nabla_{\bm{w}} \log \pi(a_\tau| s_\tau; \bm{w}) \sum_{t=\tau}^\infty \sum_{s_t,a_t} \gamma^{t-\tau} p(s_t, a_t | s_\tau, a_\tau; \bm{w}) R(s_t, a_t) \nn\\
&=  \sum_{\tau=1}^\infty \sum_{s_\tau,a_\tau} \gamma^{\tau-1} p(s_\tau, a_\tau|s_1=s; \bm{w}) \nabla_{\bm{w}} \log \pi(a_\tau| s_\tau; \bm{w}) Q(s_\tau, a_\tau; \bm{w}). \label{theorem_policy_gradient_theorem_proof_eq3}
\end{align}
where in the second line we swapped the order of summation and the third line follows from the definition (\ref{saValFn}). Identity (\ref{Q_gradient}) now follows by applying (\ref{saValFn}):
\begin{align}
\nabla_{\bm{w}} &Q(s,a;\bm{w}) = \gamma \sum_{s'} P(s'|s,a) \nabla_{\bm{w}}V(s'; \bm{w}) \nn \\
&= \gamma\sum_{s'} P(s'|s,a)\sum_{t=2}^\infty \sum_{s_t \in \myset{S}} \sum_{a_t \in \myset{A}} \gamma^{t-2}  p(s_t,a_t|s_2=s';\bm{w}) \nabla_{\bm{w}}\log \pi(a_{t}|s_{t}; \bm{w})   Q(s_t,a_t;\bm{w})  \nn \\
&= \sum_{t=2}^\infty\sum_{s_t \in \myset{S}} \sum_{a_t \in \myset{A}}  \gamma^{t-1} p(s_{t},a_{t}|s_1=s,a_1=a; \bm{w}) Q(s_{t},a_{t};\bm{w}) \nabla_{\bm{w}} \log \pi(a_{t}|s_{t}; \bm{w}). \nn
\end{align}
\end{proof}

\begin{reptheorem}{theorem_policy_gradient_theorem}
\begin{proof} Theorem~\ref{theorem_policy_gradient_theorem} follows immediately from Lemma~\ref{ValueGradientLemma} by taking the expectation over $s_1$ w.r.t. the start state distribution $p_1$ and using the definition (\ref{gamma_state_dist}) of the discounted trajectory distribution.
\end{proof}
\end{reptheorem}


\begin{reptheorem}{theorem_policy_hessian_theorem}
\begin{proof}
Starting from 
\begin{equation}
U(\bm{w}) =  \sum_{t=1}^\infty \sum_{s_{1:t}} \sum_{a_{1:t}} \gamma^{t-1} p(s_{1:t}, a_{1:t}; \bm{w}) R(s_t, a_t),\nn\\
\end{equation}
the Hessian of (\ref{ps_objectiveFunction}) takes the form
\begin{align}
& \nabla_{\bm{w}} \nabla_{\bm{w}}^\top U(\bm{w}) =  \sum_{t=1}^\infty \sum_{s_{1:t}} \sum_{a_{1:t}} \gamma^{t-1} p(s_{1:t}, a_{1:t}; \bm{w}) \nabla_{\bm{w}} \nabla_{\bm{w}}^\top \log p(s_{1:t}, a_{1:t}; \bm{w}) R(s_t, a_t)\nonumber\\
&+ \sum_{t=1}^\infty \sum_{s_{1:t}} \sum_{a_{1:t}} \gamma^{t-1} p(s_{1:t}, a_{1:t}; \bm{w}) \nabla_{\bm{w}} \log p(s_{1:t}, a_{1:t}; \bm{w}) \nabla_{\bm{w}}^\top \log p(s_{1:t}, a_{1:t}; \bm{w}) R(s_t, a_t). \label{theorem_policy_hessian_theorem_proof_eq1} 
\end{align}
Using the fact that $\nabla_{\bm{w}}  \nabla_{\bm{w}}^\top \log p(s_{1:t}, a_{1:t}|s_1=s; \bm{w}) = \sum_{\tau = 1}^t \nabla_{\bm{w}} \nabla_{\bm{w}}^\top \log \pi(a_\tau| s_\tau; \bm{w})$ we will show that the first term in (\ref{theorem_policy_hessian_theorem_proof_eq1}) is equal to $\mathcal{H}_2(\bm{w})$ as defined in (\ref{pg_hessian_H2}):
\begin{align}
&\sum_{t=1}^\infty \sum_{s_{1:t}} \sum_{a_{1:t}} \gamma^{t-1} p(s_{1:t}, a_{1:t}; \bm{w}) \nabla_{\bm{w}} \nabla_{\bm{w}}^\top \log p(s_{1:t}, a_{1:t}; \bm{w}) R(s_t, a_t)\nonumber\\
&= \sum_{t=1}^\infty \sum_{s_{1:t}} \sum_{a_{1:t}} \gamma^{t-1} p(s_{1:t}, a_{1:t}; \bm{w}) \sum_{\tau = 1}^t \nabla_{\bm{w}} \nabla_{\bm{w}}^\top \log \pi(a_\tau| s_\tau; \bm{w}) R(s_t, a_t)\nonumber\\
&= \sum_{\tau = 1}^\infty  \gamma^{\tau-1} \sum_{s_\tau,a_\tau} p(s_\tau, a_\tau ; \bm{w}) \nabla_{\bm{w}} \nabla_{\bm{w}}^\top \log \pi(a_\tau| s_\tau; \bm{w}) \sum_{t=\tau}^\infty \gamma^{t- \tau} \sum_{s_t,a_t} p(s_t, a_t | s_\tau, a_\tau ; \bm{w}) R(s_t, a_t) \nn \\
&= \sum_{\tau = 1}^\infty  \gamma^{\tau-1} \sum_{s_\tau,a_\tau} p(s_\tau, a_\tau ; \bm{w}) \nabla_{\bm{w}} \nabla_{\bm{w}}^\top \log \pi(a_\tau| s_\tau; \bm{w}) Q( s_\tau, a_\tau ; \bm{w}). \nn\\
&= \mathcal{H}_2(\bm{w}) \nn
\end{align}
where in the third line we swapped the order of summation. 

Using (\ref{gradlogsums}) we can write the second term in (\ref{theorem_policy_hessian_theorem_proof_eq1}) as,
\begin{align}
&  \sum_{t=1}^\infty \sum_{s_{1:t}} \sum_{a_{1:t}} \gamma^{t-1} p(s_{1:t}, a_{1:t}; \bm{w}) \nabla_{\bm{w}} \log p(s_{1:t}, a_{1:t}; \bm{w}) \nabla_{\bm{w}}^\top \log p(s_{1:t}, a_{1:t}; \bm{w}) R(s_t, a_t) \nonumber \\
&= \sum_{t=1}^\infty \sum_{\tau=1}^t \sum_{s_{1:t}} \sum_{a_{1:t}} \gamma^{t-1} p(s_{1:t}, a_{1:t}; \bm{w}) \nabla_{\bm{w}} \log \pi(a_\tau | s_\tau; \bm{w}) \nabla_{\bm{w}}^\top \log \pi(a_\tau | s_\tau; \bm{w}) R(s_t, a_t) \nonumber \\
&\quad+  \sum_{t=1}^\infty \sum_{\stackrel{\tau_1, \tau_2=1}{\tau_1 \ne \tau_2}}^t \sum_{s_{1:t}} \sum_{a_{1:t}} \gamma^{t-1} p(s_{1:t}, a_{1:t}; \bm{w}) \nabla_{\bm{w}} \log \pi(a_{\tau_1} | s_{\tau_1}; \bm{w}) \nabla_{\bm{w}}^\top \log \pi(a_{\tau_2} | s_{\tau_2}; \bm{w}) R(s_t, a_t). \label{theorem_policy_hessian_theorem_proof_eq2}
\end{align}
By swapping the order of summation and following analogous calculations to those above, it can be shown that the first term in (\ref{theorem_policy_hessian_theorem_proof_eq2}) is equal to $\mathcal{H}_{1}(\bm{w})$ as defined in (\ref{pg_hessian_H1}). It remains to show that the second term in (\ref{theorem_policy_hessian_theorem_proof_eq2}) is given by $\mathcal{H}_{12}(\bm{w}) + \mathcal{H}_{12}^\top(\bm{w})$, with $\mathcal{H}_{12}(\bm{w})$ as given in (\ref{pg_hessian_H12}). Splitting the second term in (\ref{theorem_policy_hessian_theorem_proof_eq2}) into two terms, 
\begin{align}
& \sum_{t=1}^\infty \sum_{\stackrel{\tau_1, \tau_2=1}{\tau_1 \ne \tau_2}}^t \sum_{s_{1:t}} \sum_{a_{1:t}} \gamma^{t-1} p(s_{1:t}, a_{1:t}; \bm{w}) \nabla_{\bm{w}} \log \pi(a_{\tau_1} | s_{\tau_1}; \bm{w}) \nabla_{\bm{w}}^\top \log \pi(a_{\tau_2} | s_{\tau_2}; \bm{w}) R(s_t, a_t) \nn \\
&= \sum_{t=1}^\infty \sum_{\tau_2=1}^t \sum_{\tau_1=1}^{\tau_2-1} \sum_{s_{1:t}} \sum_{a_{1:t}} \gamma^{t-1} p(s_{1:t}, a_{1:t}; \bm{w}) \nabla_{\bm{w}} \log \pi(a_{\tau_1} | s_{\tau_1}; \bm{w}) \nabla_{\bm{w}}^\top \log \pi(a_{\tau_2} | s_{\tau_2}; \bm{w}) R(s_t, a_t) \nn \\
&+ \sum_{t=1}^\infty \sum_{\tau_1=1}^t \sum_{\tau_2=1}^{\tau_1-1} \sum_{s_{1:t}} \sum_{a_{1:t}} \gamma^{t-1} p(s_{1:t}, a_{1:t}; \bm{w}) \nabla_{\bm{w}} \log \pi(a_{\tau_1} | s_{\tau_1}; \bm{w}) \nabla_{\bm{w}}^\top \log \pi(a_{\tau_2} | s_{\tau_2}; \bm{w}) R(s_t, a_t), \label{TheSplit}
\end{align}
we will show that the first term is equal to $\mathcal{H}_{12}(\bm{w})$. Given this, it immediately follows that the second term is equal to $\mathcal{H}_{12}^\top(\bm{w})$. Using the Markov property of the transition dynamics and the policy it follows that the first term in (\ref{TheSplit}) is given by, 
\begin{align}
& \sum_{t=1}^\infty \sum_{\tau_2=1}^t \sum_{\tau_1=1}^{\tau_2-1} \sum_{s_{\tau_1}, a_{\tau_1}} 
\gamma^{\tau_1-1} p(s_{\tau_1}, a_{\tau_1}; \bm{w}) \nabla_{\bm{w}} \log \pi(a_{\tau_1} | s_{\tau_1}; \bm{w})  \nn \\
& \times \sum_{s_{\tau_2},a_{\tau_2}} \gamma^{\tau_2-\tau_1} p(s_{\tau_2}, a_{\tau_2}| s_{\tau_1}, a_{\tau_1}; \bm{w}) \nabla_{\bm{w}}^\top \log \pi(a_{\tau_2} | s_{\tau_2}; \bm{w}) \sum_{s_t,a_t} \gamma^{t-\tau_2} p(s_t, a_t| s_{\tau_2}, a_{\tau_2}; \bm{w}) R(s_t, a_t). \nonumber
\end{align}
Rearranging the summation over $t$, $\tau_1$ and $\tau_2$ this can be rewritten in the form, 
\begin{align}
&  \sum_{\tau_1=1}^\infty \sum_{s_{\tau_1} ,a_{\tau_1}} 
\gamma^{\tau_1-1} p(s_{\tau_1}, a_{\tau_1}; \bm{w}) \nabla_{\bm{w}} \log \pi(a_{\tau_1} | s_{\tau_1}; \bm{w}) \nonumber  \\
& \quad \times \bigg\{ \sum_{\tau_2=\tau_1+1}^\infty \sum_{s_{\tau_2},a_{\tau_2}} \gamma^{\tau_2-\tau_1} p(s_{\tau_2}, a_{\tau_2}| s_{\tau_1}, a_{\tau_1}; \bm{w}) \nabla_{\bm{w}}^\top \log \pi(a_{\tau_2} | s_{\tau_2}; \bm{w}) \nn \\
& \quad\quad\quad \sum_{t=\tau_2}^\infty \sum_{s_t,a_t} \gamma^{t-\tau_2} p(s_t, a_t| s_{\tau_2}, a_{\tau_2}; \bm{w}) R(s_t, a_t) \bigg\} \nn\\
&=\sum_{\tau_1=1}^\infty \sum_{s_{\tau_1} ,a_{\tau_1}} 
\gamma^{\tau_1-1} p(s_{\tau_1}, a_{\tau_1}; \bm{w}) \nabla_{\bm{w}} \log \pi(a_{\tau_1} | s_{\tau_1}; \bm{w}) \nonumber  \\
&  \quad\times  \sum_{\tau_2=\tau_1+1}^\infty \sum_{s_{\tau_2},a_{\tau_2}} \gamma^{\tau_2-\tau_1} p(s_{\tau_2}, a_{\tau_2}| s_{\tau_1}, a_{\tau_1}; \bm{w}) \nabla_{\bm{w}}^\top \log \pi(a_{\tau_2} | s_{\tau_2}; \bm{w}) Q( s_{\tau_2}, a_{\tau_2}; \bm{w})\nn\\
&=\sum_{\tau_1=1}^\infty \sum_{s_{\tau_1} ,a_{\tau_1}} 
\gamma^{\tau_1-1} p(s_{\tau_1}, a_{\tau_1}; \bm{w}) \nabla_{\bm{w}} \log \pi(a_{\tau_1} | s_{\tau_1}; \bm{w})\nabla_{\bm{w}} Q( s_{\tau_1}, a_{\tau_1}; \bm{w}) \nn\\
&= \mathcal{H}_{12}(\bm{w}) \nn
\end{align}
Where the penultimate line follows from (\ref{Q_gradient}). This completes the proof.
\end{proof}
\end{reptheorem}

\subsection{Proof of Theorem~\ref{corollary_policy_hessian_theorem}} \label{NovelHessianProof}

Recalling that the state-action value function takes the form, $Q(s,a;\bm{w}) = V(s;\bm{w}) + A(s,a;\bm{w})$, the matrices
$\mathcal{H}_1(\bm{w})$ and $\mathcal{H}_2(\bm{w})$ can be written in the following forms,
\begin{equation}
\mathcal{H}_1(\bm{w}) = \mathcal{A}_1(\bm{w}) + \mathcal{V}_1(\bm{w}), \quad \quad \quad \quad \quad \quad \mathcal{H}_2(\bm{w}) = \mathcal{A}_2(\bm{w}) + \mathcal{V}_2(\bm{w}), \label{h11_h2_alternative_forms}
\end{equation}
where,
\begin{align}
\mathcal{A}_1(\bm{w}) &= \sum_{(s,a) \in \mathcal{S} \times \mathcal{A}} p_\gamma(s,a;\bm{w}) A(s,a;\bm{w}) \nabla_{\bm{w}} \log \pi(a|s;\bm{w}) \nabla^\top_{\bm{w}} \log \pi(a|s;\bm{w}) \nn \\
\mathcal{A}_2(\bm{w}) &= \sum_{(s,a) \in \mathcal{S} \times \mathcal{A}} p_\gamma(s,a;\bm{w}) A(s,a;\bm{w}) \nabla_{\bm{w}}  \nabla^\top_{\bm{w}} \log \pi(a|s;\bm{w}) \nn\\
\mathcal{V}_1(\bm{w}) &= \sum_{(s,a) \in \mathcal{S} \times \mathcal{A}} p_\gamma(s,a;\bm{w}) V(s,a;\bm{w}) \nabla_{\bm{w}} \log \pi(a|s;\bm{w}) \nabla^\top_{\bm{w}} \log \pi(a|s;\bm{w}) \nn \\
\mathcal{V}_2(\bm{w}) &= \sum_{(s,a) \in \mathcal{S} \times \mathcal{A}} p_\gamma(s,a;\bm{w}) V(s,a;\bm{w}) \nabla_{\bm{w}}  \nabla^\top_{\bm{w}} \log \pi(a|s;\bm{w}). \nn
\end{align}

We begin with the following auxiliary lemmas.

\begin{Lemma}\label{theorem_value_fisher_information_equiv_lemma}
Suppose we are given a Markov decision process with objective (\ref{ps_objectiveFunction}) and Markovian trajectory distribution (\ref{trajectory_prob_parametric}). Provided that the policy satisfies the Fisher regularity conditions, then for any given parameter vector, $\bm{w} \in \mathcal{W}$, the matrices $\mathcal{V}_1(\bm{w})$ and $\mathcal{V}_2(\bm{w})$ satisfy the following relation
\begin{equation}
\mathcal{V}_1(\bm{w}) = - \mathcal{V}_2(\bm{w}).
\end{equation}
\end{Lemma}
\begin{proof}
As the policy satisfies the Fisher regularity conditions, then for any state, $s \in \mathcal{S}$, the following relation holds
\begin{equation*}
\sum_{a \in \mathcal{A}} \pi(a|s;\bm{w}) \nabla_{\bm{w}} \log \pi(a|s;\bm{w}) \nabla^\top_{\bm{w}} \log \pi(a|s;\bm{w}) = - \sum_{a \in \mathcal{A}} \pi(a|s;\bm{w}) \nabla_{\bm{w}} \nabla^\top_{\bm{w}} \log \pi(a|s;\bm{w}).
\end{equation*}
This means that $\mathcal{V}_{1}(\bm{w})$ can be written in the form
\begin{align*}
\mathcal{V}_{1}(\bm{w}) &= \sum_{s \in \mathcal{S}} p_\gamma(s;\bm{w}) V(s;\bm{w}) \sum_{a \in \mathcal{A}} \pi(a|s;\bm{w}) \nabla_{\bm{w}} \log \pi(a|s;\bm{w}) \nabla^\top_{\bm{w}} \log \pi(a|s;\bm{w}), \\
&= - \sum_{s \in \mathcal{S}} p_\gamma(s;\bm{w}) V(s;\bm{w}) \sum_{a \in \mathcal{A}} \pi(a|s;\bm{w}) \nabla_{\bm{w}} \nabla^\top_{\bm{w}} \log \pi(a|s;\bm{w}) = - \mathcal{V}_2(\bm{w}),
\end{align*}
which completes the proof.
\end{proof}

\begin{Lemma}\label{Lemma_A2_equal_zero}
Suppose we are given a Markov decision process with objective (\ref{ps_objectiveFunction}) and Markovian trajectory distribution (\ref{trajectory_prob_parametric}). 
If the policy parametrization has constant curvature with respect to the action space, then
\begin{equation}
\mathcal{A}_2(\bm{w}) = \bm{0}.
\end{equation}
\end{Lemma}
\begin{proof}
Recalling Definition~\ref{ValueConsistent}, the matrix $\mathcal{A}_2(\bm{w})$ takes the form, 
\begin{align*}
\mathcal{A}_2(\bm{w}) &= \sum_{(s,a) \in \mathcal{S} \times \mathcal{A}} p_\gamma(s,a;\bm{w}) A(s,a;\bm{w}) \nabla_{\bm{w}} \nabla^\top_{\bm{w}} \log \pi(s;\bm{w}), \\
&= \sum_{s \in \mathcal{S}} p_\gamma(s;\bm{w}) \nabla_{\bm{w}} \nabla^\top_{\bm{w}} \log \pi(s;\bm{w}) \sum_{a \in \mathcal{A}} \pi(a|s;\bm{w}) A(s,a;\bm{w}).
\end{align*}
The relation $\mathcal{A}_2(\bm{w})= \bm{0}$ follows because $\sum_{a \in \mathcal{A}} \pi(a|s;\bm{w}) A(s,a;\bm{w}) = 0$, for all $s \in \mathcal{S}$.
\end{proof}

Lemmas~\ref{theorem_value_fisher_information_equiv_lemma} \& \ref{Lemma_A2_equal_zero}, along with the relation (\ref{h11_h2_alternative_forms}), directly imply the result of Theorem~\ref{corollary_policy_hessian_theorem}.

\subsection{Proof of Theorem~\ref{NegativeDefiniteLemma} and Definiteness Results}\label{negative_definite_appendix}

\begin{reptheorem}{NegativeDefiniteLemma}
\begin{proof}
The first result follows from the fact that when the policy is $\log$-concave with respect to the policy parameters, then $\mathcal{H}_2(\bm{w})$ is a non-negative mixture of negative-definite matrices, which again is negative-definite \citep{boyd-vandenberghe-book}.

The second result follows because when the policy parametrization has constant curvature with respect to the action space, then by Lemma~\ref{Lemma_A2_equal_zero} in Section~\ref{NovelHessianProof} $\mathcal{A}_2(\bm{w}) = \bm{0}$, so that 
\begin{equation*}
\mathcal{H}_2(\bm{w}) = \mathcal{A}_2(\bm{w}) + \mathcal{V}_2(\bm{w}) = \mathcal{V}_2(\bm{w}) = - \mathcal{V}_1(\bm{w}),
\end{equation*}
with
\begin{align}
\mathcal{V}_1(\bm{w}) &= \sum_{(s,a) \in \mathcal{S} \times \mathcal{A}} p_\gamma(s,a;\bm{w}) V(s,a;\bm{w}) \nabla_{\bm{w}} \log \pi(a|s;\bm{w}) \nabla^\top_{\bm{w}} \log \pi(a|s;\bm{w}) \nn \\
\mathcal{V}_2(\bm{w}) &= \sum_{(s,a) \in \mathcal{S} \times \mathcal{A}} p_\gamma(s,a;\bm{w}) V(s,a;\bm{w}) \nabla_{\bm{w}}  \nabla^\top_{\bm{w}} \log \pi(a|s;\bm{w}). \nn
\end{align}
The result now follows because $- \mathcal{V}_1(\bm{w})$ is negative-definite for all $\bm{w} \in \mathcal{W}$.
\end{proof}
\end{reptheorem}

\begin{Lemma}\label{H11_is_positive_definite}
For any $\bm{w} \in \mathcal{W}$ the matrix,
\begin{equation*}
\mathcal{H}_{11}(\bm{w}) = \mathcal{H}_1(\bm{w}) + \mathcal{H}_{12}(\bm{w}) + \mathcal{H}_{12}^\top(\bm{w})
\end{equation*}
is positive-definite.
\begin{proof}
This follows immediately from the form of $\mathcal{H}_1(\bm{w}) + \mathcal{H}_{12}(\bm{w}) + \mathcal{H}_{12}^\top(\bm{w})$ given by (\ref{theorem_policy_hessian_theorem_proof_eq2}) in Theorem~\ref{theorem_policy_hessian_theorem}, which is positive-definite since the reward function is assumed positive.
\end{proof}
\end{Lemma}

\subsection{Proof of Theorem~\ref{TabNonDec}} \label{ProofOfTabNonDec}

We first prove an auxiliary lemma about the gradient of the value function in the case of a tabular policy. As we are considering a tabular policy we have a separate parameter vector $\bm{w}_{s}$ for each state $s \in \mathcal{S}$. We denote the parameter vector of the entire policy by $\bm{w}$, in which this is given by the concatenation of the parameter vectors of the different states. The dimension of $\bm{w}$ is given by $n = \sum_{s \in \mathcal{S}} n_{s}$. In order to show that tabular policies are value consistent we start by relating the gradient of $V(\hat{s};\bm{w})$ to the gradient of $V(\bar{s};\bm{w})$, where the gradient is taken with respect to the policy parameters of state $\bar{s}$, while the policy parameters of the remaining states are held fixed. 
\begin{Lemma}\label{lemma1}
Suppose we are given a Markov decision process with a tabular policy such that $V(s; \bm{w})$ is differentiable for each $s \in \mathcal{S}$. Given $\bar{s}, \hat{s} \in \mathcal{S}$, such that $\bar{s} \ne \hat{s}$, then we have that
\begin{equation}
\nabla_{\bm{w}_{\bar{s}}} V(\hat{s}; \bm{w}) = p_{\textnormal{hit}}(\hat{s} \to \bar{s}) \nabla_{\bm{w}_{\bar{s}}} V(\bar{s}; \bm{w}), \label{value_fn_grad_lemma_eq1}
\end{equation}
where the notation $\nabla_{\bm{w}_{\bar{s}}} V(\hat{s}; \bm{w})$ is used to denote the gradient of the value function w.r.t. the policy parameter of state $\bar{s}$, with the policy parameters of all other states considered fixed. The term $p_{\textnormal{hit}}(\hat{s} \to \bar{s})$ in (\ref{value_fn_grad_lemma_eq1}) is given by
\begin{equation*}
p_{\textnormal{hit}}(\hat{s} \to \bar{s}) = \sum_{t=2}^\infty \gamma^{t-1} p(s_t = \bar{s} | s_1 = \hat{s}, s_\tau \ne \bar{s}, \tau = 1, ..., t-1; \bm{w}).
\end{equation*}
Furthermore, when Markov chain induced by the policy parameters is ergodic then $p_{\textnormal{hit}} > 0$.
\begin{proof}
Given the equality
\begin{equation*}
V(s; \bm{w}) = \sum_{a \in \mathcal{A}} \pi(a|s;\bm{w}) Q(s, a;\bm{w}),
\end{equation*}
we have that 
\begin{equation*}
\nabla_{\bm{w}_{\bar{s}}} V(\hat{s}; \bm{w}) = \sum_{a \in \mathcal{A}} \bigg( \nabla_{\bm{w}_{\bar{s}}} \pi(a|\hat{s};\bm{w}) Q(\hat{s}, a;\bm{w}) + \pi(a|\hat{s};\bm{w}) \nabla_{\bm{w}_{\bar{s}}} Q(\hat{s}, a;\bm{w}) \bigg). 
\end{equation*}
As the policy is tabular and $\hat{s} \ne \bar{s}$ we have that $\nabla_{\bm{w}_{\bar{s}}} \pi(a|\hat{s};\bm{w}) = \bm{0}$, so that this simplifies to 
\begin{equation*}
\nabla_{\bm{w}_{\bar{s}}} V(\hat{s}; \bm{w}) = \sum_{a \in \mathcal{A}} \pi(a|\hat{s};\bm{w}) \nabla_{\bm{w}_{\bar{s}}} Q(\hat{s}, a;\bm{w}). 
\end{equation*}
Using the fact that $Q(s, a; \bm{w}) = R(s, a) + \gamma \sum_{s' \in \mathcal{S}} p(s'|s, a) V(s';\bm{w})$, we have 
\begin{align}
\nabla_{\bm{w}_{\bar{s}}} V(\hat{s}; \bm{w}) &= \gamma \sum_{s' \in \mathcal{S}} p(s'|\hat{s}; \bm{w}) \nabla_{\bm{w}_{\bar{s}}} V(s';\bm{w})\nn\\
&= \gamma p(\bar{s}|\hat{s}; \bm{w}) \nabla_{\bm{w}_{\bar{s}}} V(\bar{s};\bm{w}) + \gamma \sum_{\stackrel{s' \in \mathcal{S}}{s' \ne \bar{s}} } p(s'|\hat{s}; \bm{w}) \nabla_{\bm{w}_{\bar{s}}} V(s';\bm{w}). \label{lemma1_proof_eq2}
\end{align}
Applying equation (\ref{lemma1_proof_eq2}) recursively gives
\begin{align}
\nabla_{\bm{w}_{\bar{s}}} V(\hat{s}; \bm{w}) &= \sum_{t=2}^\infty \gamma^{t-1} p(s_t = \bar{s} | s_1 = \hat{s}, s_\tau \ne \bar{s}, \tau = 1, ..., t-1; \bm{w}) \nabla_{\bm{w}_{\bar{s}}} V(\bar{s};\bm{w}) \nn\\
&= p_{\textnormal{hit}} (\hat{s} \to \bar{s}) \ \nabla_{\bm{w}_{\bar{s}}} V(\bar{s};\bm{w}), \label{lemma1_proof_eq4}
\end{align}
which completes the proof. The probability, $p(s_t = \bar{s} | s_1 = \hat{s}, s_\tau \ne \bar{s}, \tau = 1, ... , t-1; \bm{w})$, is equivalent to the probability that the first hitting time (of hitting state $\bar{s}$ when starting in state $\hat{s}$) is equal to $t$. The strict inequality, $p_{\textnormal{hit}}(\hat{s} \to \bar{s}) > 0$, follows from the ergodicity of the Markov chain induced by $\bm{w}$.
\end{proof}
\end{Lemma}

We are now ready to prove Theorem~\ref{TabNonDec}.

\begin{reptheorem}{TabNonDec}
\begin{proof}
Suppose that there exists $i \in \mathbb{N}_n$, $\bm{w} \in \mathcal{W}$ and $\hat{s} \in \mathcal{S}$  such that 
\begin{equation*}
\bm{e}_i^\top \nabla_{\bm{w}} V(\hat{s}; \bm{w}) \ne 0,
\end{equation*}
for some $\hat{s} \in \mathcal{S}$. As the policy parametrization is tabular, then the $i^\textnormal{th}$ component of $\bm{w}$ 
corresponds to a policy parameter for a particular state, $\bar{s} \in \mathcal{S}$. From Lemma~\ref{lemma1} it follows that
\begin{equation*}
\frac{d}{d w_i} V(s; \bm{w}) = p_{\textnormal{hit}}(s \to \bar{s}) \ \frac{d}{d w_i} V(\bar{s};\bm{w}),
\end{equation*}
for all $s \in \mathcal{S}$. It follows that for states, $s \in \mathcal{S}$, for which $p_{\textnormal{hit}}(s \to \bar{s}) > 0$ 
that we have  
\begin{equation*}
\textnormal{sign} \big( \bm{e}_i^\top \nabla_{\bm{w}} V(s; \bm{w}) \big) = \textnormal{sign}(\bm{e}_i^\top \nabla_{\bm{w}} V(\hat{s}; \bm{w})),
\end{equation*}
while in states for which $p_{\textnormal{hit}}(s \to \bar{s}) = 0$ we have  
\begin{equation*}
\textnormal{sign} \big( \bm{e}_i^\top \nabla_{\bm{w}} V(s; \bm{w}) \big) = 0.
\end{equation*}
It remains to show that for states in which $p_{\textnormal{hit}}(s \to \bar{s}) = 0$ that 
\begin{equation*}
\textnormal{sign} \big( \bm{e}_i^\top \nabla_{\bm{w}} \pi(a|s; \bm{w}) \big) = 0, \quad \quad \quad  \forall a \in \mathcal{A}.
\end{equation*}
This property follows immediately from the fact that the policy parametrization is tabular and $p_{\textnormal{hit}}(\bar{s} \to \bar{s}) \ne 0$. 
\end{proof}
\end{reptheorem}

\subsection{Proof of Theorem~\ref{lem:nondecreasing}}\label{supp_hessian_analysis_section_local_optimum}

\begin{reptheorem}{lem:nondecreasing}
\begin{proof}
In order to obtain a contradiction suppose that $\bm{w}^*$ is not a stationary point of $V(s;\bm{w})$, for each $s \in \mathcal{S}$. This means that there exists $i \in \mathbb{N}_n$ and $\hat{s} \in \mathcal{S}$ such that $\bm{e}_i^\top \nabla_{\bm{w} | \bm{w} = \bm{w}^*} V(\hat{s}; \bm{w}) \ne 0$. We suppose that $\bm{e}_i^\top \nabla_{\bm{w} | \bm{w} = \bm{w}^*} V(\hat{s}; \bm{w}) > 0$ (an identical argument can be used for the case $\bm{e}_i^\top \nabla_{\bm{w} | \bm{w} = \bm{w}^*} V(\hat{s}; \bm{w}) < 0$). As the policy parametrization is value consistent it follows that, for each $s \in \mathcal{S}$,
\begin{equation}
\bm{e}_i^\top \nabla_{\bm{w} | \bm{w} = \bm{w}^*} V(s; \bm{w}) \ge 0. \label{non_decreasing_lemma_proof_eq1}
\end{equation}

In order to obtain a contradiction we will show that there is no $s \in \mathcal{S}$ for which (\ref{non_decreasing_lemma_proof_eq1}) holds with equality. Given this property a contradiction is obtained because it follows that 
\begin{align*}
\bm{e}_i^\top \nabla_{\bm{w} | \bm{w} = \bm{w}^*} U(\bm{w}) &= \mathbb{E}_{p_1(\bm{s})} \bigg[ \bm{e}_i^\top \nabla_{\bm{w} | \bm{w} = \bm{w}^*}V(\bm{s}; \bm{w}) \bigg] > 0, 
\end{align*}
contradicting the fact that $\bm{w}^*$ is a local optimum of the objective function. Introducing the notation 
\begin{align*}
\mathcal{S}_= &= \big\{s \in \mathcal{S} \, \big| \, \bm{e}_i^\top \nabla_{\bm{w} | \bm{w} = \bm{w}^*} V(s; \bm{w}) = 0 \big\}, \\
\mathcal{S}_> &= \{s \in \mathcal{S} \, | \, \bm{e}_i^\top \nabla_{\bm{w} | \bm{w} = \bm{w}^*} V(s; \bm{w}) > 0 \},
\end{align*}
we wish to show that $\mathcal{S}_= = \emptyset$. In particular, for a contradiction, suppose that $\mathcal{S}_= \ne \emptyset$. This means, given the ergodicity of the Markov chain induced by $\bm{w}^*$ and the fact that $\mathcal{S}_> \ne \emptyset$, that there exists $s \in \mathcal{S}_=$ and  $s' \in \mathcal{S}_>$ such that
\begin{equation*}
p(s'|s; \bm{w}^*) = \sum_{a \in \mathcal{A}} p(s'|s, a) \pi(a|s; \bm{w}^*) > 0.
\end{equation*}
We now consider the form of $\nabla_{\bm{w} | \bm{w} = \bm{w}^*} V(s; \bm{w})$. In particular, we have
\begin{align*}
\nabla_{\bm{w}} V(s; \bm{w}) &= \sum_{a \in \mathcal{A}} \nabla_{\bm{w}} \pi(a|s; \bm{w}) \bigg( R(a, s) + \gamma \sum_{s_{\textnormal{next}} \in \mathcal{S}} p(s_{\textnormal{next}}|s, a) V(s_{\textnormal{next}}; \bm{w}) \bigg), \\
&\quad + \gamma \sum_{a \in \mathcal{A}} \pi(a|s; \bm{w}) \sum_{s_{\textnormal{next}} \in \mathcal{S}} p(s_{\textnormal{next}}|s, a) \nabla_{\bm{w}} V(s_{\textnormal{next}}; \bm{w}).
\end{align*}
As $s \in \mathcal{S}_=$, we have by value consistency that 
\begin{equation*}
\bm{e}_i^\top \nabla_{\bm{w} | \bm{w} = \bm{w}^*} \pi (a|s; \bm{w}) = 0.
\end{equation*}
This means that 
\begin{align*}
\bm{e}_i^\top \nabla_{\bm{w}| \bm{w} = \bm{w}^*} V(s; \bm{w}) &= \gamma \sum_{a \in \mathcal{A}} \pi(a|s; \bm{w}) \sum_{s_{\textnormal{next}} \in \mathcal{S}} p(s_{\textnormal{next}}|s, a) \bm{e}_i^\top \nabla_{\bm{w}| \bm{w} = \bm{w}^*} V(s_{\textnormal{next}}; \bm{w}) > 0.
\end{align*}
The inequality follows from the fact that $p(s'|s; \bm{w}^*) > 0$, for some $s' \in \mathcal{S}_>$. This is a contradiction of the fact that $s_= \in \mathcal{S}_=$, so it follows that $\mathcal{S}_= = \emptyset$ and for all $s \in \mathcal{S}$ we have 
\begin{equation*}
\bm{e}_i^\top \nabla_{\bm{w}| \bm{w} = \bm{w}^*} V(s; \bm{w}) > 0,
\end{equation*}
which completes the proof. 
\end{proof}
\end{reptheorem}

\subsection{Proof of Theorem~\ref{AffineInvariance}}\label{affine_invariance_appendix}

\begin{reptheorem}{AffineInvariance}
\begin{proof}
We shall consider the second Gauss-Newton method, with the 
result for the diagonal Gauss-Newton method following similarly.
Given a non-singular affine transformation, $T \in \mathbb{R}^{n \times n}$, define the objective, 
$\widehat{U}(\bm{w}) = U(T \bm{w}) = U(\bm{v})$, with $\bm{v} = T \bm{w}$, and denote the approximate Hessian of $\widehat{U}(\bm{w})$ 
by $\widehat{\mathcal{H}}_2(\bm{w})$. Given $\bm{w} \in \mathcal{W}$, then it is
sufficient to show that,
\begin{equation*}
T \bm{w}_{\textnormal{new}} = T \bigg( \bm{w} - \alpha \widehat{\mathcal{H}}_2^{-1}(\bm{w}) \nabla_{\bm{w}} \widehat{U}(\bm{w}) \bigg) = \bm{v} - \alpha \mathcal{H}_2^{-1}(\bm{v}) \nabla_{\bm{v}} U(\bm{v}) = \bm{v}_{\textnormal{new}}, \qquad \qquad \forall \alpha \in \mathbb{R}^+.
\end{equation*}
Following calculations analogous to those in Section~\ref{app:pol_grad_theorem} it can be shown that,
\begin{align*}
\nabla_{\bm{w}} \widehat{U}(\bm{w}) &= \sum_{s,a }p_{\gamma}(s,a; T \bm{w}) Q(s,a;T \bm{w})  \nabla_{\bm{w}} \log p(\bm{a}|\bm{s}; T \bm{w}) , \\
\widehat{\cH}_2(\bm{w}) &=  \sum_{s,a }p_{\gamma}(s,a; T \bm{w}) Q(s,a;T \bm{w}) \nabla^\top_{\bm{w}} \log p(\bm{a}|\bm{s}; T \bm{w}) .
\end{align*}
Using the relations
\begin{align*}
\nabla_{\bm{w}} \log \pi(a|s; T \bm{w}) &= T^\top \nabla_{\bm{v}} \log \pi(a|s; \bm{v}), \\
\nabla_{\bm{w}} \nabla^T_{\bm{w}} \log \pi(a|s; T \bm{w}) &= T^\top \nabla_{\bm{v}} \nabla^\top_{\bm{v}} \log \pi(a|s; \bm{v}) T,
\end{align*}
it follows that 
\begin{align*}
\nabla_{\bm{w}} \widehat{U}(\bm{w}) & = T^\top \nabla_{\bm{v}} U(\bm{v}), \\
\widehat{\mathcal{H}}_2(\bm{w}) & = T^\top \mathcal{H}_2(\bm{v}) T.
\end{align*}
From this we have, for any $\alpha \in \mathbb{R}^+$, that
\begin{equation*}
T \bm{w}_{\textnormal{new}} = T \bigg( \bm{w} - \alpha \widehat{\mathcal{H}}_2^{-1}(\bm{w}) \nabla_{\bm{w}} \widehat{U}(\bm{w}) \bigg) = \bm{v} - \alpha \mathcal{H}_2^{-1}(\bm{v}) \nabla_{\bm{v}} U(\bm{v}) = \bm{v}_{\textnormal{new}}, \qquad \qquad \forall \alpha \in \mathbb{R}^+.
\end{equation*}
which completes the proof.
\end{proof}
\end{reptheorem}

\subsection{Proofs of Theorems~\ref{FAN_lemma_convergence_lemma} and \ref{SAN_lemma_convergence_lemma}}\label{convergence_analysis_appendix}

We begin by stating a well-known tool for analysis of convergence of iterative optimization methods. Given an iterative optimization method, defined through a mapping $G : \mathcal{W} \to \mathbb{R}^n$, where $ \mathcal{W} \subseteq \mathbb{R}^n$, the local convergence at a point $\bm{w}^* \in \mathcal{W}$ is determined by the spectral radius of 
the Jacobian of $G$ at $\bm{w}^*$,
$\nabla_{\bm{w} | \bm{w} = \bm{w}^*} G(\bm{w})$. 
This is formalized through the well-known Ostrowski's Theorem, a formal proof of which can be found in \cite{ortega-rheinboldt}. 
\begin{Lemma}[Ostrowski's Theorem]
Suppose that we have a mapping $G : \mathcal{W}  \to \mathbb{R}^n$, where $ \mathcal{W} \subset \mathbb{R}^n$, such that $\bm{w}^* \in \textnormal{int}(\mathcal{W})$ is a fixed-point of $G$ and, furthermore, $G$ is Fr\'{e}chet differentiable at $\bm{w}^*$. If the spectral radius of $\nabla G(\bm{w}^*)$ satisfies $\rho(\nabla G(\bm{w}^*)) < 1$, then $\bm{w}^*$ is a point of attraction of $G$. Furthermore, if $\rho(\nabla G(\bm{w}^*)) > 0$, then the convergence towards $\bm{w}^*$ is linear and the rate is given by $\rho(\nabla G(\bm{w}^*))$.
\end{Lemma}

We now prove Theorems \ref{FAN_lemma_convergence_lemma} and \ref{SAN_lemma_convergence_lemma}.

\begin{reptheorem}{FAN_lemma_convergence_lemma}[Convergence analysis for the first Gauss-Newton method]
\begin{proof}
A formal proof that $G_1$ is Fr\'{e}chet differentiable can be found in Section~10.2.1 of \cite{ortega-rheinboldt}.
We now demonstrate the form of $\nabla G_1(\bm{w}^*)$. For simplicity we shall assume that $(\mathcal{A}_1(\bm{w}^*) + \mathcal{A}_2(\bm{w}^*))^{-1}$ is differentiable. This is not a necessary condition, and a proof that does not make this assumption can be found in Section~10.2.1 of \cite{ortega-rheinboldt}. We have that,
\begin{equation*}
G_1(\bm{w}) = \bm{w} - \alpha (\mathcal{A}_1(\bm{w}) + \mathcal{A}_2(\bm{w}))^{-1} \nabla_{\bm{w}}^\top U(\bm{w}),
\end{equation*}
so that $\nabla_{\bm{w}} G_1(\bm{w})$ is given by 
\begin{equation*}
\nabla_{\bm{w}} G_1(\bm{w}) = I - \alpha \nabla_{\bm{w}} (\mathcal{A}_1(\bm{w}) + \mathcal{A}_2(\bm{w}))^{-1} \nabla^\top U(\bm{w}) - \alpha (\mathcal{A}_1(\bm{w}) + \mathcal{A}_2(\bm{w}))^{-1} \nabla_{\bm{w}} \nabla_{\bm{w}}^\top U(\bm{w}).
\end{equation*}
The fact that $\nabla_{\bm{w}|\bm{w} = \bm{w}^*} U(\bm{w}) = \bm{0}$ means that   
\begin{equation*}
\nabla G_1(\bm{w}^*) = I - \alpha (\mathcal{A}_1(\bm{w}^*) + \mathcal{A}_2(\bm{w}^*))^{-1} \mathcal{H}(\bm{w}^*).
\end{equation*}
As $\mathcal{H}(\bm{w}^*)$ and $\mathcal{A}_1(\bm{w}^*) + \mathcal{A}_2(\bm{w}^*)$ are negative-definite, it follows that 
the eigenvalues of $(\mathcal{A}_1(\bm{w}^*) + \mathcal{A}_2(\bm{w}^*))^{-1} \mathcal{H}(\bm{w}^*)$ are positive. Hence, 
\begin{equation}
\rho(\nabla G_1(\bm{w}^*)) = \max \big\{ |1 - \alpha \lambda_{\textnormal{min}}|, |1 - \alpha \lambda_{\textnormal{max}}| \big\},
\end{equation}
with $\lambda_{\textnormal{min}}$ and $\lambda_{\textnormal{max}}$ respectively denoting the minimal and maximal eigenvalues of $(\mathcal{A}_1(\bm{w}^*) + \mathcal{A}_2(\bm{w}^*))^{-1} \mathcal{H}(\bm{w}^*)$. Hence, $\rho(\nabla G_1(\bm{w}^*)) < 1$ provided that $\alpha \in (0, 2 \lambda_{\textnormal{max}}^{-1})$, or, written in terms of the spectral radius, $\alpha \in (0, 2 / \rho( (\mathcal{A}_1(\bm{w}^*) + \mathcal{A}_2(\bm{w}^*) )^{-1} \mathcal{H}(\bm{w}^*) ) )$.


When the policy parametrization is value consistent with respect to the given MDP, then from Theorem~\ref{lem:nondecreasing}
$\mathcal{H}_{12}(\bm{w}^*) + \mathcal{H}_{12}^\top(\bm{w}^*) = \bm{0}$, so 
that $\mathcal{H}(\bm{w}^*) = \mathcal{A}_1(\bm{w}^*) + \mathcal{A}_2(\bm{w}^*)$. It then follows that 
$\nabla G_1(\bm{w}^*)= (1 - \alpha) I$. Convergence for this case follows in the same manner.
\end{proof}
\end{reptheorem}

\begin{reptheorem}{SAN_lemma_convergence_lemma}[Convergence analysis for the second Gauss-Newton method]
\begin{proof}
The formulas (\ref{G2_Jacobian}) and (\ref{G2_Jacobian_value_consistent}) follow as in the proof of Theorem~\ref{FAN_lemma_convergence_lemma}. Using the same approach as in Theorem~\ref{FAN_lemma_convergence_lemma}, it can be shown that $\rho( \nabla G_2(\bm{w}^*) ) < 1$ provided that, $\alpha \in (0, 2/ \rho(\mathcal{H}_2(\bm{w}^*)^{-1}  \mathcal{H}(\bm{w}^*)) )$.

As $\mathcal{H}(\bm{w}^*)$ and $\mathcal{H}_2(\bm{w}^*)$ are negative-definite the eigenvalues of $\mathcal{H}_2(\bm{w}^*)^{-1} \mathcal{H}(\bm{w}^*)$ are positive. Furthermore, as $\mathcal{H}(\bm{w}^*) = \mathcal{H}_{11}(\bm{w}^*) + \mathcal{H}_2(\bm{w}^*)$, and, by Lemma~\ref{H11_is_positive_definite}, $\mathcal{H}_{11}(\bm{w}^*)$ is positive-definite, it follows that the eigenvalues of $\mathcal{H}_2(\bm{w}^*)^{-1} \mathcal{H}(\bm{w}^*)$ all lie in the range $(0,1)$. This means that $\alpha \in (0,2)$ is sufficient to ensure that $\rho( \nabla G_2(\bm{w}^*) ) < 1$.

\end{proof}
\end{reptheorem}

\subsection{Proof of Theorem~\ref{theorem:analysis_em_generalized_gradient}} \label{EMconnectionProof}

\begin{reptheorem}{theorem:analysis_em_generalized_gradient}
\begin{proof} 
We use the notation $\nabla^{10}_{\bm{w}} \mathcal{Q}(\bm{w}_{j}, \bm{w}_k)$ to denote the derivative with respect to the first variable of $Q$, evaluated at $(\bm{w}_{j}, \bm{w}_k)$, and similarly $\nabla^{20}_{\bm{w}} \mathcal{Q}(\bm{w}_{j}, \bm{w}_k)$ for the second derivative and $\nabla^{01}_{\bm{w}} \mathcal{Q}(\bm{w}_{j}, \bm{w}_k)$ for the derivative with respect to the second variable etc. The idea of the proof is simple and consists of performing a Taylor expansion of $\nabla_{\bm{w}}^{10} \mathcal{Q}(\bm{w}, \bm{w}_k)$. As $\mathcal{Q}$ is assumed to be twice continuously differentiable in the first component this Taylor expansion is possible and gives 
\begin{equation}
\nabla^{10}_{\bm{w}} \mathcal{Q}(\bm{w}_{k+1}, \bm{w}_k) = \nabla^{10}_{\bm{w}} \mathcal{Q}(\bm{w}_k, \bm{w}_k) + \nabla^{20}_{\bm{w}} \mathcal{Q}(\bm{w}_k, \bm{w}_k) (\bm{w}_{k+1} - \bm{w}_k) + \mathcal{O}(\|\bm{w}_{k+1} - \bm{w}_k\|^2). \label{analysis_em_generalized_gradient_appendix}
\end{equation}
As $\bm{w}_{k+1} = \argmax{\bm{w} \in \mathcal{W}} \mathcal{Q}(\bm{w}, \bm{w}_k)$ it follows that $\nabla^{10}_{\bm{w}}\mathcal{Q}(\bm{w}_{k+1}, \bm{w}_k)  = 0$.
This means that, upon ignoring higher order terms in $\bm{w}_{k+1} - \bm{w}_k$, the Taylor expansion (\ref{analysis_em_generalized_gradient_appendix}) can be rewritten into the form
\begin{equation}
\bm{w}_{k+1} - \bm{w}_k = - \nabla^{20}_{\bm{w}} \mathcal{Q}(\bm{w}_k, \bm{w}_k)^{-1} \nabla^{10}_{\bm{w}} \mathcal{Q}(\bm{w}_k, \bm{w}_k).
\end{equation}
The proof is completed by observing that 
\begin{align*}
\nabla^{10}_{\bm{w}} \mathcal{Q}(\bm{w}_k, \bm{w}_k) &= \nabla_{\bm{w}|\bm{w} = \bm{w}_k} U(\bm{w}), \qquad \qquad \nabla^{20}_{\bm{w}} \mathcal{Q}(\bm{w}_k, \bm{w}_k) = \mathcal{H}_2(\bm{w}_k). 
\end{align*}
The second statement follows because in the case where the $\log$-policy is quadratic the higher order terms in the Taylor expansion vanish.
\end{proof}
\end{reptheorem}

\subsection{Proof of Theorem~\ref{EM_lemma_convergence_lemma}}\label{ConvEM_appendix}

\begin{reptheorem}{EM_lemma_convergence_lemma}
\begin{proof}
In the EM-algorithm the update of the policy parameters takes the form
\begin{equation*}
G_{\textnormal{EM}}(\bm{w}_{k}) = \argmax{\bm{w} \in \mathcal{W}} \mathcal{Q}(\bm{w}, \bm{w}_k),
\end{equation*}
where the function $\mathcal{Q}(\bm{w}, \bm{w}')$ is given by
\begin{equation*}
\mathcal{Q}(\bm{w}, \bm{w}') = \sum_{(s,a) \in \mathcal{S} \times \mathcal{A}} p_{\gamma}(s,a;\bm{w}') Q(s,a;\bm{w}') \bigg[ \log \pi (a|s; \bm{w}) \bigg].
\end{equation*}
Note that $\mathcal{Q}$ is a two parameter function, where the first parameter occurs inside the bracket, while the second parameter occurs outside the bracket. Also note that $\mathcal{Q}(\bm{w}, \bm{w}')$ satisfies the following identities  
\begin{align*}
\nabla^{10} \mathcal{Q}(\bm{w}, \bm{w}') &= \sum_{(s,a) \in \mathcal{S} \times \mathcal{A}} p_{\gamma}(s,a;\bm{w}') Q(s,a;\bm{w}') \bigg[ \nabla_{\bm{w}} \log \pi (a|s; \bm{w}) \bigg], \\
\nabla^{20} \mathcal{Q}(\bm{w}, \bm{w}') &= \sum_{(s,a) \in \mathcal{S} \times \mathcal{A}} p_{\gamma}(s,a;\bm{w}') Q(s,a;\bm{w}') \bigg[ \nabla_{\bm{w}} \nabla^\top_{\bm{w}} \log \pi (a|s; \bm{w}) \bigg], \\
\nabla^{11} \mathcal{Q}(\bm{w}, \bm{w}') &= \sum_{(s,a) \in \mathcal{S} \times \mathcal{A}} \nabla_{\bm{w}'} \bigg( p_{\gamma}(s,a;\bm{w}') Q(s,a;\bm{w}') \bigg) \nabla^\top_{\bm{w}} \log \pi (a|s\bm{s}; \bm{w}) .
\end{align*}
Here we have used the notation $\nabla^{ij}$ to denote the $i^{\textnormal{th}}$ derivative with respect to the first parameter and the 
$j^{\textnormal{th}}$ derivative with respect to the second parameter. Note that when we set $\bm{w} = \bm{w}'$ in the first two of these terms 
we have $\nabla^{10} \mathcal{Q}(\bm{w}, \bm{w}) = \nabla_{\bm{w}} U(\bm{w})$,  $\nabla^{20} \mathcal{Q}(\bm{w}, \bm{w}) = \mathcal{H}_2(\bm{w})$. 
A key identity that we need for the proof is that $\nabla^{11} \mathcal{Q}(\bm{w}, \bm{w}) = \mathcal{H}_1(\bm{w}) + \mathcal{H}_{12}(\bm{w}) + \mathcal{H}_{12}^\top(\bm{w})$. This follows from the observation that $\nabla_{\bm{w}} U(\bm{w}) = \nabla^{10} \mathcal{Q}(\bm{w}, \bm{w})$, so that 
\begin{align*}
\nabla_{\bm{w}} \nabla_{\bm{w}}^\top U(\bm{w}) &= \nabla_{\bm{w}} \bigg( \nabla^{10} \mathcal{Q}(\bm{w}, \bm{w}) \bigg) = \nabla^{20} \mathcal{Q}(\bm{w}, \bm{w}) + \nabla^{11} \mathcal{Q}(\bm{w}, \bm{w}),
\end{align*}
so that 
\begin{align*}
\mathcal{H}_1(\bm{w}) + \mathcal{H}_{12}(\bm{w}) + \mathcal{H}_{12}^\top(\bm{w}) &= \mathcal{H}(\bm{w}) - \mathcal{H}_2(\bm{w}) = \nabla^{20} \mathcal{Q}(\bm{w}, \bm{w}) + \nabla^{11} \mathcal{Q}(\bm{w}, \bm{w}) - \nabla^{20} \mathcal{Q}(\bm{w}, \bm{w}), \\
&= \nabla^{11} \mathcal{Q}(\bm{w}, \bm{w}),
\end{align*}
as claimed.

Now, to calculate the matrix $\nabla G_{\textnormal{EM}} (\bm{w}^*)$ we perform a Taylor series expansion of $\nabla_{\bm{w}}^{10} \mathcal{Q}(\bm{w}, \bm{w}')$ in both parameters around the point $(\bm{w}^*, \bm{w}^*)$, 
and evaluated at $(\bm{w}_{k+1}, \bm{w}_k)$, which gives
\begin{align*}
\nabla_{\bm{w}}^{10} \mathcal{Q}(\bm{w}_{k+1}, \bm{w}_{k}) = \nabla_{\bm{w}}^{10} \mathcal{Q}(\bm{w}^*, \bm{w}^*) + \nabla_{\bm{w}}^{20} \mathcal{Q}(\bm{w}^*, \bm{w}^*) & \big( \bm{w}_{k+1} - \bm{w}^* \big) \\
& + \nabla_{\bm{w}}^{11} \mathcal{Q}(\bm{w}^*, \bm{w}^*) \big( \bm{w}_{k} - \bm{w}^* \big) + \dots
\end{align*}
As $\bm{w}^*$ is a local optimum of $U(\bm{w})$ we have that $\nabla_{\bm{w}}^{10} \mathcal{Q}(\bm{w}^*, \bm{w}^*) = \bm{0}$. Furthermore, as the sequence $\{\bm{w}_{k}\}_{k \in \mathbb{N}}$ was generated by the EM-algorithm, we have, for each $k \in \mathbb{N}$, that $\bm{w}_{k+1} = \argmax{\bm{w} \in \mathcal{W}} \mathcal{Q}(\bm{w}, \bm{w}_k)$, which implies that $\nabla_{\bm{w}}^{10} \mathcal{Q}(\bm{w}_{k+1}, \bm{w}_{k}) = \bm{0}$. Finally, as  $\nabla^{20} \mathcal{Q}(\bm{w}^*, \bm{w}^*) = \mathcal{H}_2(\bm{w}^*)$ and $\nabla^{11} \mathcal{Q}(\bm{w}^*, \bm{w}^*) = \mathcal{H}_1(\bm{w}^*)$ we have
\begin{align*}
\bm{0} = \mathcal{H}_2(\bm{w}^*) (\bm{w}_{k+1} - \bm{w}^*) + \big( \mathcal{H}_1(\bm{w}^*) + \mathcal{H}_{12}(\bm{w}^*) + \mathcal{H}_{12}^\top(\bm{w}^*) \big) (\bm{w}_{k} - \bm{w}^*) + \dots
\end{align*}
Using the fact that $\bm{w}_{k+1} = G_{\textnormal{EM}}(\bm{w}_{k})$ and $\bm{w}^* = G_{\textnormal{EM}}(\bm{w}^*)$, taking the limit $k \to \infty$ gives
\begin{equation*}
\bm{0} = \mathcal{H}_2(\bm{w}^*) \nabla_{\bm{w}} G_{\textnormal{EM}} (\bm{w}^*) + \mathcal{H}_1(\bm{w}^*) + \mathcal{H}_{12}(\bm{w}^*) + \mathcal{H}_{12}^\top(\bm{w}^*),
\end{equation*}
so that 
\begin{equation*}
\nabla_{\bm{w}} G_{\textnormal{EM}} (\bm{w}^*) = - \mathcal{H}^{-1}_2(\bm{w}^*) \big( \mathcal{H}_1(\bm{w}^*) + \mathcal{H}_{12}(\bm{w}^*) + \mathcal{H}_{12}^\top(\bm{w}^*) \big) = I - \mathcal{H}_2^{-1}(\bm{w}^*) \mathcal{H}(\bm{w}^*).
\end{equation*}
In the case where the policy parametrization value consistent with respect to the given MDP then we have $\mathcal{H}_{12}(\bm{w}^*) + \mathcal{H}_{12}(\bm{w}^*)^\top = \bm{0}$, so that $\nabla_{\bm{w}} G_{\textnormal{EM}} (\bm{w}^*) = I - \mathcal{H}^{-1}_2(\bm{w}^*) \mathcal{A}_1(\bm{w}^*)$.
The rest of the proof follows from the result in Theorem~\ref{SAN_lemma_convergence_lemma} when considering $\alpha=1$.
\end{proof}
\end{reptheorem}

\section{Further Details for Estimation of Preconditioners and the Gauss-Newton Update Direction}

\subsection{Recurrent State Search Direction Evaluation for Second Gauss-Newton Method}\label{sec_supp_recurrent_state_eval}

In \cite{williams92} a sampling algorithm was provided for estimating the gradient of an infinite horizon MDP with average rewards. This algorithm makes use of a recurrent state, which we denote by $\bm{s}^*$. In Algorithm~\ref{alg:recurrent_state_inference} we detail a straightforward extension of this algorithm to the estimation the approximate Hessian, $\mathcal{H}_2(\bm{w})$, in this MDP framework. The analogous algorithm for the estimation of the diagonal matrix, $\mathcal{D}_2(\bm{w})$, follows similarly. In Algorithm~\ref{alg:recurrent_state_inference} we make use of an eligibility trace for both the gradient and the approximate Hessian, which we denote by $\bm{\Phi}^1$ and $\bm{\Phi}^2$ respectively. The estimates (up to a positive scalar) of the gradient and the approximate Hessian are denoted by $\bm{\Delta}^1$ and $\bm{\Delta}^2$ respectively.

\begin{algorithm}[tb]
  \caption{Recurrent state sampling algorithm to estimate the search direction of the second Gauss-Newton method. The algorithm is applicable to Markov decision processes with an infinite planning horizon and average rewards.}
   \label{alg:recurrent_state_inference}

  \KwIn{Policy parameter, $\bm{w} \in \mathcal{W}$, 
  \newline
  Number of restarts, $N \in \mathbb{N}$.}

  \begin{spacing}{1.05}
  \end{spacing}

  {Sample a state from the initial state distribution:}
  \begin{equation*}
  s_1 \sim p_1(\cdot).
  \end{equation*}

  \For{ $i= 1, ....., N$}{

  \begin{spacing}{1.5}
  \end{spacing}
 
  {Given the current state, sample an action from the policy:}
  \begin{equation*}
  a_t \sim \pi(\cdot|s_t;\bm{w}).
  \end{equation*}

  \eIf{ $s_t \ne s^*$,}{
  update the eligibility traces:
  
  \begin{align*}
  \bm{\Phi}^1 \leftarrow \bm{\Phi}^1 + \nabla_{\bm{w}} \log \pi (a_t|s_t;\bm{w}) \quad \quad \quad \quad  \bm{\Phi}^2 \leftarrow \bm{\Phi}^2 + \nabla_{\bm{w}} \nabla^\top_{\bm{w}} \log \pi (a_t|s_t;\bm{w})
  \end{align*}
  }
  {reset the eligibility traces:
  \begin{align*}
  \bm{\Phi}^1 = \bm{0}, \quad \quad \quad \quad \quad \bm{\Phi}^2 = \bm{0}.
  \end{align*}
  }

  \begin{spacing}{1.05}
  \end{spacing}  
  
  Update the estimates of the $\nabla_{\bm{w}} U(\bm{w})$ and $\mathcal{H}_2(\bm{w})$: 
  \begin{align*}
  \bm{\Delta}^1 \leftarrow \bm{\Delta}^1 + R(a_t, s_t) \bm{\Phi}^1, \quad \quad \quad \quad \bm{\Delta}^2 \leftarrow \bm{\Delta}^2 + R(a_t, s_t) \bm{\Phi}^2.
  \end{align*}
  \newline
  Sample state from the transition dynamics:
  \begin{equation*}
  s_{t+1} \sim p(\cdot|a_t, s_t).
  \end{equation*}

  Update time-step, $t \leftarrow t+1$.
  }
  
  \KwRet{$\bm{\Delta}^1$ \textnormal{and} $\bm{\Delta}^2$, \textnormal{which, up to a positive multiplicative constant, are estimates of $\nabla_{\bm{w}} U(\bm{w})$ and $\mathcal{H}_2(\bm{w})$.}}
\end{algorithm}

\subsection{Inversion of Preconditioning Matrices} \label{HessianInversion}

A computational bottleneck of the Newton method is the inversion of the Hessian matrix, which scales with $\mathcal{O}(n^3)$. In a standard application of the Newton method this inversion is performed during each iteration, and in large parameter systems this becomes prohibitively costly. We now consider the inversion of the preconditioning matrix in proposed Gauss-Newton methods. 

Firstly, in the diagonal forms of the Gauss-Newton methods the preconditioning matrix is diagonal, so that the inversion of this matrix is trivial and scales linearly in the number of parameters. In general the preconditioning matrix of the full Gauss-Newton methods will have no form of sparsity, and so no computational savings will be possible when inverting the preconditioning matrix. There is, however, a source of sparsity that allows for the efficient inversion of $\mathcal{H}_2$ in certain cases of interest. In particular, any product structure (with respect to the control parameters) in the model of the agent's behaviour will lead to sparsity in $\mathcal{H}_2$. For example, in partially observable Markov decision processes in which the behaviour of the agent is modeled through a finite state controller \citep{DBLP:conf/uai/MeuleauPKK99} there are three functions that are to be optimized, the initial belief distribution, 
the belief transition dynamics and the policy. In this case the dynamics of the system are given by,
\begin{equation*}
p(s',o',b',a'|s,o,b,a; \bm{v}, \bm{w}) = p(s'|s,a) p(o'|s') p(b'|b,o';\bm{v}) \pi(a'|b',o'; \bm{w}),
\end{equation*}
in which $o \in \mathcal{O}$ is an observation from a finite observation space, $\mathcal{O}$, and $b \in \mathcal{B}$ is the 
belief state from a finite belief space, $\mathcal{B}$. The initial belief is given by the initial belief distribution, $p(b|o;\bm{u})$.
The parameters to be optimized in this system are $\bm{u}$, $\bm{v}$ and $\bm{w}$.
It can be seen that in this system $\mathcal{H}_2(\bm{u}, \bm{v}, \bm{w})$ is block-diagonal (across the parameters 
$\bm{u}$, $\bm{v}$ and $\bm{w}$) and the matrix inversion can be performed more efficiently by inverting each of the block 
matrices individually. By contrast, the Hessian $\mathcal{H}(\bm{u}, \bm{v}, \bm{w})$ does not exhibit any such sparsity properties.

\subsection{A Hessian-free Conjugate Gradient Method for Fast Gauss-Newton Ascent} \label{HessianFree}

In general, the matrix inversion required in the full Gauss-Newton methods scales cubically in the number of policy parameters, which 
will be prohibitively expensive in large parameter systems. It is possible, however, to approximate the search direction of the 
second Gauss-Newton method at a computational cost that is linear in the number of policy parameters. We focus
on this form of the Gauss-Newton method for the remainder of this section. These computational savings are achieved through 
an application of the conjugate-gradient algorithm \citep{hestenes1952}, along with a Hessian-free approximation \citep{nocedal-2006} 
to a matrix-vector product that occurs within the conjugate-gradient algorithm. 

It can be seen that the search direction of the Gauss-Newton method at $\bm{w} \in \mathcal{W}$ is given by the solution to the linear system,
\begin{equation}
- \mathcal{H}_2 (\bm{w}) \bm{x} = \nabla_{\bm{w}} U(\bm{w}), \label{linearSystem1}
\end{equation}
The conjugate-gradient algorithm \citep{hestenes1952} is an iterative algorithm for solving linear systems.
The algorithm maintains an estimate to the solution of the linear system during the course of the algorithm. 
We denote the estimate at the $k^{\textnormal{th}}$ iteration by $\bm{x}_k$. 
The first approximation we propose is the use of $\bm{x}_k$, for some given $k \in \mathbb{N}$, as an approximation to the search direction of the Gauss-Newton method. 
As $- \mathcal{H}_2(\bm{w})$ is positive-definite the conjugate-gradient algorithm is guaranteed to find the 
exact solution of the linear system (\ref{linearSystem1}) within at most $n$ iterations. 
Furthermore, when $\bm{x}_0$ is appropriately selected, then $\bm{x}_k$ will be an ascent direction  
for all $k \in \mathbb{N}_n$. 
This property is guaranteed when, for instance, $\bm{x}_0 = \nabla_{\bm{w}} U(\bm{w})$. 
Each iteration of the conjugate-gradient algorithm scales quadratically in $n$. 
If $\bm{x}_k$ is used in place of $\bm{x}$ in the Gauss-Newton method, then the computational complexity will scale as 
$\mathcal{O}(k n^2)$. When 
$k \ll n$, therefore, the computational complexity of such an approach will be far less than the standard 
application of the Gauss-Newton method. 
When $n$ is large, however, this quadratic scaling in $n$ will still be too prohibitive and for this case we shall now consider an additional level of approximation in order to reduce the computational costs still further.

The computational bottleneck in each iteration of the conjugate-gradient algorithm is a matrix-vector product, that scales quadratically in the size of the linear system. In the case of the Gauss-Newton method, the matrix-vector product in the $k^{\textnormal{th}}$ iteration of the conjugate gradient algorithm takes the form, 
\begin{equation}
- \mathcal{H}_2(\bm{w}) \bm{p}_k = - \sum_{(s,a) \in \mathcal{S} \times \mathcal{A}} p_\gamma(s,a;\bm{w}) Q(s,a;\bm{w}) \nabla_{\bm{w}} \nabla_{\bm{w}}^\top \log \pi(a|s;\bm{w}) \bm{p}_k, \quad \quad \quad k \in \mathbb{N}_n, \label{apxnMatrixVectorProduct}
\end{equation}
in which $\bm{p}_k$ is the $k^{\textnormal{th}}$ conjugate direction found during the conjugate-gradient algorithm.
The matrix-vector product in (\ref{apxnMatrixVectorProduct}) can be equivalently viewed as a weighted summation of matrix-vector products which, for each state-action pair, $(s, a) \in \mathcal{S} \times \mathcal{A}$, we have the matrix-vector product, $\nabla_{\bm{w}} \nabla_{\bm{w}}^\top \log \pi(a|s;\bm{w}) \bm{p}_k$. This perspective of (\ref{apxnMatrixVectorProduct}) allows the use of standard finite-difference approximations to efficiently approximate each of these matrix-vector products, and thus (\ref{apxnMatrixVectorProduct}) itself. In particular, introducing the scalar, $\epsilon \in \mathbb{R}^+$, with $\epsilon \approx 0$, we have that 
\begin{equation}
\nabla_{\bm{w}} \nabla_{\bm{w}}^\top \log \pi(a|s;\bm{w}) \bm{p}_k \approx \frac{1}{\epsilon} \bigg( \nabla_{\bm{w}} \log \pi(a|s;\bm{w} + \epsilon \bm{p}_k) - \nabla_{\bm{w}} \log \pi(a|s;\bm{w}) \bigg), \label{finiteDifferenceMatrixVectorProduct}
\end{equation}
for each $(s,a) \in \mathcal{S} \times \mathcal{A}$. Using (\ref{finiteDifferenceMatrixVectorProduct}) in (\ref{apxnMatrixVectorProduct}) gives  
\begin{equation}
- \mathcal{H}_2(\bm{w}) \bm{p}_k \approx \frac{1}{\epsilon} \sum_{(s,a) \in \mathcal{S} \times \mathcal{A}} p_\gamma(s,a;\bm{w}) Q(s,a;\bm{w}) \bigg( \nabla_{\bm{w}} \log \pi(a|s;\bm{w} ) - \nabla_{\bm{w}} \log \pi(a|s;\bm{w} + \epsilon \bm{p}_k) \bigg). \label{matrixVectorProductApprox1} 
\end{equation}
The use of this approximation removes the necessity to either construct $\mathcal{H}_2(\bm{w})$ or to perform the matrix-vector product, and each iteration of conjugate-gradients now has a computational complexity that is linear in the dimension 
of the parameter space. Using $k$ iterations of conjugate-gradients to approximate the search direction results in a computational 
cost that scales as $\mathcal{O}(kn)$. 
We shall refer to the use of these two approximations (i.e., the use of the conjugate-gradient algorithm to approximately solve the linear system (\ref{linearSystem1}) and the use of the finite-difference approximation (\ref{matrixVectorProductApprox1})) as the conjugate-gradient Gauss-Newton method. A summary of the algorithm is given in Algorithm~\ref{alg:ConjugateGradientApproximateNewtonMethod}.

\begin{algorithm}[tb]
  \caption{The Conjugate-Gradient Gauss-Newton Method}\label{alg:ConjugateGradientApproximateNewtonMethod}
  \KwIn{Initial vector of policy parameters, $\bm{w}_0 \in \mathcal{W}$.}

  Set iteration counter, $k \leftarrow 0$.
  
  \Repeat{ \textnormal{Convergence of the policy parameters}}{
  
  Calculate the gradient of the objective at the current point in the parameter space, $\nabla_{\bm{w} = \bm{w}_k} U(\bm{w})$. 

  \begin{spacing}{1.5}  
  \end{spacing}    
  
  Use the conjugate-gradient algorithm to approximately solve the linear system, 
  \begin{equation}
  - \mathcal{H}_2(\bm{w}_k) \bm{x} = \nabla_{\bm{w} = \bm{w}_k} U(\bm{w}), \label{algLinearSystem} 
  \end{equation}
  using some given stopping criterion in the conjugate-gradient algorithm, 
  and using the finite-difference approximation (\ref{matrixVectorProductApprox1}) to approximate the matrix-vector 
  product (\ref{apxnMatrixVectorProduct}).

  \begin{spacing}{1.5}  
  \end{spacing}      
  
  Update policy parameters, $\bm{w}_{k+1} = \bm{w}_k + \alpha \bm{x}_k$, in which $\bm{x}_k$ is the approximate solution of the 
  linear system (\ref{algLinearSystem}) and $\alpha \in \mathbb{R}^+$ is the step size.
  
  \begin{spacing}{1.5}  
  \end{spacing}      
  
  Update iteration counter, $k \leftarrow k+1$.
  
  \begin{spacing}{1.5}  
  \end{spacing}        
  
  }

  \begin{spacing}{2.5}  
  \end{spacing}      
  
  \KwRet{$\bm{w}_k$}

\end{algorithm}

The use of the conjugate-gradient algorithm and the finite-difference approximation (\ref{matrixVectorProductApprox1}), 
are based upon methods used in Hessian-free algorithms \citep{nocedal-2006} from the numerical optimization literature.  
In the case of Markov decision processes, the conjugate-gradient algorithm would be used within an Hessian-free algorithm to solve the linear system, 
\begin{equation}
- \nabla_{\bm{w}} \nabla_{\bm{w}}^\top U(\bm{w}) \bm{x} = \nabla_{\bm{w}} U(\bm{w}). \label{HessianFreeLinearSystemMDP}
\end{equation}
A finite-difference approximation is also applied in Hessian-free methods, in this case taking the form,
\begin{align}
- \nabla_{\bm{w}} \nabla_{\bm{w}}^\top U(\bm{w}) \bm{p}_k \approx \frac{1}{\epsilon} & \bigg( \sum_{(s,a) \in \mathcal{S} \times \mathcal{A}} p_\gamma(s,a;\bm{w}) Q(s,a;\bm{w})  \nabla_{\bm{w}} \log \pi(a|s;\bm{w}) \label{matrixVectorProductApproxHessianFreeMDP} \\ 
& - \sum_{(s,a) \in \mathcal{S} \times \mathcal{A}} p_\gamma(s,a;\bm{w}+\epsilon\bm{p}_k) Q(s,a;\bm{w}+\epsilon\bm{p}_k)  \nabla_{\bm{w}} \log \pi(a|s;\bm{w}+\epsilon\bm{p}_k) \bigg).  \nonumber
\end{align}
Given the similarities between the conjugate-gradient Gauss-Newton method and Hessian-free methods, it is worth noting some important 
differences between the two algorithms. 
Firstly, as the Hessian is not necessarily negative-definite, it is not necessarily the case that the conjugate-gradient 
algorithm will be able to solve the linear system (\ref{HessianFreeLinearSystemMDP}).
It is no longer the case that $\bm{x}_k$, $k \in \mathbb{N}$, will be an ascent direction of the objective
function, regardless of the initialization of the conjugate-gradient algorithm \citep{nocedal-2006}.  
Additionally, comparing the finite-difference approximation (\ref{matrixVectorProductApprox1}) with the finite-difference approximation 
(\ref{matrixVectorProductApproxHessianFreeMDP}) it can be seen that in the rightmost term of 
(\ref{matrixVectorProductApprox1}) the discounted occupancy distribution and the state-action value function depend on the current 
policy parameters, $\bm{w} \in \mathcal{W}$, while in the corresponding term in (\ref{matrixVectorProductApproxHessianFreeMDP}) these 
quantities depend on the perturbed policy parameters, $\bm{w} + \epsilon \bm{p}_k \in \mathcal{W}$. Terms such as the state-action 
value function cannot be calculated exactly in most large-scale MDPs of interest and instead must be estimated. 
In a standard application of a Hessian-free method, therefore, it would be necessary to re-estimate such quantities in 
each iteration of the conjugate-gradient algorithm. By contrast, in the conjugate-gradient Gauss-Newton method the same 
estimate of the state-action value function and discounted occupancy marginals can be used in all of the iterations of the 
conjugate-gradient algorithm. In policy gradient algorithms estimating such terms typically forms an expensive part of the 
overall algorithm, which means that each iteration of the conjugate-gradient Gauss-Newton method will be more 
computationally efficient than Hessian-free methods. It also means that while it may appear that the 
approximation (\ref{matrixVectorProductApprox1}) should have the same cost as two gradient evaluations, 
it will be typically be cheaper than this in practice. Furthermore, it means that there is an additional 
level of approximation in Hessian-free methods that is not present in the conjugate-gradient Gauss-Newton method.

Additionally, by considering the Fisher information matrix that takes the form (\ref{fisher_information2}), 
the approach presented in this section could also be applied in the natural gradient framework.

\vskip 0.2in
\bibliography{pol_grad_bib}

\end{document}

%% file: inputs.tex
\newcolumntype{W}{ >{\raggedleft \arraybackslash} X }

\makeatletter
\newcommand{\vast}{\bBigg@{4}}
\newcommand{\Vast}{\bBigg@{5}}
\makeatother

\newcommand{\expectation}[1]{\mathbb{E}_{#1}}

\newcommand{\cut}[1]{}

\newcommand{\ceql}{\! = \!}

\newcommand{\myset}[1]{\mathcal{\uppercase{#1}}}

\newcommand{\argmax}[1]{\operatorname{argmax}_{#1}\hspace{1mm}}

\newcommand{\RealLine}{\mathbb{R}}
\newcommand{\NaturalNumbers}{\mathbb{N}}
\newcommand{\intSet}[1]{\NaturalNumbers_{#1}}

\newcommand{\beq}{\begin{equation}}
\newcommand{\eeq}{\end{equation}}

\newcommand{\bmp}[1]{\begin{minipage}{#1}}
\newcommand{\emp}{\end{minipage}}

\newcommand{\rhs}{\emph{r.h.s.}\xspace}

\tikzstyle{dgraph}=[->, line width=1.5pt]

\tikzstyle{dec}=[rectangle,draw=red!50,thick,minimum size=6mm]  
\tikzstyle{utility}=[diamond,draw=red!50,thick,minimum size=6mm]  
\tikzstyle{sobs}=[fill=green!15,thick]  
\tikzstyle{cont}=[circle,draw=blue!50,thick,minimum size=6mm,line width=2pt,>=stealth]  

\tikzstyle{oval}=[ellipse,draw=blue!50,thick,minimum size=6mm,line width=2pt,>=stealth]  
\tikzstyle{disc}=[rectangle,draw=blue!50,thick,minimum size=6mm]  
\tikzstyle{obs}=[circle,fill=black!20,thick,draw]  
\tikzstyle{hid}=[circle,draw,thick]  
\tikzstyle{sep}=[rectangle,draw=magenta!50,thick,minimum size=6mm]  
\tikzstyle{det}=[fill=red!15,,rectangle,draw=red!50,thick,minimum size=6mm]  

%% file: GaussNewtonArxiv.bbl
\begin{thebibliography}{86}
\providecommand{\natexlab}[1]{#1}
\providecommand{\url}[1]{\texttt{#1}}
\expandafter\ifx\csname urlstyle\endcsname\relax
  \providecommand{\doi}[1]{doi: #1}\else
  \providecommand{\doi}{doi: \begingroup \urlstyle{rm}\Url}\fi

\bibitem[Abbeel et~al.(2007)Abbeel, Coates, Quigley, and Ng]{Abbeel-Coates}
P.~Abbeel, A.~Coates, M.~Quigley, and A.~Ng.
\newblock {An Application of Reinforcement Learning to Aerobatic Helicopter
  Flight}.
\newblock \emph{{NIPS}}, 19:\penalty0 1--8, 2007.

\bibitem[Amari(1997)]{amari-1996}
S.~Amari.
\newblock {Neural Learning in Structured Parameter Spaces - Natural Riemannian
  Gradient}.
\newblock \emph{NIPS}, 9:\penalty0 127--133, 1997.

\bibitem[Amari(1998)]{amari-1998}
S.~Amari.
\newblock {Natural Gradient Works Efficiently in Learning}.
\newblock \emph{Neural Computation}, 10:\penalty0 251--276, 1998.

\bibitem[Amari et~al.(1992)Amari, Kurata, and Nagaoka]{amari-etal-1992}
S.~Amari, K.~Kurata, and H.~Nagaoka.
\newblock {Information Geometry of Boltzmann Machines}.
\newblock \emph{IEEE Transactions on Neural Networks}, 3\penalty0 (2):\penalty0
  260--271, 1992.

\bibitem[Amari et~al.(1996)Amari, Cichocki, and Yang]{amari-1996b}
S.~Amari, A.~Cichocki, and H.~Yang.
\newblock {A New Learning Algorithm for Blind Signal Separation}.
\newblock \emph{NIPS}, 8:\penalty0 757--763, 1996.

\bibitem[Bagnell and Schneider(2003)]{Bagnell-2003}
J.~Bagnell and J.~Schneider.
\newblock {Covariant Policy Search}.
\newblock \emph{{IJCAI}}, 18:\penalty0 1019--1024, 2003.

\bibitem[Baird(1993)]{Baird93}
L.~Baird.
\newblock {Advantage Updating}.
\newblock Research Report WL-TR-93-1146, Wright Lab, 1993.

\bibitem[Barber(2011)]{barber-book-11}
D.~Barber.
\newblock \emph{Bayesian Reasoning and Machine Learning}.
\newblock Cambridge University Press, 2011.

\bibitem[Baxter and Bartlett(2001)]{baxter-etal-2001}
J.~Baxter and P.~Bartlett.
\newblock {Infinite Horizon Policy Gradient Estimation}.
\newblock \emph{Journal of Artificial Intelligence Research}, 15:\penalty0
  319--350, 2001.

\bibitem[Baxter et~al.(2001)Baxter, Bartlett, and Weaver]{baxter-etal-2001b}
J.~Baxter, P.~Bartlett, and L.~Weaver.
\newblock {Experiments with Infinite Horizon, Policy Gradient Estimation}.
\newblock \emph{Journal of Artificial Intelligence Research}, 15:\penalty0
  351--381, 2001.

\bibitem[Bellman(1957)]{bellman-book}
R.~Bellman.
\newblock \emph{Dynamic Programming}.
\newblock Princeton University Press, 1957.

\bibitem[Bertsekas(2010)]{Bertsekas-2010}
D.~P. Bertsekas.
\newblock {Approximate Policy Iteration: A Survey and Some New Methods}.
\newblock Research report, Massachusetts Institute of Technology, 2010.

\bibitem[Bertsekas and Ioffe(1996)]{Bertsekas-1996}
D.~P. Bertsekas and S.~Ioffe.
\newblock {Temporal Differences-Based Policy Iteration and Applications in
  Neuro-Dynamic Programming}.
\newblock {Research Report LIDS-P-2349}, Massachusetts Institute of Technology,
  1996.

\bibitem[Bhatnagar et~al.(2008)Bhatnagar, Sutton, Ghavamzadeh, and
  Lee]{bhatnagar-incremental-NAC-2008}
S.~Bhatnagar, R.~Sutton, M.~Ghavamzadeh, and M.~Lee.
\newblock {Incremental Natural Actor-Critic Algorithms}.
\newblock \emph{NIPS}, 20:\penalty0 105--112, 2008.

\bibitem[Bhatnagar et~al.(2009)Bhatnagar, Sutton, Ghavamzadeh, and
  Mark]{bhatnagar_2009}
S.~Bhatnagar, R.~Sutton, M.~Ghavamzadeh, and L.~Mark.
\newblock {Natural Actor-Critic Algorithms}.
\newblock \emph{Automatica}, 45:\penalty0 2471--2482, 2009.

\bibitem[Boyd and Vandenberghe(2004)]{boyd-vandenberghe-book}
S.~Boyd and L.~Vandenberghe.
\newblock \emph{Convex Optimization}.
\newblock {Cambridge University Press}, 2004.

\bibitem[Browne et~al.(2012)Browne, Powley, Whitehouse, Lucas, Cowling,
  Rohlfshagen, Tavener, Perez, Samothrakis, and Colton]{Browne-etal-2012}
C.~Browne, E.~Powley, D.~Whitehouse, S.~Lucas, P.~Cowling, P.~Rohlfshagen,
  S.~Tavener, D.~Perez, S.~Samothrakis, and S.~Colton.
\newblock {A Survey of Monte Carlo Tree Search Methods}.
\newblock \emph{IEEE Transactions on Computational Intelligence and AI in
  Games}, 4:\penalty0 1--43, 2012.

\bibitem[Crites and Barto(1995)]{Crites-Barto}
R.~Crites and A.~Barto.
\newblock {Improving Elevator Performance Using Reinforcement Learning}.
\newblock \emph{{NIPS}}, 8:\penalty0 1017--1023, 1995.

\bibitem[Dayan and Hinton(1997)]{Dayan-Hinton-EM-RL-97}
P.~Dayan and G.~E. Hinton.
\newblock {Using Expectation-Maximization for Reinforcement Learning}.
\newblock \emph{Neural Computation}, 9:\penalty0 271--278, 1997.

\bibitem[Deisenroth and Rasmussen(2011)]{Deisenroth2011c}
M.~P. Deisenroth and C.~E. Rasmussen.
\newblock {PILCO: A Model-Based and Data-Efficient Approach to Policy Search}.
\newblock In L.~Getoor and T.~Scheffer, editors, \emph{Proceedings of the 28th
  International Conference on Machine Learning}, Bellevue, WA, USA, June 2011.

\bibitem[Dempster et~al.(1977)Dempster, Laird, and Rubin]{dempster-1977}
A.~P. Dempster, N.~M. Laird, and D.~B. Rubin.
\newblock {Maximum Likelihood from Incomplete Data via the EM Algorithm}.
\newblock \emph{Journal of the Royal Statistical Society. Series B
  (Methodological)}, 39\penalty0 (1):\penalty0 1--38, 1977.

\bibitem[Fahey(2003)]{Fahey_tetris_site}
C.~Fahey.
\newblock {Tetris AI, Computers Play Tetris}
  \url{http://colinfahey.com/tetris/tetris_en.html}, 2003.

\bibitem[Furmston(2012)]{furmston-thesis}
T.~Furmston.
\newblock \emph{Applications of Probabilistic Inference to Planning \&
  Reinforcement Learning}.
\newblock PhD thesis, University College London, 2012.

\bibitem[Furmston and Barber(2009)]{furmston-barber}
T.~Furmston and D.~Barber.
\newblock {Solving Deterministic Policy (PO)MPDs using Expectation-Maximisation
  and Antifreeze}.
\newblock \emph{European Conference on Machine Learning (ECML)}, 1:\penalty0
  50--65, 2009.
\newblock {Workshop on Learning and data Mining for Robotics}.

\bibitem[Furmston and Barber(2010)]{furmston-barber-10}
T.~Furmston and D.~Barber.
\newblock {Variational Methods for Reinforcement Learning}.
\newblock \emph{AISTATS}, 9:\penalty0 241--248, 2010.

\bibitem[Gabillon et~al.(2013)Gabillon, Ghavamzadeh, and
  Scherrr]{gabillon-etal-2013-nips}
V.~Gabillon, M.~Ghavamzadeh, and B.~Scherrr.
\newblock {Approximate Dynamic Programming Finally Performs Well in the Game of
  Tetris}.
\newblock \emph{NIPS}, 26, 2013.

\bibitem[Gelly and Silver(2008)]{gelly-2008}
S.~Gelly and D.~Silver.
\newblock {Achieving Master Level Play in 9 x 9 Computer Go}.
\newblock \emph{{AAAI}}, 23:\penalty0 1537--1540, 2008.

\bibitem[Ghavamzadeh and Engel(2007)]{ghavamzadeh-bayesian-gradient}
M.~Ghavamzadeh and Y.~Engel.
\newblock {Bayesian Policy Gradient Algorithms}.
\newblock \emph{NIPS}, 19:\penalty0 457--464, 2007.

\bibitem[Glynn(1986)]{glynn-86}
P.~W. Glynn.
\newblock {Stochastic Approximation for Monte-Carlo Optimisation}.
\newblock \emph{Proceedings of the 1986 ACM Winter Simulation Conference},
  18:\penalty0 356--365, 1986.

\bibitem[Glynn(1990)]{glynn-90}
P.~W. Glynn.
\newblock {Likelihood Ratio Gradient Estimation for Stochastic Systems}.
\newblock \emph{Communications of the ACM}, 33:\penalty0 97--84, 1990.

\bibitem[Greensmith et~al.(2004)Greensmith, Bartlett, and
  Baxter]{Greensmith-etal-2004}
E.~Greensmith, P.~Bartlett, and J.~Baxter.
\newblock {Variance Reduction Techniques For Gradient Based Estimates in
  Reinforcement Learning}.
\newblock \emph{Journal of Machine Learning Research}, 5:\penalty0 1471--1530,
  2004.

\bibitem[Hestenes and Stiefel(1952)]{hestenes1952}
M.~Hestenes and E.~Stiefel.
\newblock {Methods of Conjugate Gradients for Solving Linear Systems}.
\newblock \emph{{Journal of Research of the National Bureau of Standards}},
  49:\penalty0 409--436, 1952.

\bibitem[Hoffman et~al.(2009)Hoffman, de~Freitas, Doucet, and
  Peters]{hoffman-etal-09}
M.~Hoffman, N.~de~Freitas, A.~Doucet, and J.~Peters.
\newblock {An Expectation Maximization Algorithm for Continuous Markov Decision
  Processes with Arbitrary Rewards.}
\newblock \emph{{AISTATS}}, 12\penalty0 (5):\penalty0 232--239, 2009.

\bibitem[Howard(1960)]{howard-policy-iteration-1960}
R.~A. Howard.
\newblock \emph{Dynamic Programming and Markov Processes}.
\newblock M.I.T. Press, 1960.

\bibitem[Ijspeert et~al.(2002)Ijspeert, Nakanishi, and Schaal]{ijspeert-2002}
A.~Ijspeert, J.~Nakanishi, and S.~Schaal.
\newblock {Motor Imitation with Nonlinear Dynamical Systems in Humanoid
  Robots}.
\newblock \emph{IEEE International Conference on Robotic and Automation}, pages
  1398--1403, 2002.

\bibitem[Ijspeert et~al.(2003)Ijspeert, Nakanishi, and Schaal]{ijspeert-2003}
A.~Ijspeert, J.~Nakanishi, and S.~Schaal.
\newblock {Learning Attractor Landscapes for Learning Motor Primitives}.
\newblock \emph{Neural Information Processing Systems}, 15:\penalty0
  1547--1554, 2003.

\bibitem[Jacobson and Mayne(1970)]{Jacobson-70}
D.~Jacobson and D.~Mayne.
\newblock \emph{Differential Dynamic Programming}.
\newblock Elsevier, 1970.

\bibitem[Kakade(2001)]{Kakade-2001}
S.~Kakade.
\newblock {Optimizing Average Reward Using Discounted Rewards}.
\newblock \emph{COLT}, 14:\penalty0 605--615, 2001.

\bibitem[Kakade(2002)]{Kakade-2002}
S.~Kakade.
\newblock {A Natural Policy Gradient}.
\newblock \emph{NIPS}, 14, 2002.

\bibitem[Khalil(2001)]{khalil-2001}
H.~Khalil.
\newblock \emph{Nonlinear Systems}.
\newblock Prentice Hall, 2001.

\bibitem[Kiefer and Wolfowitz(1952)]{kiefer-1952}
J.~Kiefer and J.~Wolfowitz.
\newblock {Stochastic Estimation of the Maximum of a Regression Function}.
\newblock \emph{Annals of Mathematical Statistics}, 23:\penalty0 462--466,
  1952.

\bibitem[Kober and Peters(2009)]{Kober-Peters-08}
J.~Kober and J.~Peters.
\newblock {Policy Search for Motor Primitives in Robotics}.
\newblock \emph{NIPS}, 21:\penalty0 849--856, 2009.

\bibitem[Kober and Peters(2011)]{Kober-peters-2011}
J.~Kober and J.~Peters.
\newblock {Policy Search for Motor Primitives in Robotics}.
\newblock \emph{Machine Learning}, 84\penalty0 (1-2):\penalty0 171--203, 2011.

\bibitem[Kocsis and Szepesv\'{a}ri(2006)]{kocsis-2006-ecml}
L.~Kocsis and C.~Szepesv\'{a}ri.
\newblock {Bandit Based Monte-Carlo Planning}.
\newblock \emph{European Conference on Machine Learning (ECML)}, 17:\penalty0
  282--293, 2006.

\bibitem[Kohl and Stone(2004)]{kohl-2004}
N.~Kohl and P.~Stone.
\newblock {Policy Gradient Reinforcement Learning for Fast Quadrupedal
  Locomotion}.
\newblock \emph{IEEE International Conference on Robotics and Automation},
  2004.

\bibitem[Konda and Tsitsiklis(1999)]{DBLP:conf/nips/KondaT99}
V.~Konda and J.~Tsitsiklis.
\newblock {Actor-Critic Algorithms}.
\newblock In \emph{NIPS}, pages 1008--1014, 1999.

\bibitem[Konda and Tsitsiklis(2003)]{Konda:2003:AA:942271.942292}
V.~R. Konda and J.~N. Tsitsiklis.
\newblock {On Actor-Critic Algorithms}.
\newblock \emph{SIAM J. Control Optim.}, 42\penalty0 (4):\penalty0 1143--1166,
  2003.

\bibitem[Lagoudakis and Parr(2003)]{DBLP:journals/jmlr/LagoudakisP03}
M.~G. Lagoudakis and R.~Parr.
\newblock {Least-Squares Policy Iteration}.
\newblock \emph{Journal of Machine Learning Research}, 4:\penalty0 1107--1149,
  2003.

\bibitem[Lehmann and Casella(1998)]{Lehmann-book}
E.~L. Lehmann and G.~Casella.
\newblock \emph{Theory of Point Estimation}.
\newblock {Springer}, 1998.

\bibitem[Li(2006)]{li-thesis}
W.~Li.
\newblock \emph{Optimal Control for Biological Movement Systems}.
\newblock PhD thesis, University of San Diego, 2006.

\bibitem[Li and Todorov(2004)]{li-todorov-2004}
W.~Li and E.~Todorov.
\newblock {Iterative Linear Quadratic Regulator Design for Nonlinear Biological
  Movement Systems}.
\newblock \emph{International Conference on Informatics in Control, Automation
  and Robotics}, 1, 2004.

\bibitem[Li and Todorov(2006)]{li-todorov-2006}
W.~Li and E.~Todorov.
\newblock {Iterative Optimal Control and Estimation Design for Nonlinear
  Stochastic Systems}.
\newblock \emph{IEEE Conference on Decision and Control}, 45, 2006.

\bibitem[Little and Rubin(2002)]{little-em-book}
R.~Little and D.~Rubin.
\newblock \emph{Statistical Analysis with Missing Data}.
\newblock {Wiley-Blackwell}, 2002.

\bibitem[Marbach and Tsitsiklis(2001)]{marbach-etal-2001}
P.~Marbach and J.~Tsitsiklis.
\newblock {Simulation-Based Optimisation of Markov Reward Processes}.
\newblock \emph{IEEE Transactions on Automatic Control}, 46\penalty0
  (2):\penalty0 191--209, 2001.

\bibitem[Meuleau et~al.(1999)Meuleau, Peshkin, Kim, and
  Kaelbling]{DBLP:conf/uai/MeuleauPKK99}
N.~Meuleau, L.~Peshkin, K~Kim, and L.~Kaelbling.
\newblock {Learning Finite-State Controllers for Partially Observable
  Environments}.
\newblock \emph{UAI}, 15:\penalty0 427--436, 1999.

\bibitem[Mnih et~al.(2015)Mnih, Kavukcuoglu, Silver, Rusu, Veness, Bellemare,
  Graves, Riedmiller, Fidjeland, Ostrovski, Petersen, Beattie, Sadik,
  Antonoglou, King, Kumaran, Wierstra, Legg, and Hassabis]{citeulike:13527831}
V.~Mnih, K.~Kavukcuoglu, D.~Silver, A.~A. Rusu, J.~Veness, M.~G. Bellemare,
  A.~Graves, M.~Riedmiller, A.~K. Fidjeland, G.~Ostrovski, S.~Petersen,
  C.~Beattie, A.~Sadik, I.~Antonoglou, H.~King, D.~Kumaran, D.~Wierstra,
  S.~Legg, and D.~Hassabis.
\newblock Human-level control through deep reinforcement learning.
\newblock \emph{Nature}, 518\penalty0 (7540):\penalty0 529--533, February 2015.

\bibitem[Neal and Hinton(1999)]{neal-hinton-em-variants}
R.~Neal and G.~Hinton.
\newblock {A View of the EM Algorithm that Justifies Incremental, Sparse and
  Other Variants}.
\newblock \emph{Learning in Graphical Models}, pages 355--368, 1999.

\bibitem[Ngo et~al.(2011)Ngo, Hwanjo, and TaeChoong]{ngo-etal-2011}
A~Ngo, Y.~Hwanjo, and C.~TaeChoong.
\newblock {Hessian Matrix Distribution for Bayesian Policy Gradient
  Reinforcment Learning}.
\newblock \emph{Information Sciences}, 181:\penalty0 1671--1685, 2011.

\bibitem[Nocedal and Wright(2006)]{nocedal-2006}
J.~Nocedal and S.~Wright.
\newblock \emph{Numerical Optimisation}.
\newblock Springer, 2006.

\bibitem[Ortega and Rheinboldt(1970)]{ortega-rheinboldt}
J.~M. Ortega and W.~C. Rheinboldt.
\newblock \emph{Iterative Solution of Nonlinear Equations in Several
  Variables}.
\newblock Academic Press, {first} edition, 1970.

\bibitem[Peters and Schaal(2008)]{peters-NAC_2008}
J.~Peters and S.~Schaal.
\newblock {Natural Actor-Critic}.
\newblock \emph{Neurocomputing}, 71\penalty0 (7-9):\penalty0 1180--1190, 2008.

\bibitem[Rawlik et~al.(2012)Rawlik, Toussaint, and
  Vijayakumar]{rawlik-2012-etal}
K.~Rawlik, M.~Toussaint, and S.~Vijayakumar.
\newblock {On Stochastic Optimal Control and Reinforcement Learning by
  Approximate Inference}.
\newblock \emph{International Conference on Robotics Science and Systems},
  2012.

\bibitem[Richter et~al.(2007)Richter, Aberdeen, and Yu]{richter-2007}
S.~Richter, D.~Aberdeen, and J.~Yu.
\newblock {Natural Actor-Critic for Road Traffic Optimisation}.
\newblock \emph{NIPS}, 19:\penalty0 1169--1176, 2007.

\bibitem[Russell and Norvig(2009)]{russell-norig-09}
S.~Russell and P.~Norvig.
\newblock \emph{Artificial Intelligence: A Modern Approach}.
\newblock Prentice Hall, 2009.

\bibitem[Saul et~al.(1996)Saul, Jaakkola, and Jordan]{saul-etal-1996}
L.~Saul, T.~Jaakkola, and M.~Jordan.
\newblock {Mean Field Theory for Sigmoid Belief Networks}.
\newblock \emph{Journal of Artificial Intelligence Research}, 4:\penalty0
  61--76, 1996.

\bibitem[Schaal(2006)]{schaal-2006}
S.~Schaal.
\newblock {The SL Simulation and Real-Time Control Software Package}.
\newblock Technical report, University of Southern California, 2006.

\bibitem[Schaal et~al.(2007)Schaal, Mohajerian, and Ijspeert]{schaal-2007}
S.~Schaal, P.~Mohajerian, and A.~Ijspeert.
\newblock {Dynamics Systems Vs. Optimal Control - A Unifying View}.
\newblock \emph{Progress in Brain Research}, 165\penalty0 (1):\penalty0
  425--445, 2007.

\bibitem[Schraudolph et~al.(2006)Schraudolph, Yu, and
  Aberdeen]{Schraudolph-etal-nips-2006}
N.~Schraudolph, J.~Yu, and D.~Aberdeen.
\newblock {Fast Online Policy Gradient Learning with SMD Gain Vector
  Adaptation}.
\newblock \emph{NIPS}, 18:\penalty0 1185--1192, 2006.

\bibitem[Schraudolph et~al.(2007)Schraudolph, Yu, and
  Gunter]{Schraudolph-2007-etal}
N.~Schraudolph, J.~Yu, and S.~Gunter.
\newblock {A Stochastic Quasi-Newton Method for Online Convex Optimization}.
\newblock \emph{AISTATS}, 11:\penalty0 433--440, 2007.

\bibitem[Spall(1992)]{spall-1992}
J.~Spall.
\newblock {Multivariate Stochastic Approximation Using a Simultaneous
  Perturbation Gradient Approximation}.
\newblock \emph{IEEE Transactions on Automatic Control}, 37:\penalty0 332--341,
  1992.

\bibitem[Spall and Cristion(1998)]{spall-1998}
J.~Spall and J.~Cristion.
\newblock {Model-Free Control of Nonlinear Stochastic Systems with
  Discrete-Time Measurements}.
\newblock \emph{IEEE Transactions on Automatic Control}, 43:\penalty0
  1198--1210, 1998.

\bibitem[Spong et~al.(2005)Spong, Hutchinson, and Vidyasagar]{spong-etal-2005}
M.~Spong, S.~Hutchinson, and M.~Vidyasagar.
\newblock \emph{Robot Modelling and Control}.
\newblock John Wiley \& Sons, 2005.

\bibitem[Srinivasan et~al.(2006)Srinivasan, Choy, and Cheu]{srinivasan-2006}
D.~Srinivasan, M.~C. Choy, and R.~L. Cheu.
\newblock {Neural Networks for Real-Time Traffic Signal Control}.
\newblock \emph{IEEE Transactions on Intelligent Transportation Systems},
  7:\penalty0 261--272, 2006.

\bibitem[Stengel(1993)]{stengel-1993}
R.~Stengel.
\newblock \emph{Optimal Control and Estimation}.
\newblock Dover, 1993.

\bibitem[Sutton et~al.(2000)Sutton, McAllester, Singh, and
  Mansour]{sutton-etal-00}
R.~Sutton, D.~McAllester, S.~Singh, and Y.~Mansour.
\newblock {Policy Gradient Methods for Reinforcement Learning with Function
  Approximation}.
\newblock \emph{{NIPS}}, 13, 2000.

\bibitem[Tedrake and Zhang(2005)]{Tedrake-2005}
R.~Tedrake and T.~Zhang.
\newblock {Learning to Walk in 20 Minutes}.
\newblock \emph{Proceedings of the Fourteenth Yale Workshop on Adaptive and
  Learning Systems}, 2005.

\bibitem[Tesauro(1994)]{Tesauro-1994}
G.~Tesauro.
\newblock {TD-Gammon, A Self-Teaching Backgammon Program Achieves Master-Level
  Play}.
\newblock \emph{Neural Computation}, 6:\penalty0 215--219, 1994.

\bibitem[Todorov and Tassa(2009)]{Todorov09iterativelocal}
E.~Todorov and Y.~Tassa.
\newblock {Iterative Local Dynamic Programming}.
\newblock \emph{IEEE Symposium on Adaptive Dynamic Programming and
  Reinforcement Learning}, pages 90--95, 2009.

\bibitem[Toussaint et~al.(2006)Toussaint, Harmeling, and
  Storkey]{toussaint-etal-06}
M.~Toussaint, S.~Harmeling, and A.~Storkey.
\newblock {Probabilistic Inference for Solving (PO)MDPs}.
\newblock Research Report EDI-INF-RR-0934, University of Edinburgh, School of
  Informatics, 2006.

\bibitem[Toussaint et~al.(2011)Toussaint, Storkey, and
  Harmeling]{toussaint-etal-2011}
M.~Toussaint, A.~Storkey, and S.~Harmeling.
\newblock \emph{Bayesian Time Series Models}, chapter {Expectation-Maximization
  Methods for Solving (PO)MDPs and Optimal Control Problems.}
\newblock Cambridge University Press, 2011.
\newblock In press. See \verb"userpage.fu-berlin.de/~mtoussai".

\bibitem[Veness et~al.(2009)Veness, Silver, Blair, and Uther]{veness-2009}
J.~Veness, D.~Silver, A.~Blair, and W.~Uther.
\newblock {Bootstrapping from Game Tree Search}.
\newblock \emph{{NIPS}}, 19:\penalty0 1937--1945, 2009.

\bibitem[Vlassis et~al.(2009)Vlassis, Toussaint, Kontes, and
  Piperidis]{vlassis-etal-2009}
N.~Vlassis, M.~Toussaint, G.~Kontes, and S.~Piperidis.
\newblock {Learning Model-Free Robot Control by a Monte Carlo EM Algorithm}.
\newblock \emph{{Autonomous Robots}}, 27\penalty0 (2):\penalty0 123--130, 2009.

\bibitem[Weaver and Tao(2001)]{weaver-etal-2001}
L.~Weaver and N.~Tao.
\newblock {The Optimal Reward Baseline for Gradient Based Reinforcement
  Learning}.
\newblock volume~17, 2001.

\bibitem[Whitehead(1992)]{whitehead-thesis}
S.~Whitehead.
\newblock \emph{Reinforcement Learning for Adaptive Control of Perception and
  Action}.
\newblock PhD thesis, University of Rochester, 1992.

\bibitem[Wierstra et~al.(2014)Wierstra, Schaul, Glasmachers, Sun, Peters, and
  Schmidhuber]{wierstra-etal-2014}
D.~Wierstra, T.~Schaul, T.~Glasmachers, Y.~Sun, J.~Peters, and J.~Schmidhuber.
\newblock {Natural Evolution Strategies}.
\newblock \emph{Journal of Machine Learning Researchs}, 15:\penalty0 949--980,
  2014.

\bibitem[Williams(1992)]{williams92}
R.~Williams.
\newblock {Simple Statistical Gradient Following Algorithms for Connectionist
  Reinforcement Learning}.
\newblock \emph{{Machine Learning}}, 8:\penalty0 229--256, 1992.

\end{thebibliography}
